%% file: aaai2026.tex
\documentclass[letterpaper]{article} 
\usepackage{aaai2026}  
\usepackage{times}  
\usepackage{helvet}  
\usepackage{courier}  
\usepackage[hyphens]{url}  
\usepackage{graphicx} 
\usepackage{natbib}  
\usepackage{caption} 
\usepackage{algorithm}
\usepackage{algorithmic}
\usepackage{amsmath}
\usepackage{amsthm}
\usepackage{amsfonts}
\usepackage{dsfont}
\usepackage{mathrsfs}
\usepackage{amssymb}

\newtheorem{theorem}{Theorem}[section]
\newtheorem{proposition}[theorem]{Proposition}
\newtheorem{lemma}[theorem]{Lemma}

\newtheorem{definition}[theorem]{Definition}
\newtheorem{assumption}[theorem]{Assumption}

\usepackage{newfloat}
\usepackage{listings}
\DeclareCaptionStyle{ruled}{labelfont=normalfont,labelsep=colon,strut=off} 
\floatstyle{ruled}
\newfloat{listing}{tb}{lst}{}
\floatname{listing}{Listing}
\title{A Differential Perspective on Distributional Reinforcement Learning}
\author{
    Juan Sebastian Rojas\textsuperscript{\rm 1}, Chi-Guhn Lee\textsuperscript{\rm 1}
}
\affiliations{
    \textsuperscript{\rm 1}Department of Mechanical \& Industrial Engineering, University of Toronto\\
    juan.rojas@mail.utoronto.ca, cglee@mie.utoronto.ca
}

\begin{document}

\maketitle

\begin{abstract}
To date, distributional reinforcement learning (distributional RL) methods have exclusively focused on the \emph{discounted} setting, where an agent aims to optimize a discounted sum of rewards over time. In this work, we extend distributional RL to the \emph{average-reward} setting, where an agent aims to optimize the reward received per time step. In particular, we utilize a quantile-based approach to develop the first set of algorithms that can successfully learn and/or optimize the long-run per-step reward distribution, as well as the differential return distribution of an average-reward MDP. We derive proven-convergent tabular algorithms for both prediction and control, as well as a broader family of algorithms that have appealing scaling properties. Empirically, we find that these algorithms yield competitive and sometimes superior performance when compared to their non-distributional equivalents, while also capturing rich information about the long-run per-step reward and differential return distributions.
\end{abstract}


\section{Introduction}
\label{introduction}
Distributional reinforcement learning (distributional RL) \citep{Bellemare2023-mn} equips decision-making agents with the ability to learn and reason about the probability distribution over a given objective. This approach transcends the traditional RL paradigm of focusing solely on expected values, thereby offering a more insightful and methodical understanding of the variability, uncertainty, and risk associated with a given objective. To date, distributional RL methods have exclusively focused on the \emph{discounted} setting, where an RL agent aims to optimize a potentially-discounted sum of rewards over time (e.g. \citet{Bellemare2017-lf}). 

In this work, we extend distributional RL to the \emph{average-reward} setting, where an RL agent aims to optimize the reward received per time step. This extension offers a timely opportunity to extend the benefits of distributional RL to a promising and growing family of RL methods. In particular, unlike discounted RL methods, average-reward RL methods do not face fundamental challenges in continuing control tasks that require function approximation \citep{Naik2019-ch}. Moreover, average-reward RL methods have been shown to outperform discounted RL methods in some instances (e.g. \citet{Adamczyk2025-cv}). It has even been shown that integrating the average-reward itself into discounted RL methods can increase their performance (e.g. \citet{Naik2024-km}). Accordingly, extending distributional RL to the average-reward setting allows us to combine the strengths of two increasingly important RL frameworks. 

Through this extension, we will see that, from a distributional perspective, we will need to rethink the following foundational questions: \emph{what do we want to learn?} and \emph{how can we learn it?} In this work, we address these questions, and in the process, derive the first distributional framework for the average-reward setting.

\section{Related Work}
\label{related_work}
Early works on distributional RL include \citet{Sobel1982-pi}, \citet{White1988-ml}, and \citet{Morimura2010-ve}. Modern distributional RL is typically associated with the distributional Bellman operator, which was introduced in \citet{Bellemare2017-lf}, along with empirical results which showed that such methods could outperform their non-distributional equivalents. Subsequent research has expanded upon this framework by developing methods that can typically be classified as either \emph{categorical} (e.g. \citet{Bellemare2017-lf}) or \emph{quantile}-based (e.g. \citet{Dabney2018-ta}). While these methods have proven to be effective in the discounted setting, no attention has been given to the average-reward setting. In this work, we address this gap by introducing the first distributional RL algorithms specifically designed for the average-reward setting.

\section{Preliminaries}
\label{preliminaries}

\subsection{Average-Reward Reinforcement Learning}
A finite average-reward MDP is the tuple \(\mathcal{M} \doteq \langle \mathcal{S}, \mathcal{A}, \mathcal{R}, p \rangle\), where \(\mathcal{S}\) is a finite set of states, \(\mathcal{A}\) is a finite set of actions, \(\mathcal{R} \subset \mathbb{R}\) is a finite set of rewards, and \(p: \mathcal{S}\, \times\, \mathcal{A}\, \times\, \mathcal{R}\, \times\,  \mathcal{S} \rightarrow{} [0, 1]\) is a probabilistic transition function that describes the dynamics of the environment. At each discrete time step, \(t = 0, 1, 2, \ldots\), an agent chooses an action, \(A_t \in \mathcal{A}\), based on its current state, \(S_t \in \mathcal{S}\), and receives a reward, \(R_{t+1} \in \mathcal{R}\), while transitioning to a (potentially) new state, \(S_{t+1}\), such that \(p(s', r \mid s, a) = \mathbb{P}(S_{t+1} = s', R_{t+1} = r \mid S_t = s, A_t = a)\). In an average-reward MDP, an agent aims to find a policy, \(\pi: \mathcal{S} \rightarrow{} \mathcal{A}\), that optimizes the long-run (or limiting) average-reward, \(\bar{r}\), which is defined as follows for a given policy, \(\pi\):
\begin{equation}
\label{eq_avg_reward}
\bar{r}_{\pi}(s) \doteq  \lim_{n \rightarrow{} \infty} \frac{1}{n} \sum_{t=1}^{n} \mathbb{E}[R_t \mid S_0=s, A_{0:t-1} \sim \pi].
\end{equation}
In this work, we limit our discussion to \emph{stationary Markov} policies, which are time-independent policies that satisfy the Markov property.

When working with average-reward MDPs, it is common to simplify Equation \eqref{eq_avg_reward} into a more workable form by making certain assumptions about the Markov chain induced by following policy \(\pi\). To this end, a \emph{unichain} assumption is typically used when doing prediction (learning) because it ensures the existence of a unique limiting distribution of states, \(\mu_{\pi}(s) \doteq \lim_{t \rightarrow{} \infty} \mathbb{P}(S_t = s \mid A_{0:t-1} \sim \pi)\), that is independent of the initial state, thereby simplifying Equation \eqref{eq_avg_reward} to the following:
\begin{equation}
\label{eq_avg_reward_2}
\bar{r}_{\pi} = \sum_{s \in \mathcal{S}} \mu_{\pi}(s) \sum_{a \in \mathcal{A}} \pi(a \mid s) \sum_{s' \in \mathcal{S}} \sum_{r \in  \mathcal{R}}p(s', r \mid s, a)r.
\end{equation}
Similarly, a \emph{communicating} assumption is typically used when doing control (optimization) because it ensures the existence of a unique optimal average-reward, \(\bar{r}*\), that is independent of the initial state. 

The \textit{return} of an MDP, \(G_t\), captures how rewards are aggregated over the time horizon. In an average-reward MDP, the return is referred to as the \textit{differential return}, and is defined as follows:
\begin{equation}
\label{eq_avg_reward_3}
G_t \doteq R_{t+1} - \bar{r}_{\pi} + R_{t+2} - \bar{r}_{\pi} + R_{t+3} - \bar{r}_{\pi} + \ldots.
\end{equation}

To solve an average-reward MDP, solution methods such as dynamic programming or RL are typically used in conjunction with the \textit{Bellman equation \eqref{eq_avg_reward_4}} for the state-value function, \(v_{\pi}(s) \doteq \mathbb{E}_{\pi}[G_t \mid S_t = s] = \mathbb{E}_{\pi}[R_{t+1} - \bar{r}_{\pi} + G_{t+1} \mid S_t = s]\), or the \textit{Bellman optimality equation \eqref{eq_avg_reward_5}} for the state-action value function, \(q_{\pi}(s, a) \doteq \mathbb{E}_{\pi}[G_t \mid S_t = s, A_t = a]\). Solution methods for average-reward MDPs are typically referred to as \emph{differential} methods because of the reward difference (i.e., \(r - \bar{r}_{\pi}\)) operation that occurs in the Bellman equations \eqref{eq_avg_reward_4} and \eqref{eq_avg_reward_5}: 
\begin{equation}
\label{eq_avg_reward_4}
v_{\pi}(s) = \sum_{a}\pi(a | s) \sum_{s', r}p(s', r \mid s, a)[r - \bar{r}_{\pi} + v_{\pi}(s')],
\end{equation}

\begin{equation}
\label{eq_avg_reward_5}
q_{\pi}(s, a) = \sum_{s', r}p(s', r \mid s, a)[r - \bar{r}_{\pi} + \max_{a'}q_{\pi}(s', a')].
\end{equation}

\subsection{Quantile Regression}
Quantile regression \citep{Koenker2005-hy} refers to the process of estimating a predetermined quantile of a probability distribution from samples. More specifically, for \(\tau \in (0, 1)\), let \(F_{w}^{-1}(\tau)\) denote the \(\tau^\text{th}\) quantile that we are trying to estimate from probability distribution \(w\) (where \(F_w: \mathbb{R} \to [0, 1]\) denotes the CDF of \(w\)). Quantile regression maintains an estimate, \(\theta\), of this value, and updates the estimate based on samples drawn from \(w\) (i.e., \(y \sim w\)) as follows \citep{Bellemare2023-mn}:
\begin{equation}
\label{eq_qr_1}
\theta \xleftarrow{} \theta + \alpha_t (\tau - \mathds{1}_{ \{ y < \theta \}}),
\end{equation}
where \(\alpha_t\) denotes the step size for the update. The estimate for \(\theta\) will continue to adjust until the equilibrium point, \(\theta^*\), which corresponds to \(F_{w}^{-1}(\tau)\), is reached \citep{Bellemare2023-mn}. In other words, we have that
\begin{subequations}
\label{eq_qr_2}
\begin{align}
\label{eq_qr_2_1}
0 & = \mathbb{E}[(\tau - \mathds{1}_{ \{ y < \theta^* \}})]\\
\label{eq_qr_2_2}
& = \tau - \mathbb{E}[\mathds{1}_{ \{ y < \theta^* \}}]\\
\label{eq_qr_2_3}
& = \tau - \mathbb{P}(y < \theta^* )\\
\label{eq_qr_2_4}
\implies \theta^* & = F_{w}^{-1}(\tau).
\end{align}
\end{subequations}

\subsection{Discounted Distributional RL}
\label{section_discounted}
Like average-reward MDPs, discounted MDPs have their own return which captures how rewards are aggregated over the time horizon. More specifically, let
\begin{equation}
\label{eq_disc_return}
G^{\gamma}_t = \sum_{k \geq 0} \gamma^{k} R_{t+k+1} = R_{t+1} + \gamma G^{\gamma}_{t+1}
\end{equation}
denote the \emph{discounted return} from a discounted MDP, where \(R_{t+1} \in \mathcal{R}\) denotes the per-step reward and \(\gamma \in [0, 1]\) denotes the discount factor \citep{Puterman1994-dq, Sutton2018-eh}. The aim of \emph{discounted distributional RL} is to learn the probability distribution over discounted returns. More formally, discounted distributional RL aims to learn the discounted return distribution function, \(\Phi^{\gamma}_\pi\), such that, with a slight abuse of notation, \(\Phi^{\gamma}_\pi(s)\) denotes the probability distribution over discounted returns when starting from state \(s \in \mathcal{S}\) and following policy \(\pi\), and \(\Phi^{\gamma}_\pi(s, a)\) denotes the probability distribution over discounted returns when starting from state \(s \in \mathcal{S}\), taking action \(a \in \mathcal{A}\), and following policy \(\pi\). 

Broadly speaking, discounted distributional RL methods can be categorized based on how they approximate (or parameterize) the discounted return distribution function. In this work, we take inspiration from quantile-based methods \citep{Dabney2018-ta, Rowland2024-sg}, which parameterize the discounted return distribution function as follows:
\begin{equation}
\label{eq_ddrl_1}
\Phi^{\gamma}_\pi(s) = \sum_{i = 1}^{n}{\frac{1}{n}\delta_{\Omega^{\gamma}_i(s)}} \quad \text{or} \quad \Phi^{\gamma}_\pi(s, a) = \sum_{i = 1}^{n}{\frac{1}{n}\delta_{\Omega^{\gamma}_i(s, a)}},
\end{equation}
where \(\Omega^{\gamma}_i\) denotes the \(\tau_i\)-quantile of the discounted return distribution, and \(\delta_{\Omega^{\gamma}_i}\) denotes a Dirac at \(\Omega^{\gamma}_i\). 

More formally, the set of \(\tau\)-quantiles of a probability distribution can be defined as follows:
\begin{definition}[\citet{Rowland2024-sg}]
\label{defn_quantile}
Let \(\mathscr{P}(\mathbb{R})\) denote the set of probability distributions over \(\mathbb{R}\). For a probability distribution \(w \in \mathscr{P}(\mathbb{R})\) and parameter \(\tau \in (0, 1)\), the set of \(\tau\)-quantiles of \(w\) is given by the set \(\{ z \in \mathbb{R}: F_w(z) = \tau \} \cup \inf\{ y \in \mathbb{R}: F_w(y) > \tau \}\), where \(F_w: \mathbb{R} \to [0, 1]\) is the CDF of \(w\), defined by \(F_w(t) = \mathbb{P}_{Z \sim w}(Z \leq t)\) for all \(t \in \mathbb{R}\).
\end{definition}
Quantile-based discounted distributional RL methods learn the \(\tau\)-quantiles of the discounted return distribution, \(\{\Omega^{\gamma}_i\}_{i=1}^{n}\), from samples using quantile regression (see Equation \ref{eq_qr_1}) as follows \citep{Dabney2018-ta}:
\begin{equation}
\label{eq_ddrl_2}
\Omega^{\gamma}_{i, t+1} = \Omega^{\gamma}_{i, t} + \alpha_t\frac{1}{n}\sum_{j=1}^{n}\left[ 
 \tau_i - \mathds{1}_{ \{  \psi_t < 0 \} } \right], \; \forall i = 1,  2,  \ldots, n,
\end{equation}
where \(\Omega^{\gamma}_{i, t} \) denotes the estimate of the \(\tau_i\)-quantile of the discounted return distribution at time \(t\), \(\alpha_t\) denotes the step size for the update, and \(\{\tau_i \in (0, 1)\}_{i=1}^{n}\), such that \(\tau_i = \frac{2i - 1}{2n}, \; i = 1, 2, \ldots, n\). Here, the choice of \(\psi_t\) depends on whether we want to do prediction (learning) or control (optimization). In the case of prediction, \(\psi_t = R_{t+1} + \gamma \Omega^{\gamma}_{j, t}(S_{t+1}) - \Omega^{\gamma}_{i, t}(S_t)\). In the case of control (via Q-learning \citep{Watkins1992-nq}), \(\psi_t = R_{t+1} + \gamma \Omega^{\gamma}_{j, t}(S_{t+1}, a^{*}) - \Omega^{\gamma}_{i, t}(S_t, A_t)\), where \(a^{*} \doteq \text{argmax}_{a'} {\frac{1}{n}\sum_{k = 1}^{n}\Omega^{\gamma}_{k, t}(S_{t+1}, a')}\).

Let us take a moment to highlight the definition of the above greedy action, \(a^{*}\). We are being greedy with respect to the \emph{average} of the \(\tau\)-quantiles of the discounted return distribution, which, by definition, is equivalent to being greedy with respect to the expected discounted return, and hence, the (discounted) state-action value function. More formally, the \emph{greedy action selection rule} for a distributional RL algorithm can be defined as follows:
\begin{definition}[\citet{Bellemare2023-mn}]
\label{defn_greedy_rule}
Let \(\boldsymbol{\pi}_{_\text{MS}}\) denote the space of stationary Markov policies, and let \(Q: \mathcal{S} \times \mathcal{A} \to \mathbb{R}\) denote a state-action value function. A greedy action selection rule is a mapping, \(\mathcal{G}: \mathbb{R}^{\mathcal{S} \times \mathcal{A}} \to \boldsymbol{\pi}_{_\text{MS}}\), with the property that for any \(Q \in \mathbb{R}^{\mathcal{S}  \times \mathcal{A}}\), \(\mathcal{G}(Q)\) is greedy with respect to \(Q\). That is, \(\mathcal{G}(Q)(a \mid s) > 0 \implies Q(s, a) = \max_{a' \in \mathcal{A}}Q(s,a')\).
\end{definition}
Importantly, we note that although discounted distributional RL methods make it possible to capture information related to the underlying (discounted) return distribution, much of the theoretical guarantees associated with these methods in the control setting require an action selection rule that satisfies Definition \ref{defn_greedy_rule}. That is, they require a greedy action selection rule that is greedy with respect to the expected return, and hence, the state-action value function \citep{Bellemare2023-mn}. 

Similarly, we note that the convergence of discounted distributional RL approaches, including the above quantile-based approach, is governed by the (discounted) distributional Bellman operator \citep{Bellemare2017-lf, Bellemare2023-mn}. Importantly, under certain conditions, the distributional Bellman operator is a contraction in the Wasserstein metric, thereby ensuring convergence to a fixed point.

\section{Differential Distributional RL}
\label{ar_drl}
In this section, we derive and present our primary contribution: the first set of distributional RL algorithms specifically designed for the average-reward setting. We call these algorithms, \emph{Differential Distributional RL} algorithms, due to the differential nature of average-reward RL algorithms. 

Importantly, this extension of distributional RL into the average-reward setting requires that we move away from using the discounted MDP, which forms the basis of existing distributional RL methods, and, in turn, utilize the average-reward MDP. Consequently, we cannot rely on the (discounted) distributional Bellman operator \citep{Bellemare2017-lf, Bellemare2023-mn} when formulating our approach.

In fact, this shift requires that we rethink and answer the following foundational questions: \textit{what do we want to learn?} and \textit{how can we learn it?} In this section, we propose a quantile-based framework that allows us to address these two questions, and in the process, extend distributional RL into the average-reward setting.

\subsection{The Differential Distributional Objective}
\label{objective}
Before we can derive our algorithms, we first need to establish an appropriate distributional objective. In particular, the average-reward formulation suggests \emph{two} distributions that may be of interest. The first corresponds to the probability distribution over differential returns (where the differential return is defined in Equation \eqref{eq_avg_reward_3}). The second corresponds to the limiting (or long-run) per-step reward distribution, whose mean yields the regular (non-distributional) average-reward objective (i.e., Equation \eqref{eq_avg_reward}).

At a first glance, it may seem natural to mirror the approach taken in the discounted setting by pursuing the differential return as the distributional objective. However, such an approach would not be fully aligned with the nature of the average-reward setting. In particular, in the average-reward setting, the differential return serves as a surrogate objective, optimized only to facilitate the optimization of the long-run per-step average-reward (Equation \eqref{eq_avg_reward}). As such, while it may be feasible to capture information about the differential return distribution, such information would be of little relevance in the average-reward setting, where the long-term, per-step behaviour is what is of interest. Conversely, the limiting per-step reward distribution aligns directly with the average-reward objective (its mean yields the average-reward itself), thereby making it an appealing distributional objective for the average-reward setting.

As such, given the above reasoning, which is formalized as Proposition \ref{proposition_1} below, the primary focus of this work will be to derive differential algorithms that can learn and/or optimize the \emph{limiting per-step reward distribution}. We will briefly revisit the return distribution from a purely empirical perspective in Section \ref{longrun_value_dist}.

\begin{proposition}
\label{proposition_1}
The limiting per-step reward distribution is the natural distributional objective in the average-reward setting, given that its mean yields the long-run average-reward, which is the primary prediction and control objective of (non-distributional) average-reward RL.
\end{proposition}

\begin{figure*}[htbp]
\centerline{\includegraphics[scale=0.53]{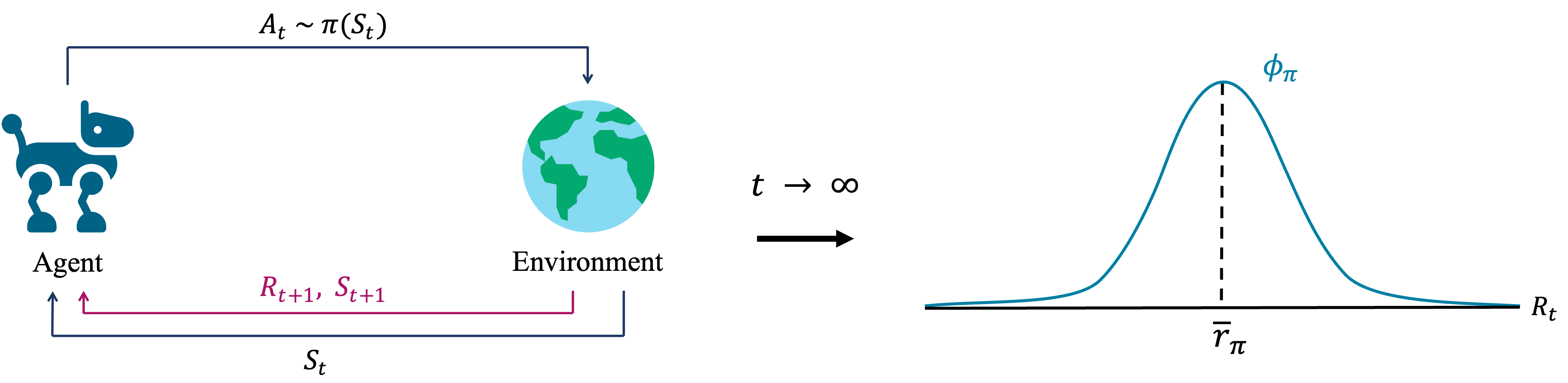}}
\caption{Illustration of the agent-environment interaction in an average-reward MDP. As \(t \to \infty\), following policy \(\pi\) yields a limiting per-step reward distribution, \(\phi_{\pi}\), with an average-reward, \(\bar{r}_{\pi}\). Standard average-reward RL methods aim to learn and/or optimize the average-reward, \(\bar{r}_{\pi}\). By contrast, the differential distributional RL methods proposed in this work aim to learn and/or optimize the limiting per-step reward distribution, \(\phi_{\pi}\).}
\label{fig_average_reward}
\end{figure*}

We are now ready to begin our theoretical treatment of the limiting per-step reward distribution. We begin by formally defining our distributional objective. In particular, for a given policy, \(\pi\), let
\begin{equation}
\label{eq_r_dist_1}
\phi_{\pi}(s) \doteq \lim_{t \rightarrow{} \infty} \mathbb{P}(R_t \mid S_0 = s, A_{0:t-1} \sim \pi)
\end{equation}
denote the limiting per-step reward distribution induced by following policy \(\pi\), where \(R_t \in \mathcal{R}\) denotes the per-step reward. As is standard practice with the regular average-reward objective, we can simplify Equation \eqref{eq_r_dist_1} by making certain assumptions about the Markov chain induced by following policy \(\pi\). To this end, we will utilize a \emph{unichain} assumption when doing prediction (learning), because it ensures the existence of a unique limiting distribution of per-step rewards that is independent of the initial state, such that \(\phi_{\pi}(s) = \phi_{\pi}\). Similarly, we will utilize a \emph{communicating} assumption when doing control (optimization), because it ensures the existence of a unique optimal distribution of per-step rewards that is independent of the initial state. Figure \ref{fig_average_reward} depicts the agent-environment interaction in an average-reward MDP, where following policy \( \pi \) yields a limiting reward distribution, \(\phi_{\pi}\), whose mean corresponds to the average-reward objective, \(\bar{r}_{\pi}\). 

\subsection{A Quantile-Based Approach}
\label{quantile_approach}
In this section, we take inspiration from the quantile-based methods utilized for discounted distributional RL to develop a quantile-based approach that can be used to learn and/or optimize the limiting per-step reward distribution of an average-reward MDP. Like the discounted setting (see Section \ref{section_discounted}), our overall goal and strategy is to derive an appropriate set of quantile regression-based updates that approximate our target distribution. More formally, we approximate (or parameterize) the limiting per-step reward distribution as follows:
\begin{equation}
\label{eq_r_dist_2}
\phi_\pi = \sum_{i = 1}^{m}{\frac{1}{m}\delta_{\theta_i}},
\end{equation}
where \(\theta_i\) denotes the \(\tau_i\)-quantile of the limiting per-step reward distribution, and \(\delta_{\theta_i}\) denotes a Dirac at \(\theta_i\). Here, we adopt the same formal definition for the \(\tau\)-quantiles of a probability distribution as in the discounted setting (i.e., Definition \ref{defn_quantile}).

Now, let us consider the \(\tau\)-quantiles of the limiting per-step reward distribution, \(\{\theta_i\}_{i=1}^{m}\). As previously mentioned, we wish to approximate these quantiles using quantile regression. To this end, the generic quantile regression framework outlined in Equations \eqref{eq_qr_1} and \eqref{eq_qr_2} suggest the following set of quantile update rules for the per-step reward quantiles: 
\begin{equation}
\label{eq_r_dist_3}
\theta_{i, t+1} = \theta_{i, t} + \alpha_t\left( 
 \tau_i - \mathds{1}_{ \{  R_{t+1} < \theta_{i,t} \} } \right) \; \forall i = 1,  2,  \ldots, m,
\end{equation}
where \(\theta_{i, t} \) denotes the estimate of the \(\tau_i\)-quantile of the limiting per-step reward distribution, \(\alpha_t\) denotes the step size for the update, and \(R_{t+1}\) denotes the per-step reward. 

When comparing this quantile update rule to that of the discounted setting (i.e., Equation \eqref{eq_ddrl_2}), we can see that both rules have a similar structure, with the key difference being that we have replaced \(\psi_t\) with \(R_{t+1} - \theta_{i,t}\). Most notably, unlike the return-based quantile update rule used in the discounted setting, the per-step reward quantile update rule \eqref{eq_r_dist_3} does not contain the full TD or Q-learning targets and estimates. This is not necessarily an issue, so long as we can find a way to properly incorporate the quantile update rule \eqref{eq_r_dist_3} into a broader algorithm that does contain the full TD or Q-learning targets and estimates.

To this end, we note that, typically, differential RL algorithms will learn and/or optimize the value function and average-reward simultaneously. From an algorithmic perspective, this means that the algorithms start with initial estimates (or guesses) for the value function and average-reward, then update these estimates over time, until they have learned and/or optimized these two objectives. As such, one way through which we can incorporate the quantile updates into a differential algorithm is by incorporating them into the average-reward estimate update. In particular, we can relate the average-reward, \(\bar{r}_{\pi}\), to the \(\tau\)-quantiles of the limiting per-step reward distribution, \(\{\theta_i\}_{i=1}^{m}\), as follows:
\begin{lemma}
\label{lemma_1}
Given a unichain, communicating, or equivalent assumption, consider the limiting per-step reward distribution parameterized by the \(\tau\)-quantiles, \(\{\theta_i\}_{i=1}^{m}\), as per Equation \eqref{eq_r_dist_2} for \(\{\tau_i \in (0, 1)\}_{i=1}^{m}\), such that \(\tau_i = \frac{2i - 1}{2m}, \; i = 1, 2, \ldots, m\). The expected value of the \(\tau\)-quantiles converges to the average-reward of the limiting per-step reward distribution, \(\bar{r}_{\pi}\), as \(m \to \infty\).
\end{lemma}

\begin{proof}
Let \(F(r)\) denote the CDF of the limiting per-step reward distribution (which exists and is unique given a unichain, communicating, or equivalent assumption), and let \(\theta_i\) represent its \(\tau_i\)-quantile, such that \(F(\theta_i) = \tau_i\). The expected value of the \(\tau\)-quantiles is: 
\begin{equation}
\label{eq_r_dist_4}
\frac{1}{m} \sum_{i=1}^m \theta_i = \frac{1}{m} \sum_{i=1}^m F^{-1}\left(\frac{2i - 1}{2m}\right),
\end{equation}

where \(F^{-1}\) is the inverse CDF. 

By the definition of \(\tau_i = \frac{2i - 1}{2m}\), the \(\tau_i\) values are evenly spaced over \([0, 1]\), creating a uniform partition. Hence, as \(m \to \infty\), the summation \(\frac{1}{m} \sum_{i=1}^m \theta_i\) becomes a Riemann sum for the following integral: 

\begin{equation}
\label{eq_r_dist_5}
\int_0^1 F^{-1}(\tau) \, d\tau.
\end{equation}
It is a well-known result that for a random variable \(X\) with CDF \(F(x)\), \(E[X] = \int_0^1 F^{-1}(x) dx\). Therefore, as \(m \to \infty\),
\begin{equation}
\label{eq_r_dist_6}
\frac{1}{m} \sum_{i=1}^m \theta_i \to \int_0^1 F^{-1}(\tau) \, d\tau = \bar{r}_{\pi}.
\end{equation}
Thus, the expected value of the \(\tau\)-quantiles converges to the average reward, \(\bar{r}_{\pi}\), of the limiting per-step reward distribution as \(m \to \infty\). This completes the proof.
\end{proof}

Hence, we now have a way to express the set of quantile updates \eqref{eq_r_dist_3} as an average-reward update:
\begin{subequations}
\label{eq_r_dist_7}
\begin{align}
\label{eq_r_dist_7_1}
\bar{R}_{t+1} & = \frac{1}{m}\sum_{i=1}^{m}\theta_{i, t+1},\text{  where}\\
\label{eq_r_dist_7_2}
\theta_{i, t+1} & = \theta_{i, t} + \alpha_{i, t}\left( 
 \tau_i - \mathds{1}_{ \{  R_{t+1} < \theta_{i,t} \} } \right), \; \forall i = 1,  2,  \ldots, m,
\end{align}
\end{subequations}
and \(\bar{R}_{t+1}\) denotes an estimate of the average-reward, \(\bar{r}_\pi\). Importantly, we can show that this quantile-based average-reward update converges to the average-reward of the limiting per-step reward distribution induced by a given policy:
\begin{theorem}
\label{theorem_1}
Given a unichain, communicating, or equivalent assumption, consider the limiting per-step reward distribution induced by following policy \(\pi\), and parameterized by the \(\tau\)-quantiles, \(\{\theta_i\}_{i=1}^{m}\), as per Equation \eqref{eq_r_dist_2} for \(\{\tau_i \in (0, 1)\}_{i=1}^{m}\), such that \(\tau_i = \frac{2i - 1}{2m}, \; i = 1, 2, \ldots, m\). Also consider the quantile-based update rules for the average-reward estimate \eqref{eq_r_dist_7}, along with corresponding step sizes that satisfy: \(\alpha_t > 0, \; \sup_{t \geq 0} \alpha_t < \infty, \; \sum_{t = 0}^\infty \alpha_t = \infty, \text{and}\; \alpha_t = o(1 / \text{log } t)\). If the quantile estimates, \(\{\theta_{i,t}\}_{i=1}^{m}\), converge a.s. to the \(\tau\)-quantiles of the limiting per-step reward distribution, then the average-reward estimate, \(\bar{R}_{t}\), converges a.s. to the average-reward of the limiting per-step reward distribution, \(\bar{r}_{\pi}\), as \(t \to \infty\).
\end{theorem}
\begin{proof}
If \(\{\theta_{i,t}\}_{i=1}^{m} \to \{\theta_i\}_{i=1}^{m}\) as \(t \to \infty\), then, by Lemma \ref{lemma_1}, it directly follows that the average of these estimates, which corresponds to \(\bar{R}_{t}\), converges, almost surely, to \(\bar{r}_{\pi}\) as \(t \to \infty\). This completes the proof.
\end{proof}
Hence, given the convergence of the \(\tau\)-quantile estimates, it is intuitive to see that the average-reward estimate will also converge. We provide a full convergence proof for the \(\tau\)-quantile estimates in Appendix B. 

As such, we now have a way to learn the limiting per-step reward distribution induced by a given policy. We can further extend these results into control-based settings by choosing an appropriate action selection rule. As in the discounted setting, we will utilize an action selection rule that is consistent with Definition \ref{defn_greedy_rule}. To this end, Algorithm \ref{alg_1} shows the Differential Distributional (or \emph{D2}) Q-learning algorithm, which adopts a greedy action selection rule with respect to the state-action value function. The full set of Differential Distributional RL algorithms is included in Appendix A.
\begin{algorithm}
   \caption{Differential Distributional (D2) Q-Learning}
   \label{alg_1}
\begin{algorithmic}
    \STATE Obtain initial \(S\)
    \WHILE{still time to train}
        \STATE \(A \leftarrow\) action given by \(\pi\) for \(S\)
        \STATE Take action \(A\), observe \(R, S'\)
        \FOR{$i = 1, 2, \ldots, m$}
            \STATE \(\theta_{i} = \theta_{i} + \alpha_{_\theta}\left( 
 \tau_i - \mathds{1}_{ \{  R < \theta_{i} \} } \right)\)
        \ENDFOR
        \STATE \(\bar{R} = \frac{1}{m}\sum_{i=1}^{m}\theta_{i}\)
        \STATE \(\delta = R - \bar{R} + \max_a Q(S', a) - Q(S, A)\)
        \STATE \(Q(S, A) = Q(S, A) + \alpha\delta\)
        \STATE \(S = S'\)
    \ENDWHILE
    \STATE return \(\{\theta_i\}_{i=1}^{m}\)
\end{algorithmic}
\end{algorithm}

Importantly, with the aforementioned action selection rule, we retain the regular (non-distributional) Q-learning operator for the average-reward setting: 
\begin{equation}
\label{eq_q_operator}
TQ(s, a) \doteq \sum_{s', r} p(s', r \; | \; s, a) (r + \max_{a'} Q(s', a')). 
\end{equation}
As such, we can use existing results to establish that the convergence of the state-action value estimates implies that the average-reward converges to the optimal average-reward:
\begin{theorem}
Let \(\bar{r}_{\pi_t}\) denote the average-reward obtained when following policy \(\pi_t\), such that \(\pi_t\) is a greedy policy with respect to the state-action value estimates, \(Q_t(s, a)\), where \(s \in \mathcal{S}\) and \(a \in \mathcal{A}\). Also, let \(q^{*}\) denote the solution of the Bellman optimality equation \eqref{eq_avg_reward_5}. If a communicating or equivalent assumption holds, such that there exists a unique \(q^{*}\) up to an additive constant, then we have that if \(Q_t \to q^{*}\) a.s. as \(t \to \infty\), then \(\bar{r}_{\pi_t} \to \bar{r}*\) a.s. as \(t \to \infty\).
\end{theorem}
\begin{proof}
We adopt the proof technique used in \citet{Wan2021-re} to prove a similar result. In particular, the desired result follows directly from Theorem 8.5.5 of \citet{Puterman1994-dq}. More specifically, we have that:
\begin{align}
    \nonumber 
    \min_{s, a} (TQ_t - Q_t) &\leq \bar{r}_{\pi_t} \\\nonumber &\leq \bar{r}*\\ & \leq \max_{s, a} (TQ_t - Q_t),\\
    \implies \vert \bar{r}* - \bar{r}_{\pi_t} \vert &\leq sp(TQ_t - Q_t),
\end{align}
where \(TQ(s, a)\) is defined in Equation \eqref{eq_q_operator}. Because \(Q_t \to q^{*}\) a.s. as \(t \to \infty\), and \(sp(TQ_t - Q_t)\) is clearly a continuous function of \(Q_t\), we have, by the continuous mapping theorem, that \(sp(TQ_t - Q_t) \to sp(Tq^{*} - q^{*}) = 0\) a.s. as \(t \to \infty\). Hence, we can conclude that \(\bar{r}_{\pi_t} \to \bar{r}*\) a.s. as \(t \to \infty\). This completes the proof.
\end{proof}

The above theorem requires that \(Q_t \to q^{*}\) a.s. as \(t \to \infty\). In this regard, we provide a full convergence proof for the convergence of \(Q_t\) in Appendix B. We also provide a full proof for the convergence of the value function estimates in the prediction setting (i.e., \(V_t\)) in Appendix B. Altogether, the theoretical results provided in this section and in Appendix B show the almost sure convergence of the tabular D2 algorithms for both prediction and control.

\begin{figure*}[htbp]
\centerline{\includegraphics[scale=0.57]{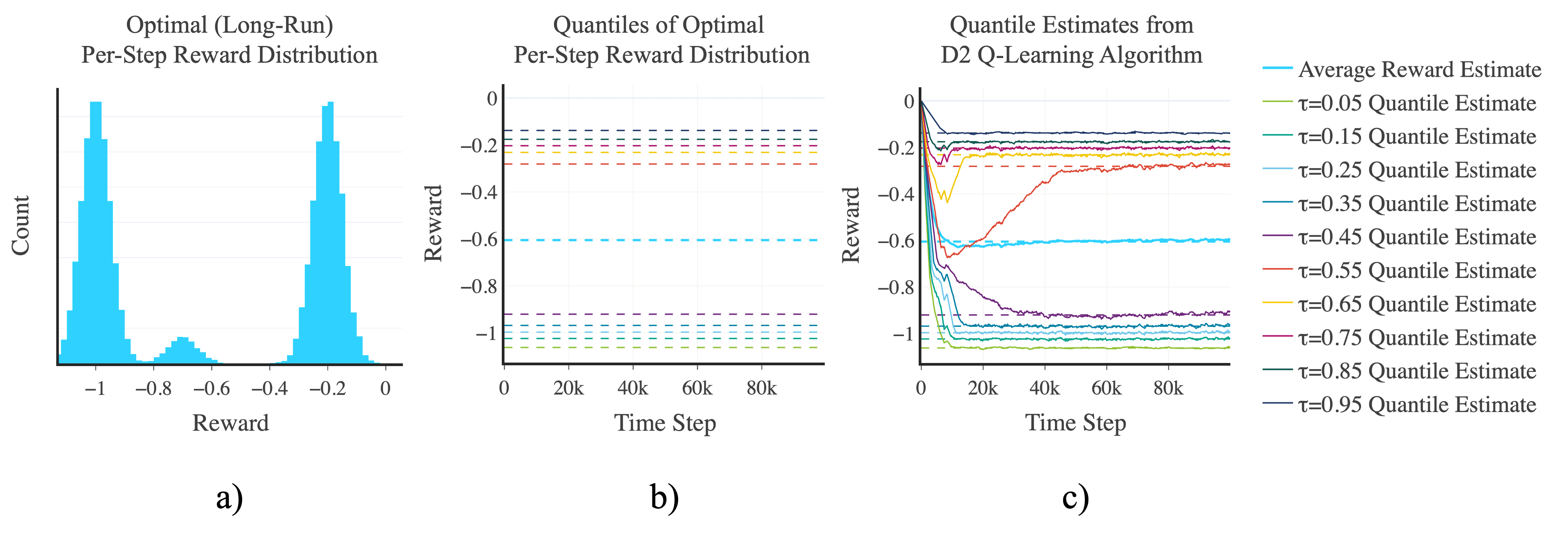}}
\caption{\textbf{a)} Histogram showing the empirical (\(\varepsilon\)-greedy) optimal (long-run) per-step reward distribution in the red-pill blue-pill task. \textbf{b)}  Quantiles of the optimal per-step reward distribution in the red-pill blue-pill task. \textbf{c)} Convergence plot of the per-step reward quantile estimates as learning progresses when using the D2 Q-learning algorithm in the red-pill blue-pill task.}
\label{fig_results_1}
\end{figure*}

\subsection{Learning the Differential Return Distribution}
\label{longrun_value_dist}
Although we argue in this work that the limiting per-step reward distribution is the natural distributional objective in the average-reward setting, it may still be useful to consider the differential return distribution from an empirical perspective (i.e., it could potentially yield superior empirical performance). As such, in this work we explore such an approach from a purely empirical perspective, and propose a set of algorithms that parameterize the probability distribution over differential returns, as well as the limiting per-step reward distribution. We call the resulting set of algorithms, \emph{Double Differential Distributional}, or \emph{D3}, algorithms because they simultaneously learn and/or optimize both distributions. 

To this end, when we apply an analogous framework to the one described in Section \ref{section_discounted}, such that \(\psi_t = R_{t+1} - \bar{R}_t + \Omega_{j, t}(S_{t+1}, a^{*}) - \Omega_{i, t}(S_t, A_t)\), and incorporate the quantile-based average-reward update \eqref{eq_r_dist_7}, we arrive at the D3 Q-learning algorithm (\ref{alg_2}). The full set of Double Differential Distributional RL algorithms is included in Appendix A.
\begin{algorithm}
   \caption{D3 Q-Learning}
   \label{alg_2}
\begin{algorithmic}
    \STATE Obtain initial \(S\)
    \WHILE{still time to train}
        \STATE \(A \leftarrow\) action given by \(\pi\) for \(S\)
        \STATE Take action \(A\), observe \(R, S'\)
        \STATE \(\theta_{i} = \theta_{i} + \alpha_{_\theta}\left( 
\tau_i - \mathds{1}_{ \{  R < \theta_{i} \} } \right), \, \forall i = 1, 2, \ldots, m\)
        \STATE \(\bar{R} = \frac{1}{m}\sum_{i=1}^{m}\theta_{i}\)
        \STATE \(a^{*} = \text{argmax}_{a'}{\frac{1}{n}\sum_{j = 1}^{n}\Omega_j(S', a')}\)
        \FOR{$j = 1, 2, \ldots, n$}
        \STATE \(\beta = \frac{1}{n}\sum_{k=1}^{n}\Big[\tau_j - \mathds{1}_{ \{  R - \bar{R} + \Omega_k(S',a^{*}) - \Omega_j(S, A)  < 0 \} } \Big]\)
        \(\Omega_j(S, A) = \Omega_j(S, A) + \alpha\beta\)
        \ENDFOR
        \STATE \(S = S'\)
    \ENDWHILE
    \STATE return \(\{\theta_i\}_{i=1}^{m}\) and \(\{\Omega_j\}_{j=1}^{n}\)
\end{algorithmic}
\end{algorithm}

\begin{figure}[htbp]
\centerline{\includegraphics[scale=0.49]{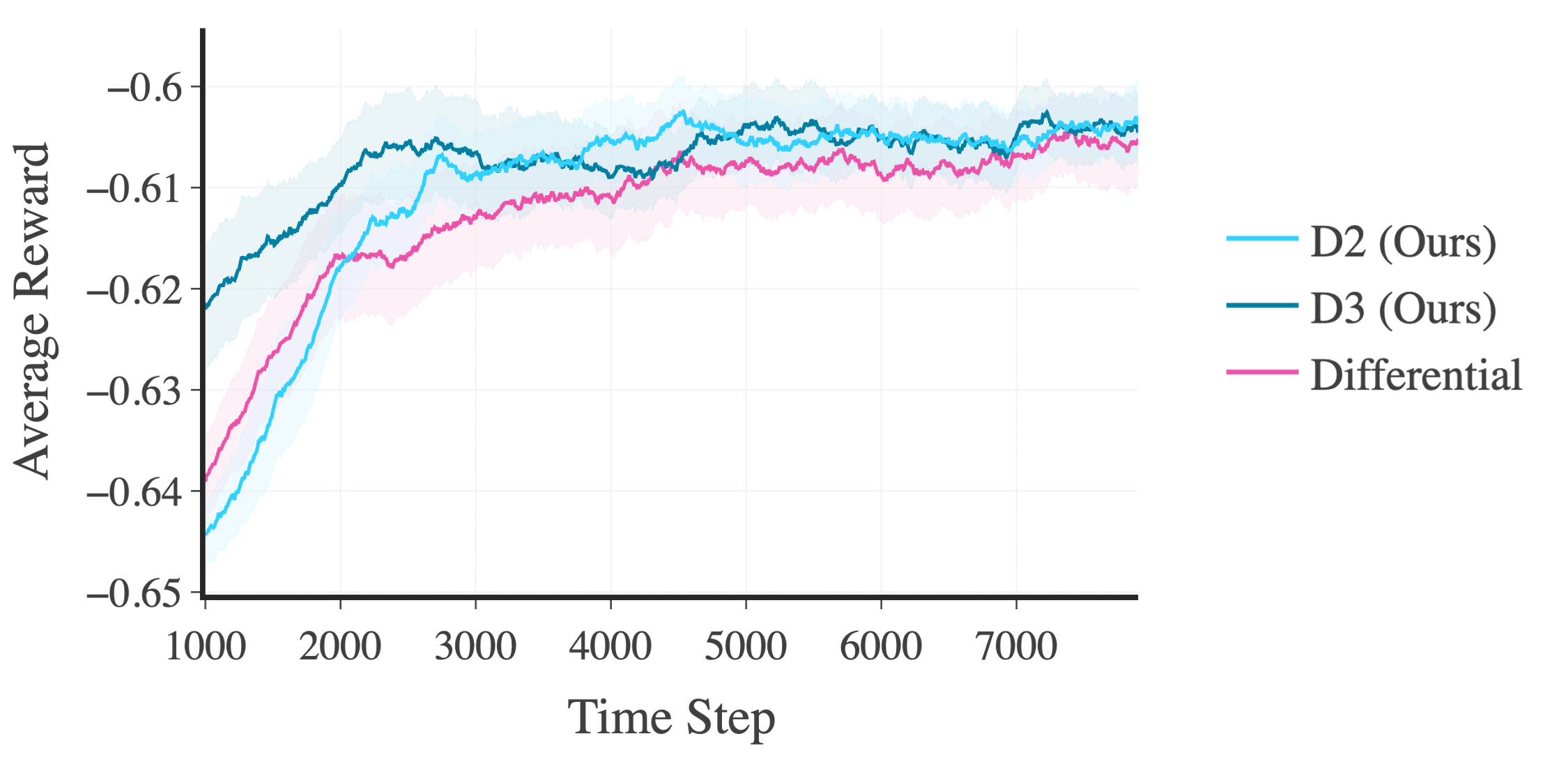}}
\caption{Rolling average-reward when using the D2 and D3 algorithms vs. a non-distributional Differential algorithm in the red-pill blue-pill environment. A solid line denotes the mean average-reward, and the corresponding shaded region denotes a 95\% confidence interval over 50 runs.}
\label{fig_results_2}
\end{figure}

\section{Experimental Results}
\label{experiments}
In this section, we present empirical results obtained when applying our D2 and D3 algorithms on two groups of experiments. In the first group of experiments, we aimed to validate whether the D2 and D3 algorithms could successfully learn the optimal per-step reward distribution by conducting experiments in environments where the optimal per-step reward distribution is known. In the second group of experiments, we aimed to evaluate the empirical performance of the D2 and D3 algorithms in more difficult environments. In both groups, we compared the performance of our algorithms to that of non-distributional Differential algorithms derived from the framework proposed in \citet{Wan2021-re}. The full set of experimental details and results, including additional experiments performed, is provided in Appendix C. Below, we highlight the key results.  

\begin{figure*}[htbp]
\centerline{\includegraphics[scale=0.54]{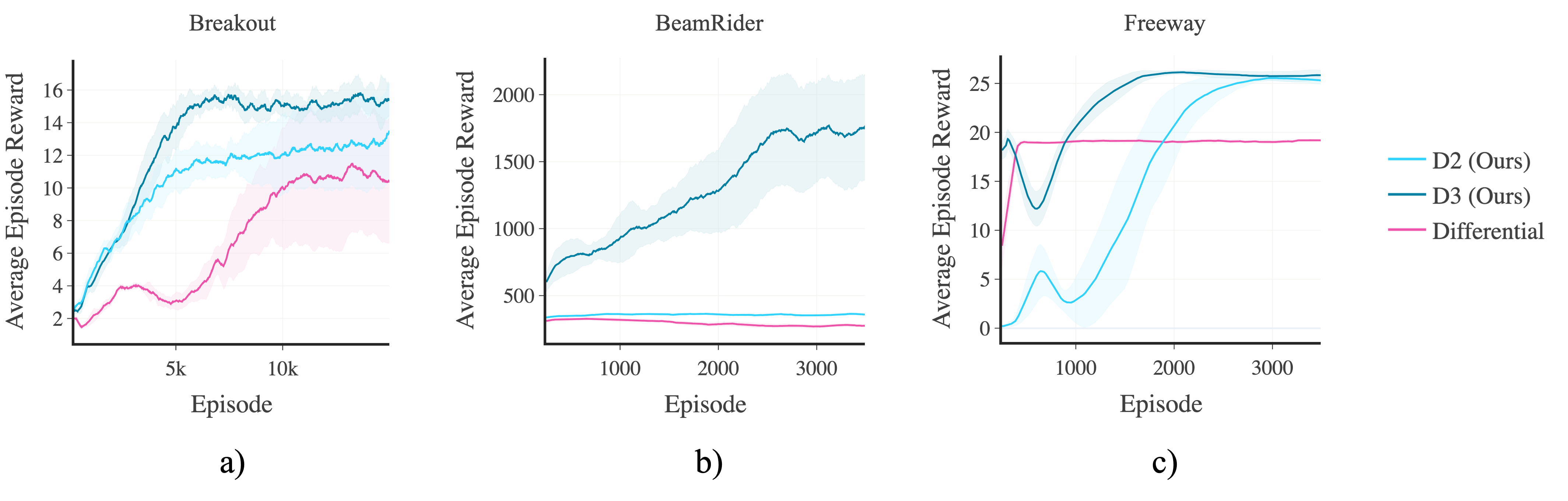}}
\caption{Rolling averages of the total reward per episode when using the D2 and D3 algorithms vs. non-distributional Differential algorithms in the \textbf{a)} \emph{Breakout}, \textbf{b)} \emph{BeamRider}, and \textbf{c)} \emph{Freeway} Atari 2600 environments. A solid line denotes the mean total reward per episode, and the corresponding shaded region denotes a 95\% confidence interval over 8 runs.}
\label{fig_results_3}
\end{figure*}

In terms of empirical results from the first group of experiments, Figure \ref{fig_results_1} shows the agent's (per-step) reward quantile estimates as learning progresses when using the D2 Q-learning algorithm in the \emph{red-pill blue-pill} environment \citep{Rojas2025-bf}. As shown in the figure, the agent's quantile estimates converge to the quantiles of the limiting per-step reward distribution induced by the optimal policy, thereby empirically showing that the D2 algorithm not only converges, but converges to the optimal solution. 

Subsequently, Figure \ref{fig_results_2} shows the rolling average-reward as learning progresses when using the D2 and D3 algorithms vs. a non-distributional Differential algorithm in the red-pill blue-pill environment. As shown in the figure, the D2 and D3 algorithms yield competitive performance when compared to their non-distributional counterpart, while also capturing rich information about the long-run reward distribution (i.e., as shown in Figure \ref{fig_results_1} for the D2 algorithm). 

In terms of empirical results from the second group of experiments, Figure \ref{fig_results_3} shows rolling averages of the total reward per episode as learning progresses when using the D2 and D3 algorithms vs. a non-distributional Differential algorithm in various Atari 2600 games from the \emph{Arcade Learning Environment} \citep{Bellemare2013-ef}. As shown in the figure, the D2 and D3 algorithms consistently outperform the non-distributional baseline, with the D3 algorithm showing better performance than the D2 algorithm.

Our motivation for testing our algorithms in the Arcade Learning Environment (ALE) is twofold. First, ALE is a standard benchmark for evaluating distributional RL algorithms in the discounted setting, thereby making it a potentially-effective tool to use when one aims to measure the relative performance gains of distributional methods over their non-distributional counterparts. Second, there are currently no widely adopted benchmarks for the average-reward setting that match the scale and diversity of ALE. As such, ALE is one of the closest approximations we have for evaluating our algorithms under complex, high-dimensional environments. In this regard, as per Figure \ref{fig_results_3}, we find that our algorithms fare better than the non-distributional baselines in these complex environments that do not strictly satisfy the theoretical assumptions of the average-reward setting.

\section{Discussion, Limitations, and Future Work}
\label{discussion}

In this work, we introduced the first distributional RL algorithms specifically designed for the average-reward setting. In particular, we motivated an appropriate distributional objective for the average-reward setting, as well as derived two quantile-based approaches that can learn and/or optimize this objective. We showed, both theoretically and empirically, that these \emph{Differential Distributional} algorithms are able to learn the average-reward-optimal policy, as well as the corresponding (optimal) distributional objective. 

In terms of empirical performance, we showed that our algorithms are able to achieve competitive and sometimes superior performance when compared to non-distributional differential algorithms, while also capturing rich information about the long-run reward and return distributions. 

Moreover, we note that our choice of distributional objective for the average-reward setting enables the resulting D2 algorithms to be more \emph{scalable} in nature in comparison to distributional RL algorithms in the discounted setting. In particular, the D2 algorithms only require \(m\) parameters to parameterize the distributional objective, where \(m\) is the number of quantiles. This is in contrast to discounted distributional algorithms, which require \(\vert \mathcal{S} \vert \times \vert \mathcal{A} \vert \times m\) parameters to parameterize the discounted distributional objective. As such, the computational complexity of our D2 algorithms remains constant with respect to the state and action-space sizes, thereby making them more scalable.

In terms of limitations (in the context of the average-reward setting), it remains to be seen how different aspects (beyond the mean) of the chosen distributional objective can be optimized. Similarly, it remains to be seen how a categorical approach could be employed instead of a quantile-based approach. Future work should look to address these limitations, as well as tackle the theoretical aspects of the D3 algorithms. Exploring these directions may yield deeper insights, and in the process, continue unlocking the full potential of distributional approaches in the average-reward setting.

\section*{Acknowledgments}
We gratefully acknowledge funding from NSERC Discovery Grant \# RGPIN-2021-02760. We are also grateful for the computing resources provided by the Digital Research Alliance of Canada. We thank the anonymous AAAI reviewers and area chair for their useful feedback and commentary during the review process.

\bibliography{references}

\input{appendix}

\end{document}

%% file: appendix.tex
\newpage
\appendix
\onecolumn
\numberwithin{equation}{section}
\numberwithin{figure}{section}
\numberwithin{table}{section}
\numberwithin{theorem}{subsection}
\sloppy 

\begin{center}
    {\huge \textbf{Technical Appendix}}
\end{center}

\section{Differential Distributional RL Algorithms}
\label{appendix_algs}

In this appendix, we provide pseudocode for our Differential Distributional RL algorithms.\\

We note that, as is common practice in the discounted setting, we adopt the smooth L1 loss, \(L^{\lambda}(x)\), in the D2 algorithms that utilize function approximation with a replay buffer (that is, Algorithm \ref{alg_6}):
\begin{equation}
\label{eq_smoothl1}
L^{\lambda}(x) = 
\begin{cases}
\frac{1}{2\lambda}x^2, & \text{if } |x| \leq \lambda \\
|x| - \frac{1}{2}\lambda, & \text{if } |x| > \lambda \text{.}
\end{cases}
\end{equation}

Similarly, we adopt the quantile Huber loss \citep{Dabney2018-ta}, \(h_{\tau}^{\lambda}(x)\), in the D3 algorithms that utilize function approximation (that is, Algorithms \ref{alg_9} and \ref{alg_10}):
\begin{equation}
\label{eq_huber}
h_{\tau}^{\lambda}(x) = 
\begin{cases}
|\tau - \mathds{1}{\{x < 0\}}|\frac{1}{2}x^2, & \text{if } |x| \leq \lambda \\
|\tau - \mathds{1}{\{x < 0\}}|\lambda\left(|x| - \frac{1}{2}\lambda\right), & \text{if } |x| > \lambda \text{.}
\end{cases}
\end{equation}

\begin{algorithm}
   \caption{Differential Distributional (D2) TD-Learning (Tabular)}
   \label{alg_3}
\begin{algorithmic}
    \STATE {\bfseries Input:} the policy \(\pi\) to be used (e.g., \(\varepsilon\)-greedy), number of quantiles, \(m\), and \(\tau\)-locations, such that \(\tau_i = \frac{2i - 1}{2m}\)
    \STATE {\bfseries Algorithm parameters:} step size parameters \(\alpha\), \(\alpha_{_\theta}\)
    \STATE Initialize \(V(s) \: \forall s\) arbitrarily (e.g. to zero)
    \STATE Initialize quantiles \(\theta_1, \theta_2, \ldots, \theta_m\) arbitrarily (e.g. to zero)
    \STATE Obtain initial \(S\)
    \WHILE{still time to train}
        \STATE \(A \leftarrow\) action given by \(\pi\) for \(S\)
        \STATE Take action \(A\), observe \(R, S'\)
        \FOR{$i = 1, 2, \ldots, m$}
            \STATE \(\theta_{i} = \theta_{i} + \alpha_{_\theta}\left( 
 \tau_i - \mathds{1}{ \{  R < \theta_{i} \} } \right)\)
        \ENDFOR
        \STATE \(\bar{R} = \frac{1}{m}\sum_{i=1}^{m}\theta_{i}\)
        \STATE \(\delta = R - \bar{R} + V(S') - V(S)\)
        \STATE \(V(S) = V(S) + \alpha\delta\)
        \STATE \(S = S'\)
    \ENDWHILE
    \STATE return \(\{\theta_i\}_{i=1}^{m}\)
\end{algorithmic}
\end{algorithm}

\begin{algorithm}
   \caption{Differential Distributional (D2) Q-Learning (Tabular)}
   \label{alg_4}
\begin{algorithmic}
    \STATE {\bfseries Input:} the policy \(\pi\) to be used (e.g., \(\varepsilon\)-greedy), number of quantiles, \(m\), and \(\tau\)-locations, such that \(\tau_i = \frac{2i - 1}{2m}\)
    \STATE {\bfseries Algorithm parameters:} step size parameters \(\alpha\), \(\alpha_{_\theta}\)
    \STATE Initialize \(Q(s, a) \: \forall s, a\) arbitrarily (e.g. to zero)
    \STATE Initialize quantiles \(\theta_1, \theta_2, \ldots, \theta_m\) arbitrarily (e.g. to zero)
    \STATE Obtain initial \(S\)
    \WHILE{still time to train}
        \STATE \(A \leftarrow\) action given by \(\pi\) for \(S\)
        \STATE Take action \(A\), observe \(R, S'\)
        \FOR{$i = 1, 2, \ldots, m$}
            \STATE \(\theta_{i} = \theta_{i} + \alpha_{_\theta}\left( 
 \tau_i - \mathds{1}{ \{  R < \theta_{i} \} } \right)\)
        \ENDFOR
        \STATE \(\bar{R} = \frac{1}{m}\sum_{i=1}^{m}\theta_{i}\)
        \STATE \(\delta = R - \bar{R} + \max_a Q(S', a) - Q(S, A)\)
        \STATE \(Q(S, A) = Q(S, A) + \alpha\delta\)
        \STATE \(S = S'\)
    \ENDWHILE
    \STATE return \(\{\theta_i\}_{i=1}^{m}\)
\end{algorithmic}
\end{algorithm}

\begin{algorithm}
   \caption{Differential Distributional (D2) Actor-Critic}
   \label{alg_5}
\begin{algorithmic}
    \STATE {\bfseries Input:} a differentiable state-value function parameterization \(\hat{v}(s, \boldsymbol{w})\), a differentiable policy parameterization \(\pi(a \mid s, \boldsymbol{u})\), number of quantiles, \(m\), and \(\tau\)-locations, such that \(\tau_i = \frac{2i - 1}{2m}\)
    \STATE {\bfseries Algorithm parameters:} step size parameters \(\alpha\), \(\alpha_{\pi}\), \(\alpha_{_\theta}\)
    \STATE Initialize state-value weights \(\boldsymbol{w} \in \mathbb{R}^{d}\) and policy weights \(\boldsymbol{u} \in \mathbb{R}^{d'}\) (e.g. to \(\boldsymbol{0}\))
    \STATE Initialize quantiles \(\theta_1, \theta_2, \ldots, \theta_m\) arbitrarily (e.g. to zero)
    \STATE Obtain initial \(S\)
    \WHILE{still time to train}
        \STATE \(A \sim \pi(\cdot \mid S, \boldsymbol{u})\)
        \STATE Take action \(A\), observe \(R, S'\)
        \FOR{$i = 1, 2, \ldots, m$}
            \STATE \(\theta_{i} = \theta_{i} + \alpha_{_\theta}\left( 
 \tau_i - \mathds{1}{ \{  R < \theta_{i} \} } \right)\)
        \ENDFOR
        \STATE \(\bar{R} = \frac{1}{m}\sum_{i=1}^{m}\theta_{i}\)
        \STATE \(\delta = R - \bar{R} + \hat{v}(S', \boldsymbol{w}) - \hat{v}(S, \boldsymbol{w})\)
        \STATE \(\boldsymbol{w} = \boldsymbol{w} + \alpha\delta\nabla\hat{v}(S, \boldsymbol{w})\)
        \STATE \(\boldsymbol{u} = \boldsymbol{u} + \alpha_{\pi}\delta\nabla \text{ln} \pi(A \mid S, \boldsymbol{u})\)
        \STATE \(S = S'\)
    \ENDWHILE
    \STATE return \(\{\theta_i\}_{i=1}^{m}\)
\end{algorithmic}
\end{algorithm}

\begin{algorithm}
   \caption{Differential Distributional (D2) Q-Learning (Function Approximation with Replay Buffer)}
   \label{alg_6}
\begin{algorithmic}
    \STATE {\bfseries Input:} a differentiable state-action value function parameterization: \(\hat{q}(s, a, \boldsymbol{w})\) (with target network \(\hat{q}_{_T}(s, a, \boldsymbol{w_{_T}})\)), the policy \(\pi\) to be used (e.g., \(\varepsilon\)-greedy), number of quantiles, \(m\), and \(\tau\)-locations, such that \(\tau_i = \frac{2i - 1}{2m}\)
    \STATE {\bfseries Algorithm parameters:} step size parameters \(\{\alpha\), \(\alpha_{_\theta}\}\), smooth L1 loss parameter \(\lambda\)
    \STATE Initialize state-action value weights \(\boldsymbol{w}, \boldsymbol{w_{_T}} \in \mathbb{R}^{d}\) arbitrarily (e.g. to \(\boldsymbol{0}\))    
    \STATE Initialize quantiles \(\theta_1, \theta_2, \ldots, \theta_m\) arbitrarily (e.g. to zero)
    \STATE Obtain initial \(S\)
    \WHILE{still time to train}
        \STATE \(A \leftarrow\) action given by \(\pi\) for \(S\)
        \STATE Take action \(A\), observe \(R, S'\)
        \STATE Store \((S, A, R, S')\) in replay buffer
        \IF {time to update estimates}
        \STATE Sample a minibatch of \(B\) transitions from replay buffer: \(\{(S_b, A_b, R_b, S_b')\}_{b=1}^{B}\)
            \FOR{$i = 1, 2, \ldots, m$}
                    \STATE \(\theta_{i} = \theta_{i} + \alpha_{_\theta}\left( 
         \tau_i - \mathds{1}{ \{  \frac{1}{B}\sum_{b=1}^{B}R_b < \theta_{i} \} } \right)\)
            \ENDFOR
            \STATE \(\bar{R} = \frac{1}{m}\sum_{i=1}^{m}\theta_{i}\)
            \STATE \(\ell = -\frac{1}{B}\sum_{b=1}^{B}\Big[L^{ \lambda}\Big(R_b - \bar{R} + \max_a \hat{q}_{_T}(S_b', a, \boldsymbol{w_{_T}}) - \hat{q}(S_b, A_b, \boldsymbol{w})\Big)\Big]\) (See Equation \eqref{eq_smoothl1})
            \STATE \(\boldsymbol{w} = \boldsymbol{w} + \alpha\frac{\partial \ell}{\partial \boldsymbol{w}}\)
            \STATE Update \(\boldsymbol{w_{_T}}\) as needed (e.g. using Polyak averaging)
        \ENDIF
        \STATE \(S = S'\)
    \ENDWHILE
    \STATE return \(\{\theta_i\}_{i=1}^{m}\)
\end{algorithmic}
\end{algorithm}

\begin{algorithm}
   \caption{Double Differential Distributional (D3) TD-Learning (Tabular)}
   \label{alg_7}
\begin{algorithmic}
    \STATE {\bfseries Input:} the policy \(\pi\) to be used (e.g., \(\varepsilon\)-greedy), number of per-step reward quantiles, \(m\), number of differential return quantiles, \(n\), and \(\tau\)-locations, such that \(\tau_i = \frac{2i - 1}{2m}\) and \(\tau_j = \frac{2j - 1}{2n}\)
    \STATE {\bfseries Algorithm parameters:} step size parameters \(\alpha\), \(\alpha_{_\theta}\)
    \STATE Initialize \(\Omega_j(s) \: \forall s,j\) arbitrarily (e.g. to zero)
    \STATE Initialize per-step reward quantiles \(\theta_1, \theta_2, \ldots, \theta_m\) arbitrarily (e.g. to zero)
    \STATE Obtain initial \(S\)
    \WHILE{still time to train}
        \STATE \(A \leftarrow\) action given by \(\pi\) for \(S\)
        \STATE Take action \(A\), observe \(R, S'\)
        \FOR{$i = 1, 2, \ldots, m$}
            \STATE \(\theta_{i} = \theta_{i} + \alpha_{_\theta}\left( 
 \tau_i - \mathds{1}{ \{  R < \theta_{i} \} } \right)\)
        \ENDFOR
        \STATE \(\bar{R} = \frac{1}{m}\sum_{i=1}^{m}\theta_{i}\)
        \FOR{$j = 1, 2, \ldots, n$}
           \STATE \(\Omega_j(S) = \Omega_j(S) + \alpha\frac{1}{n}\sum_{k=1}^{n}\Big[\tau_j - \mathds{1}{ \{  R - \bar{R} + \Omega_k(S') - \Omega_j(S)  < 0 \} } \Big]\)
        \ENDFOR
        \STATE \(S = S'\)
    \ENDWHILE
    \STATE return \(\{\theta_i\}_{i=1}^{m}\) and \(\{\Omega_j\}_{j=1}^{n}\)
\end{algorithmic}
\end{algorithm}

\begin{algorithm}
   \caption{Double Differential Distributional (D3) Q-Learning (Tabular)}
   \label{alg_8}
\begin{algorithmic}
    \STATE {\bfseries Input:} the policy \(\pi\) to be used (e.g., \(\varepsilon\)-greedy), number of per-step reward quantiles, \(m\), number of differential return quantiles, \(n\), and \(\tau\)-locations, such that \(\tau_i = \frac{2i - 1}{2m}\) and \(\tau_j = \frac{2j - 1}{2n}\)
    \STATE {\bfseries Algorithm parameters:} step size parameters \(\alpha\), \(\alpha_{_\theta}\)
    \STATE Initialize \(\Omega_j(s, a) \: \forall s, a, j\) arbitrarily (e.g. to zero)
    \STATE Initialize per-step reward quantiles \(\theta_1, \theta_2, \ldots, \theta_m\) arbitrarily (e.g. to zero)
    \STATE Obtain initial \(S\)
    \WHILE{still time to train}
        \STATE \(A \leftarrow\) action given by \(\pi\) for \(S\)
        \STATE Take action \(A\), observe \(R, S'\)
        \FOR{$i = 1, 2, \ldots, m$}
            \STATE \(\theta_{i} = \theta_{i} + \alpha_{_\theta}\left( 
 \tau_i - \mathds{1}{ \{  R < \theta_{i} \} } \right)\)
        \ENDFOR
        \STATE \(\bar{R} = \frac{1}{m}\sum_{i=1}^{m}\theta_{i}\)
        \STATE \(a^{*} = \text{argmax}_{a'}{\frac{1}{n}\sum_{j = 1}^{n}\Omega_j(S', a')}\)
        \FOR{$j = 1, 2, \ldots, n$}
           \STATE \(\Omega_j(S, A) = \Omega_j(S, A) + \alpha\frac{1}{n}\sum_{k=1}^{n}\Big[\tau_j - \mathds{1}{ \{  R - \bar{R} + \Omega_k(S',a^{*}) - \Omega_j(S, A)  < 0 \} } \Big]\)
        \ENDFOR
        \STATE \(S = S'\)
    \ENDWHILE
    \STATE return \(\{\theta_i\}_{i=1}^{m}\) and \(\{\Omega_j\}_{j=1}^{n}\)
\end{algorithmic}
\end{algorithm}

\begin{algorithm}
   \caption{Double Differential Distributional (D3) Actor-Critic}
   \label{alg_9}
\begin{algorithmic}
    \STATE {\bfseries Input:} a differentiable differential return quantile parameterization: \(\hat{\Omega}(s, j, \boldsymbol{w})\), a differentiable policy parameterization \(\pi(a \mid s, \boldsymbol{u})\), number of per-step reward quantiles, \(m\), number of differential return quantiles, \(n\), and \(\tau\)-locations, such that \(\tau_i = \frac{2i - 1}{2m}\) and \(\tau_j = \frac{2j - 1}{2n}\)
    \STATE {\bfseries Algorithm parameters:} step size parameters \(\{\alpha\), \(\alpha_{\pi}\), \(\alpha_{_\theta}\}\), quantile Huber loss parameter \(\lambda\)
    \STATE Initialize \(\boldsymbol{w} \in \mathbb{R}^{d}\) and \(\boldsymbol{u} \in \mathbb{R}^{d'}\) (e.g. to \(\boldsymbol{0}\))
    \STATE Initialize per-step reward quantiles \(\theta_1, \theta_2, \ldots, \theta_m\) arbitrarily (e.g. to zero)
    \STATE Obtain initial \(S\)
    \WHILE{still time to train}
        \STATE \(A \sim \pi(\cdot \mid S, \boldsymbol{u})\)
        \STATE Take action \(A\), observe \(R, S'\)
        \FOR{$i = 1, 2, \ldots, m$}
            \STATE \(\theta_{i} = \theta_{i} + \alpha_{_\theta}\left( 
 \tau_i - \mathds{1}{ \{  R < \theta_{i} \} } \right)\)
        \ENDFOR
        \STATE \(\bar{R} = \frac{1}{m}\sum_{i=1}^{m}\theta_{i}\)
        \STATE \(\delta = R - \bar{R} + \mathbb{E}_j[\hat{\Omega}(S', j, \boldsymbol{w})] - \mathbb{E}_j[\hat{\Omega}(S, j, \boldsymbol{w})]\)
        \STATE \(\ell = -\sum_{j=1}^{n}\mathbb{E}_k\Big[h_{\tau_j}^{ \lambda}\Big(R - \bar{R} + \hat{\Omega}(S', k, \boldsymbol{w}) - \hat{\Omega}(S, j, \boldsymbol{w})\Big)\Big]\) (See Equation \eqref{eq_huber})
        \STATE \(\boldsymbol{w} = \boldsymbol{w} + \alpha\frac{\partial \ell}{\partial \boldsymbol{w}}\)
        \STATE \(\boldsymbol{u} = \boldsymbol{u} + \alpha_{\pi}\delta\nabla \text{ln} \pi(A \mid S, \boldsymbol{u})\)
        \STATE \(S = S'\)
    \ENDWHILE
    \STATE return \(\{\theta_i\}_{i=1}^{m}\) and \(\boldsymbol{w}\)
\end{algorithmic}
\end{algorithm}

\begin{algorithm}
   \caption{Double Differential Distributional (D3) Q-Learning (Function Approximation with Replay Buffer)}
   \label{alg_10}
\begin{algorithmic}
    \STATE {\bfseries Input:} a differentiable differential return quantile parameterization: \(\hat{\Omega}(s, a, j, \boldsymbol{w})\) (with target network \(\hat{\Omega}_{_T}(s, a, j, \boldsymbol{w_{_T}})\)), the policy \(\pi\) to be used (e.g., \(\varepsilon\)-greedy), number of per-step reward quantiles, \(m\), number of differential return quantiles, \(n\), and \(\tau\)-locations, such that \(\tau_i = \frac{2i - 1}{2m}\) and \(\tau_j = \frac{2j - 1}{2n}\)
    \STATE {\bfseries Algorithm parameters:} step size parameters \(\{\alpha\), \(\alpha_{_\theta}\}\), quantile Huber loss parameter \(\lambda\)
    \STATE Initialize \(\boldsymbol{w}, \boldsymbol{w_{_T}} \in \mathbb{R}^{d}\) arbitrarily (e.g. to \(\boldsymbol{0}\))    
    \STATE Initialize per-step reward quantiles \(\theta_1, \theta_2, \ldots, \theta_m\) arbitrarily (e.g. to zero)
    \STATE Obtain initial \(S\)
    \WHILE{still time to train}
        \STATE \(A \leftarrow\) action given by \(\pi\) for \(S\)
        \STATE Take action \(A\), observe \(R, S'\)
        \STATE Store \((S, A, R, S')\) in replay buffer
        \IF {time to update estimates}
        \STATE Sample a minibatch of \(B\) transitions from replay buffer: \(\{(S_b, A_b, R_b, S_b')\}_{b=1}^{B}\)
        \FOR{$i = 1, 2, \ldots, m$}
            \STATE \(\theta_{i} = \theta_{i} + \alpha_{_\theta}\left( 
         \tau_i - \mathds{1}{ \{  \frac{1}{B}\sum_{b=1}^{B}R_b < \theta_{i} \} } \right)\)
        \ENDFOR
        \STATE \(\bar{R} = \frac{1}{m}\sum_{i=1}^{m}\theta_{i}\)
        \STATE For each \(b\)-th transition: \(a_b^{*} = \text{argmax}_{a'}{\frac{1}{n}\sum_{j = 1}^{n}\hat{\Omega}(S_b', a', j, \boldsymbol{w})}\)
        \STATE \(\ell = -\frac{1}{B}\sum_{b=1}^{B}\Big[\sum_{j=1}^{n}\mathbb{E}_k\Big[h_{\tau_j}^{ \lambda}\Big(R_b - \bar{R} + \hat{\Omega}_{_T}(S_b', a_b^{*}, k, \boldsymbol{w_{_T}}) - \hat{\Omega}(S_b, A_b, j, \boldsymbol{w})\Big)\Big]\Big]\) (See Equation \eqref{eq_huber})
        \STATE \(\boldsymbol{w} = \boldsymbol{w} + \alpha\frac{\partial \ell}{\partial \boldsymbol{w}}\)
        \STATE Update \(\boldsymbol{w_{_T}}\) as needed (e.g. using Polyak averaging)
        \ENDIF
        \STATE \(S = S'\)
    \ENDWHILE
    \STATE return \(\{\theta_i\}_{i=1}^{m}\) and \(\boldsymbol{w}\)
\end{algorithmic}
\end{algorithm}


\clearpage

\section{Convergence Proofs}
\label{appendix_proofs}

In this appendix, we present the full convergence proofs for the \emph{per-step reward} quantile estimates and value function estimates of our tabular Differential Distributional (D2) algorithms (Algorithms \ref{alg_3} and \ref{alg_4}). Our general strategy is as follows: we first show that the results from \citet{Rowland2024-sg}, which show the almost sure convergence of the \emph{discounted return} quantile estimates of discounted distributional algorithms, can be used to show the convergence of the per-step reward quantile estimates of our algorithms. We then leverage a two-timescales argument to show the almost sure convergence of the value function estimates. In this regard, Section \ref{section_proof_q} contains the proof for the per-step reward quantile estimates, and Section \ref{section_proof_values} contains the proof for the value function estimates.

\subsection{Convergence of the Per-Step Reward Quantile Estimates}
\label{section_proof_q}

This section contains the proof for the per-step reward quantile estimates. Our general strategy is as follows: we show that the results from \citet{Rowland2024-sg}, which show the almost sure convergence of the \emph{discounted return} quantile estimates of discounted distributional algorithms, can be used to show the almost sure convergence of the \emph{per-step reward} quantile estimates of our D2 algorithms. 

We begin by formally defining our distributional objective. In particular, for a given policy, \(\pi\), let
\begin{equation}
\label{eq_r_dist_appendix}
\phi_{\pi}(s) \doteq \lim_{t \rightarrow{} \infty} \mathbb{P}(R_t \mid S_0 = s, A_{0:t-1} \sim \pi)
\end{equation}
denote the limiting per-step reward distribution induced by following policy \(\pi\), where \(R_t \in \mathcal{R}\) denotes the per-step reward. As is standard practice with the regular average-reward objective, we can simplify Equation \eqref{eq_r_dist_appendix} by making certain assumptions about the Markov chain induced by following policy \(\pi\). To this end, we will utilize a \emph{unichain} assumption when doing prediction (learning), because it ensures the existence of a unique limiting distribution of states that is independent of the initial state, such that \(\phi_{\pi}(s) = \phi_{\pi}\). Similarly, we will utilize a \emph{communicating} assumption when doing control (optimization), because it ensures the existence of a unique optimal average-reward, along with a corresponding per-step reward distribution that is independent of the initial state. Both assumptions are formally listed below:
\begin{assumption}[Unichain Assumption for Prediction]\label{assumption_unichain}
The Markov chain induced by the policy is unichain. That is, the induced Markov chain consists of a single recurrent class and a potentially-empty set of transient states.
\end{assumption}
\begin{assumption}[Communicating Assumption for Control] \label{assumption_communicating}
The MDP has a single communicating class. That is, each state in the MDP is accessible from every other state under some deterministic stationary policy.
\end{assumption}
 
Next, we consider the \emph{discounted return} quantile update rule considered in \citet{Rowland2024-sg}:
\begin{equation}
\label{eqn_proofs_1}
\Omega^{\gamma}_{i, t+1} = \Omega^{\gamma}_{i, t} + \alpha_t\frac{1}{n}\sum_{j=1}^{n}\left[ 
 \tau_i - \mathds{1}_{ \{  \psi_t < 0 \} } \right], \; \forall i = 1,  2,  \ldots, n,
\end{equation}
where \(\Omega^{\gamma}_{i, t} \) denotes the estimate of the \(\tau_i\)-quantile of the discounted return distribution at time \(t\), \(\alpha_t\) denotes the step size for the update, \(\psi_t = R_{t+1} + \gamma \Omega^{\gamma}_{j, t}(S_{t+1}) - \Omega^{\gamma}_{i, t}(S_t)\) (\citet{Rowland2024-sg} only considered the prediction case), and \(\{\tau_i \in (0, 1)\}_{i=1}^{n}\), such that \(\tau_i = \frac{2i - 1}{2n}, \; i = 1, 2, \ldots, n\).

To show the convergence of the quantile estimates obtained via the update rule \eqref{eqn_proofs_1}, \citet{Rowland2024-sg} argued that the update rule \eqref{eqn_proofs_1} was a specific instance of the generic update rule:
\begin{equation}
\label{eqn_proofs_2}
\theta_{t+1} = \theta_{t} + \alpha_t\left(g(\theta_{t}) + M_{t} \right),
\end{equation}
where \(M_t\) denotes some noise (or error) term, and \(g(\theta_t) \doteq g_t \in H(\theta_t)\), where \(H\) denotes a \emph{Marchaud map} (we will formally define these terms later on). In particular, \citet{Rowland2024-sg} proved the convergence of the generic update rule \eqref{eqn_proofs_2}, and then showed that the discounted return quantile update rule \eqref{eqn_proofs_1} is a specific instance of the generic update rule \eqref{eqn_proofs_2} and satisfied the assumptions required for convergence.

Let us now consider the \emph{per-step reward} quantile update rule from our D2 Algorithms:
\begin{equation}
\label{eqn_proofs_3}
\theta_{i, t+1} = \theta_{i, t} + \alpha_t\left( 
 \tau_i - \mathds{1}_{ \{  R_{t+1} < \theta_{i,t} \} } \right), \; \forall i = 1,  2,  \ldots, m,
\end{equation}
where \(\theta_{i, t} \) denotes the estimate of the \(\tau_i\)-quantile of the limiting per-step reward distribution \eqref{eq_r_dist_appendix}, \(\alpha_t\) denotes the step size for the update, and \(R_{t+1}\) denotes the per-step reward. We claim that the per-step reward quantile update rule \eqref{eqn_proofs_3} can also be viewed as an instance of the more generic update rule \eqref{eqn_proofs_2}. As such, similar to what was done in \citet{Rowland2024-sg} for the discounted return quantile update rule, we will show that the per-step reward quantile update rule \eqref{eqn_proofs_3} is indeed an instance of the generic update rule \eqref{eqn_proofs_2}, and also satisfies the assumptions required for convergence of the more generic update rule \eqref{eqn_proofs_2}.

We will begin by first summarizing key concepts and relevant results from \citet{Rowland2024-sg}:

\begin{definition}[Differential Inclusion]
\label{def:diff-inclusion}
A differential inclusion is a generalization of an ordinary differential equation (ODE), where the derivative of the unknown function is constrained to lie within a set of possible values, rather than being given by a single-valued function. More formally, given a set-valued map, \( H: \mathbb{R}^m \rightrightarrows \mathbb{R}^m \), the differential inclusion,
\[
\partial_t z_t \in H(z_t),
\]
describes a family of trajectories \( \{z_t\}_{t \geq 0} \) such that the derivative at each time \( t \) belongs to the set \( H(z_t) \).
\end{definition}

\begin{definition}[Definition 11 of \citet{Rowland2024-sg}]
\label{def:di-soln}
Let \( H : \mathbb{R}^m \rightrightarrows \mathbb{R}^m \) be a set-valued map. 
The path \( \{z_t\}_{t \geq 0} \) is a solution to the differential inclusion \( \partial_t z_t \in H(z_t) \) if there exists an integrable function \( g : [0, \infty) \rightarrow \mathbb{R}^{m} \) such that
\begin{align}\label{eq:integral}
    z_t = \int_0^t g_s \mathrm{d}s \nonumber
\end{align}
 for all \( t \geq 0 \), and \( g_t \in H(z_t) \) for almost all \( t \geq 0 \).
\end{definition}

\begin{proposition}[Proposition 12 of \citet{Rowland2024-sg}]
\label{proposition_marchaud}
    Consider a set-valued map, \( H : \mathbb{R}^m \rightrightarrows \mathbb{R}^m \), and suppose that \( H \) is a \emph{Marchaud map}; that is,
    \begin{itemize}
        \item the set \( \{ (z, h) : z \in \mathbb{R}^m, h \in H(z) \} \) is closed,
        \item for all \( z \in \mathbb{R}^m \), \( H(z) \) is non-empty, compact, and convex, and
        \item there exists a constant \( C > 0 \) such that for all \( z \in \mathbb{R}^m \),
        \begin{align*}
            \max_{h \in H(z)} \| h \| \leq C ( 1 + \|z\| ) \, .
        \end{align*}
    \end{itemize}
    Then, the differential inclusion \( \partial_t z_t \in H(z_t) \) has a global solution for any initial condition.
\end{proposition}

\begin{definition}[Definition 13 of \citet{Rowland2024-sg}]
Consider a Marchaud map, \( H : \mathbb{R}^m \rightrightarrows \mathbb{R}^m \), and a subset \( \Lambda \subseteq \mathbb{R}^m \). A continuous function \( L : \mathbb{R}^m \rightarrow [0, \infty) \) is said to be a \emph{Lyapunov function} for the differential inclusion \( \partial_t z_t \in H(z_t) \) and subset \( \Lambda \) if for any solution \( \{z_t\}_{t \geq 0} \) of the differential inclusion and \( 0 \leq s < t \), we have \( L(z_t) < L(z_s) \) for all \( z_s \not\in \Lambda \) and \( L(z) = 0 \) for all \( z \in \Lambda \).
\end{definition}

The above definitions and proposition pertain to the theory of \emph{differential inclusions}. In essence, the notion of differential inclusions is invoked to account for the fact that the CDF of the limiting per-step reward distribution may not be strictly increasing. That is, there may be discontinuities in the CDF, which means that using a standard ODE framework for the analysis is insufficient as it does not guarantee the existence of a solution. Differential inclusions, on the other hand, provide a more general framework in which the dynamics are defined in terms of a set-valued map, \( H \), allowing for the solution to take any value within a specified range at any point, including at a given discontinuity. As established in Proposition \ref{proposition_marchaud}, if the set-valued map, \( H \), satisfies the conditions of a \emph{Marchaud map}, then the existence of global solutions is guaranteed, even if the underlying system exhibits discontinuous behavior.

Having stated the necessary background as it relates to differential inclusions, we are now ready to state the main convergence result from \citet{Rowland2024-sg}:
\begin{theorem}[Theorem 14 of \citet{Rowland2024-sg}]
\label{thm:benaim-result}
Consider a Marchaud map, \( H : \mathbb{R}^m \rightrightarrows \mathbb{R}^m \), and the corresponding differential inclusion, \( \partial_t z_t \in H(z_t) \). Suppose there exists a Lyapunov function, \( L \), for this differential inclusion and a subset, \( \Lambda \subseteq \mathbb{R}^m \). Suppose also that we have a sequence \( \{\theta_t\}_{t \geq 0} \) satisfying
\begin{align*}
    \theta_{t+1} = \theta_t + \alpha_t (g(\theta_t) + M_t) \, ,
\end{align*}
where,
\begin{itemize}
    \item \( \{\alpha_t\}_{t= 0}^\infty \) satisfy the conditions \( \sum_{t=0}^\infty \alpha_t = \infty \), \( \alpha_t = o(1/\log(t)) \);
    \item \( g(\theta_t) \in H(\theta_t) \) for all \( t \geq 0 \);
    \item \( \{M_t\}_{t=0}^\infty \) is a bounded martingale difference sequence with respect to the natural filtration generated by \( \{\theta_t\}_{t=0}^\infty \). That is, there is an absolute constant \( C \) such that \( \|M_t\|_\infty < C \) almost surely, and \( \mathbb{E}[M_t \mid \theta_{0},\ldots,\theta_{t} ] = 0 \).
\end{itemize}
If further \( \{\theta_t\}_{t=0}^\infty \) is bounded almost surely (that is, \( \sup_{t \geq 0}\| \theta_t \|_\infty < \infty \) almost surely), then \( \theta_t \rightarrow \Lambda \) almost surely.
\end{theorem}

In short, Theorem~\ref{thm:benaim-result} establishes the convergence of the generic update rule \eqref{eqn_proofs_2}. As such, as per Theorem~\ref{thm:benaim-result}, in order to establish that the \emph{per-step reward} quantile estimates generated by the update rule \eqref{eqn_proofs_3} converge to the quantiles of the limiting per-step reward distribution, it suffices to verify that the following conditions are satisfied:

\begin{itemize}
    \item that the resulting set-valued map, \(H\), is a Marchaud map;
    \item that the update rule \eqref{eqn_proofs_3} can be written in terms of the more generic update rule \eqref{eqn_proofs_2}, such that \( \{M_t\}_{t=0}^\infty \) is a bounded martingale difference sequence;
    \item that \( \{\theta_t\}_{t=0}^\infty \) is bounded almost surely; and
    \item that a Lyapunov function, \(L\), exists for the associated differential inclusion, \(\partial_t z_t \in H(z_t)\) and subset \(\Lambda\).
\end{itemize}

We will now show that the above properties hold for the sequence of per-step reward quantile estimates obtained via the update rule \eqref{eqn_proofs_3}:

\begin{lemma}[Marchaud Map for the Per-Step Reward Quantile Estimates]
\label{marchaud_lemma}
Consider the per-step reward quantile update rule \eqref{eqn_proofs_3}. Let \( F \) denote the CDF of the limiting per-step reward distribution, and let \( \tau_i \in (0, 1) \) denote the \(i^\text{th}\) target quantile level, with the corresponding quantile value given by \( q_i = F^{-1}(\tau_i) \). Define the set-valued map, \(H: \mathbb{R}^m \rightrightarrows \mathbb{R}^m\), as follows: For any \(\theta \in \mathbb{R}^m\), let
\begin{equation}
H(\theta)_i = \tau_i - F^*(\theta_i), \nonumber
\end{equation}
where,
\begin{equation}
F^*(\theta_i) = 
\begin{cases}
\{F(\theta_i)\}, & \text{if } F \text{ is continuous at } \theta_i \\
[F(\theta_i^-), F(\theta_i)], & \text{if } F \text{ has a jump (or discontinuity) at } \theta_i.
\end{cases}\nonumber
\end{equation}

Then \(H\) is a Marchaud map. That is, it satisfies the following properties:
\begin{enumerate}
    \item The set \(\{(\theta, h) : \theta \in \mathbb{R}^m, h \in H(\theta)\}\) is closed;
    \item For all \(\theta \in \mathbb{R}^m\), \(H(\theta)\) is non-empty, compact, and convex;
    \item There exists a constant \(C > 0\) such that for all \(\theta \in \mathbb{R}^m\),
    \begin{equation}
        \max_{h \in H(\theta)} \|h\| \leq C(1 + \|\theta\|).\nonumber
    \end{equation}
\end{enumerate}
\end{lemma}

\begin{proof}
We verify each condition of a Marchaud map:\\

\textbf{Condition 1:} To show that the graph of \(H\) is closed, consider a convergent sequence \((\theta^k, h^k) \rightarrow (\theta, h)\) where \(h^k \in H(\theta^k)\) for all \(k\). For each component, \(i\), we have \(h^k_i = \tau_i - p^k_i\) for some \(p^k_i \in F^*(\theta^k_i)\). Since \(F^*\) takes values in \([0, 1]\), the sequence \(\{p^k_i\}\) is bounded and thus has a convergent subsequence with limit \(p_i\). If \(F\) is continuous at \(\theta_i\), then by definition of continuity, \(F(\theta^k_i) \rightarrow F(\theta_i)\), and since \(p^k_i = F(\theta^k_i)\) for large enough \(k\), it follows that \(p_i = F(\theta_i) \in F^*(\theta_i)\). Conversely, if \(F\) has a jump at \(\theta_i\), then by the upper hemicontinuity of \(F^*\), we have \(p_i \in [F(\theta_i^-), F(\theta_i)] = F^*(\theta_i)\). Therefore, \(h_i = \tau_i - p_i\) with \(p_i \in F^*(\theta_i)\), which implies \(h \in H(\theta)\), confirming that the graph of \(H\) is closed.\\

\textbf{Condition 2:} For any \(\theta \in \mathbb{R}^m\), we have:
\begin{itemize}
    \item \textit{Non-empty:} For each component, \(i\), the set \(F^*(\theta_i)\) is either the singleton \(\{F(\theta_i)\}\) when \(F\) is continuous at \(\theta_i\), or the closed interval \([F(\theta_i^-), F(\theta_i)]\) if \(F\) has a jump at \(\theta_i\). In both cases, \(F^*(\theta_i)\) is clearly non-empty, which implies that \(H(\theta)_i = \tau_i - F^*(\theta_i)\) is also non-empty.

    \item \textit{Compact:} Each \(F^*(\theta_i)\) is either a singleton \(\{F(\theta_i)\}\) or a closed interval \([F(\theta_i^-), F(\theta_i)]\), both of which are compact subsets of \(\mathbb{R}\). Since translation (via subtraction by the constant \(\tau_i\)) is a continuous map on \(\mathbb{R}\), the image \(\tau_i - F^*(\theta_i)\) is also compact.

    \item \textit{Convex:} Each \(F^*(\theta_i)\) is convex, since both singletons and closed intervals are convex subsets of \(\mathbb{R}\). The affine transformation \(x \mapsto \tau_i - x\) preserves convexity, hence we can conclude that \(H(\theta)_i = \tau_i - F^*(\theta_i)\) is also convex.

\end{itemize}

Since these properties hold component-wise, \(H(\theta)\) as a whole is non-empty, compact, and convex.\\

\textbf{Condition 3:} For any \(\theta \in \mathbb{R}^m\) and \(h \in H(\theta)\), each component satisfies \(h_i = \tau_i - p_i\) for some \(p_i \in F^*(\theta_i)\). Since \(\tau_i \in (0, 1)\) and \(p_i \in [0, 1]\) (as \(F^*\) takes values in \([0, 1]\)), we have \(|h_i| \leq 1\) for all \(i\). Therefore, when we consider the infinity norm, we have that \(\|h\|_\infty \leq 1\) for all \(h \in H(\theta)\). Hence, we may take \(C = 1\). This gives:
\begin{equation}
    \max_{h \in H(\theta)} \|h\|_\infty \leq 1 \leq C(1 + \|\theta\|_\infty),
\end{equation}
which satisfies the desired linear growth condition.\\

As such, given that the three conditions are satisfied, we can conclude that \(H\) is a Marchaud map.
\end{proof}

\vspace{11pt}

\begin{lemma}[Bounded Martingale Difference Sequence]
\label{martingale_lemma}
Consider the per-step reward quantile update rule \eqref{eqn_proofs_3}. This update rule can be rewritten in terms of the generic update rule \eqref{eqn_proofs_2} as follows:
\begin{equation}
\theta_{i, t+1} = \theta_{i, t} + \alpha_t (g(\theta_{i,t}) + M_{i,t}), \; \forall i=1, 2, \ldots, m, \nonumber
\end{equation}
where,
\begin{equation}
g(\theta_{i,t}) = \mathbb{E}[\tau_i - \mathds{1}_{\{R_{t+1} < \theta_{i,t}\}} \mid \mathcal{F}_{i,t} \doteq \sigma(\theta_{i, 0}, R_1, \theta_{i, 1}, R_2, \ldots,\theta_{i, t})],
\end{equation}
and
\begin{equation}
\label{eqn_proofs_4}
M_{i,t} = \tau_i -\mathds{1}_{\{R_{t+1} < \theta_{i,t}\}} - \mathbb{E}[\tau_i - \mathds{1}_{\{R_{t+1} < \theta_{i,t}\}} \mid \mathcal{F}_{i,t}],
\end{equation}

such that the sequence \(\{M_t\}_{t=0}^\infty\), where \(M_t \doteq (M_{1,t}, M_{2,t}, \ldots, M_{m,t})\), is a bounded martingale difference sequence with respect to the filtration \(\{\mathcal{F}_{ t}\}_{t=0}^{\infty}\), where \(\mathcal{F}_t \doteq (\mathcal{F}_{1,t}, \mathcal{F}_{2,t}, \ldots, \mathcal{F}_{m,t})\).
\end{lemma}

\begin{proof}
We need to show that the sequence \(\{M_t\}_{t \geq 0}\) satisfies two key properties:\\

\textit{1) Martingale Difference:} \(\mathbb{E}[M_{i,t} \mid \mathcal{F}_{i,t}] = 0 \; \forall i = 1, 2, \ldots, m,\) for all \(t \geq 0\).\\

\textit{2) Boundedness:} There exists a constant \(C_i > 0\), such that \(|M_{i,t}| < C_i \; \forall i = 1, 2, \ldots, m,\) almost surely for all \(t \geq 0\).

\vspace{0.5cm}

Let us consider an arbitrary \(i^\text{th}\) component:\\

\textbf{Martingale Difference:}

Consider Equation \eqref{eqn_proofs_4}. Taking the conditional expectation with respect to \(\mathcal{F}_{i,t}\) we obtain:
\begin{align}
\mathbb{E}[M_{i, t} \mid \mathcal{F}_{i, t}] &= \mathbb{E}[\tau_i -\mathds{1}_{\{R_{t+1} < \theta_{i,t}\}} - \mathbb{E}[\tau_i - \mathds{1}_{\{R_{t+1} < \theta_{i,t}\}} \mid \mathcal{F}_{i, t}] \mid \mathcal{F}_{i, t}] \\
&= -\mathbb{E}[\mathds{1}_{\{R_{t+1} < \theta_{i,t}\}} \mid \mathcal{F}_{i, t}] + \mathbb{E}[\mathbb{E}[\mathds{1}_{\{R_{t+1} < \theta_{i,t}\}} \mid \mathcal{F}_{i, t}] \mid \mathcal{F}_{i, t}].
\end{align}

By the tower property of conditional expectations, we have that:
\begin{align}
\mathbb{E}[\mathbb{E}[\mathds{1}_{\{R_{t+1} < \theta_{i,t}\}} \mid \mathcal{F}_{i, t}] \mid \mathcal{F}_{i, t}] &= \mathbb{E}[\mathds{1}_{\{R_{t+1} < \theta_{i,t}\}} \mid \mathcal{F}_{i, t}].
\end{align}

As such,
\begin{align}
\mathbb{E}[M_{i, t} \mid \mathcal{F}_{i, t}] &= -\mathbb{E}[\mathds{1}_{\{R_{t+1} < \theta_{i,t}\}} \mid \mathcal{F}_{i, t}] + \mathbb{E}[\mathds{1}_{\{R_{t+1} < \theta_{i,t}\}} \mid \mathcal{F}_{i, t}] \\
&= 0,
\end{align}
which confirms that \(\{M_{i, t}\}_{t \geq 0}\) is indeed a martingale difference sequence with respect to the filtration \(\{\mathcal{F}_{i, t}\}_{t \geq 0}\).\\

\textbf{Boundedness:}

Consider Equation \eqref{eqn_proofs_4}. We can see that the indicator function \(\mathds{1}_{\{R_{t+1} < \theta_{i,t}\}}\) can only take values in \(\{0, 1\}\). Similarly, its conditional expectation \(\mathbb{E}[\mathds{1}_{\{R_{t+1} < \theta_{i,t}\}} \mid \mathcal{F}_{i, t}]\) is a probability, and hence must lie in the interval \([0, 1]\). As such, for all realizations and for all \(t \geq 0\):
\begin{align}
|M_{i, t}| &= |\tau_i -\mathds{1}_{\{R_{t+1} < \theta_{i,t}\}} - \mathbb{E}[\tau_i - \mathds{1}_{\{R_{t+1} < \theta_{i,t}\}} \mid \mathcal{F}_{i, t}]| \\
&\leq \max\{|\mathbb{E}[\mathds{1}_{\{R_{t+1} < \theta_{i,t}\}} \mid \mathcal{F}_{i, t}] - 0|, |\mathbb{E}[\mathds{1}_{\{R_{t+1} < \theta_{i,t}\}} \mid \mathcal{F}_{i, t}] - 1|\} \\
&\leq \max\{1, 1\} \\
&= 1
\end{align}

Taking \(C_i = 1\), we have that \(|M_{i, t}| \leq C_i\) almost surely for all \(t \geq 0\).

\vspace{0.5cm}
Finally, since we chose \(i\) arbitrarily, we can extend the above results for all \(i = 1, 2, \ldots, m\). Consequently, we can conclude that the sequence \(\{M_t\}_{t=0}^\infty\) is a bounded martingale difference sequence with respect to the filtration \(\{\mathcal{F}_{ t}\}_{t=0}^{\infty}\).
\end{proof}

\vspace{11pt}

\begin{lemma}[Boundedness of the Per-Step Reward Quantile Estimates]
\label{boundedness_lemma}
Consider the per-step reward quantile update rule \eqref{eqn_proofs_3}:
\begin{equation}
\theta_{i, t+1} = \theta_{i, t} + \alpha_t\left(\tau_i - \mathds{1}\left\{R_{t+1} < \theta_{i,t}\right\}\right), \; \forall i = 1, 2, \ldots, m, \nonumber
\end{equation}
where \( \theta_{i, t} \) denotes the estimate of the \( \tau_i \)-quantile of the limiting per-step reward distribution, \( \alpha_t \) denotes the step size for the update, and \( R_{t+1} \) denotes the per-step reward.

If:
\begin{enumerate}
\item the per-step rewards \( \{R_{t+1}\}_{t \geq 0} \) are bounded (i.e., there exist some finite constants \( R_\text{min}, R_\text{max} \in \mathbb{R} \) such that \( R_\text{min} \leq R_{t+1} \leq R_\text{max} \) for all \( t \geq 0 \)), and
\item the step sizes satisfy \(\alpha_t > 0, \sup_{t \geq 0} \alpha_t < \infty\), \( \sum_{t=0}^{\infty} \alpha_t = \infty\), and \(\alpha_t = o(1/\log(t))\),
\end{enumerate}
then the sequence \( \{\theta_{t}\}_{t=0}^{\infty} \) is bounded almost surely. That is, \( \sup_{t \geq 0} |\theta_{i,t}| < \infty \) almost surely for each \( i = 1, 2, \ldots, m \), such that, \( \sup_{t \geq 0}\| \theta_t \|_\infty < \infty\).
\end{lemma}

\begin{proof}
First, we note that since \( \tau_i \in (0,1) \), and the indicator function \( \mathds{1}\{R_{t+1} < \theta_{i,t}\} \in \{0,1\} \), the absolute value of the difference \( \left|\tau_i - \mathds{1}\{R_{t+1} < \theta_{i,t}\}\right| \) is always at most 1. As such, the magnitude of each per-step update is bounded by the step size. In other words, we have that:
\begin{equation}
|\theta_{i,t+1} - \theta_{i,t}| = \alpha_t \left|\tau_i - \mathds{1}\left\{R_{t+1} < \theta_{i,t}\right\}\right| \leq \alpha_t.
\end{equation}

Moreover, given the step size assumptions, it follows that \( \sup_{t \geq 0} \alpha_t < \infty \), and hence, that the per-step update magnitudes are uniformly bounded over time.

\vspace{11pt}

We will next show that a sequence \( \{\theta_{i,t}\}_{t \geq 0} \) for an arbitrary \(i^\text{th}\) component cannot diverge to either \( +\infty \) or \( -\infty \):\\

\textbf{Upper Bound:}

We proceed by contradiction. Suppose there exists a sample path where \( \theta_{i,t} \rightarrow +\infty \) as \( t \rightarrow \infty \). Then there must exist some time \( T \) such that \( \theta_{i,t} > R_\text{max} \) for all \( t \geq T \).  

Consider the update rule for this sample path. For all \( t \geq T \), we have that:
\begin{subequations}
\begin{align}
\theta_{i,t+1} - \theta_{i,t} &= \alpha_t(\tau_i - \mathds{1}\left\{R_{t+1} < \theta_{i,t}\right\}) \\
&= \alpha_t(\tau_i - 1) \\
&= -\alpha_t(1-\tau_i).
\end{align}
\end{subequations}

Since \( \tau_i \in (0,1) \), we have that \( -\alpha_t(1-\tau_i) < 0 \), which means that the change in \( \theta_{i,t} \) is negative.

Now consider the cumulative effect over time \( T \) through some time \( T+k-1 \):
\begin{subequations}
\begin{align}
\theta_{i,T+k} - \theta_{i,T} &= \sum_{t=T}^{T+k-1} (\theta_{i,t+1} - \theta_{i,t}) \\
&= \sum_{t=T}^{T+k-1} -\alpha_t(1-\tau_i) \\
\label{eqn_proofs_5}
&= -(1-\tau_i)\sum_{t=T}^{T+k-1} \alpha_t.
\end{align}
\end{subequations}

Looking at Equation \eqref{eqn_proofs_5}, we know by assumption that \( \sum_{t=0}^{\infty} \alpha_t = \infty \). As such, there must exist some  \( k^{*} \) such that eventually, the sample path will eventually return to and cross \(R_{\text{max}}\). In other words, we will eventually have:
\begin{equation}
\theta_{i,T} - (1-\tau_i)\sum_{t=T}^{T+k^{*}-1}\alpha_t \leq R_\text{max},
\end{equation}
and thus, that
\begin{equation}
\theta_{i,T+k^{*}} = \theta_{i,T} - (1-\tau_i)\sum_{t=T}^{T+k^{*}-1}\alpha_t \leq R_\text{max},
\end{equation}

which contradicts the assumption that \( \theta_{i,t} > R_\text{max} \) for all \( t \geq T \).\\

\textbf{Lower Bound:}

We proceed by contradiction. Suppose there exists a sample path where \( \theta_{i,t} \rightarrow -\infty \) as \( t \rightarrow \infty \). Then there must exist some time \( T \) such that \( \theta_{i,t} < R_\text{min} \) for all \( t \geq T \).

Consider the update rule for this sample path. For all \( t \geq T \), we have that:
\begin{subequations}
\begin{align}
\theta_{i,t+1} - \theta_{i,t} &= \alpha_t(\tau_i - \mathds{1}\left\{R_{t+1} < \theta_{i,t}\right\}) \\
&= \alpha_t \tau_i.
\end{align}
\end{subequations}

Since \( \tau_i > 0 \), the change in \( \theta_{i,t} \) is positive.

Now consider the cumulative effect over time \( T \) through some time \( T+k-1 \):
\begin{equation}
\label{eqn_proofs_6}
\theta_{i,T+k} - \theta_{i,T} = \tau_i \sum_{t=T}^{T+k-1} \alpha_t.
\end{equation}

Looking at Equation \eqref{eqn_proofs_6}, we know by assumption that \( \sum_{t=0}^{\infty} \alpha_t = \infty \). As such, there must exist some  \( k^{*} \) such that eventually, the sample path will eventually return to and cross \(R_{\text{min}}\). In other words, we will eventually have:
\begin{equation}
\theta_{i,T} + \tau_i \sum_{t=T}^{T+k^{*}-1} \alpha_t \geq R_\text{min},
\end{equation}
and thus, that
\begin{equation}
\theta_{i,T+k^{*}} =\theta_{i,T} + \tau_i \sum_{t=T}^{T+k^{*}-1} \alpha_t \geq R_\text{min},
\end{equation}

which contradicts the assumption \( \theta_{i,t} < R_\text{min} \) for all \( t \geq T \).\\

As such, we have shown that \( \theta_{i,t} \) cannot diverge to \( +\infty \) or \( -\infty \) along any sample path. Hence, the sequence \( \{\theta_{i,t}\} \) must remain within a bounded interval, almost surely. Moreover, because we chose \(i\) arbitrarily, we can extend the above results for all \(i = 1, 2, \ldots, m\). Consequently, we can conclude that \( \sup_{t \geq 0}\| \theta_t \|_\infty < \infty\).
\end{proof}

\begin{lemma}[Existence of a Lyapunov Function for the Per-Step Reward Quantile Estimates]
\label{lyapunov_lemma}
Consider the differential inclusion \(\partial_t z_t \in H(z_t)\) associated with the per-step reward quantile update rule \eqref{eqn_proofs_3}, where \(H\) is the Marchaud map defined in Lemma \ref{marchaud_lemma}. Let \(F\) be the CDF of the limiting per-step reward distribution, and let \(q_i = F^{-1}(\tau_i)\) denote the true \(\tau_i\)-quantile of this distribution. Define the subset \(\Lambda = \{ \theta \in \mathbb{R}^m : \theta_i \in q_i \;\; \text{for all } i = 1, 2, \dots, m \}\), where:
\[
q_i =
\begin{cases}
F^{-1}(\tau_i), & \text{if } F^{-1}(\tau_i) \text{ is unique}, \\
[a_i, b_i], & \text{if } F \text{ has a flat region at level } \tau_i,
\end{cases}
\]
with \(a_i = \inf\{x : F(x) \geq \tau_i\}\) and \(b_i = \sup\{x : F(x) \leq \tau_i\}\). The function \(L: \mathbb{R}^m \rightarrow [0, \infty)\) defined by:
\[
L(\theta) = \sum_{i=1}^{m} \int_{q_i}^{\theta_i} (F(s) - \tau_i)\, ds
\]
is a Lyapunov function for the differential inclusion \(\partial_t z_t \in H(z_t)\) and the subset \(\Lambda\).
\end{lemma}

\begin{proof}
To establish that \(L\) is a Lyapunov function, we must show that:
\begin{enumerate}
\item \(L\) is continuous,
\item \(L(\theta) = 0\) for all \(\theta \in \Lambda\),
\item \(L(\theta) > 0\) for all \(\theta \notin \Lambda\), and
\item For any solution \(\{z_t\}_{t \geq 0}\) of the differential inclusion and \(0 \leq s < t\), we have \(L(z_t) < L(z_s)\) for all \(z_s \not\in \Lambda\).
\end{enumerate}

\textbf{Continuity of \(L\):}  

First, let us consider the integral in the definition of \(L\). For each component, \(i\), we can write:
\begin{equation}
\int_{q_i}^{\theta_i} (F(s) - \tau_i)\, ds \doteq
\begin{cases}
\int_{q_i}^{\theta_i} (F(s) - \tau_i)\, ds, & \text{if } \theta_i \geq q_i, \\
- \int_{\theta_i}^{q_i} (F(s) - \tau_i)\, ds, & \text{if } \theta_i < q_i.
\end{cases}
\end{equation}
Since \(F\) is a CDF, it is right-continuous by definition. Thus, for any discontinuity in \(F\), we interpret \(F(s)\) as its right-hand limit. Moreover, the function \(F(s) - \tau_i\) is bounded (since \(\tau_i \in (0,1)\) and \(F(s) \in [0,1]\)) and can only have countably many jump discontinuities (since \(F\) is non-decreasing), thereby making the integral a well-defined Riemann integral for any choice of \(\theta_i\) and \(q_i\).

Now, for the continuity of \(L\), consider a sequence \(\{\theta^k\}\) converging to \(\theta\). For each component, \(i\), we claim that:
\begin{equation}
\lim_{k \rightarrow \infty} \int_{q_i}^{\theta_i^k} (F(s) - \tau_i)\, ds = \int_{q_i}^{\theta_i} (F(s) - \tau_i)\, ds.
\end{equation}
This convergence holds because, as discussed in the above paragraph, the integrand \(s \mapsto F(s) - \tau_i\) is bounded and has at most countably many discontinuities. More specifically, it is well-known that for such functions, the integral is continuous with respect to its upper limit of integration. That is, if \(\theta_i^k \to \theta_i\), then the corresponding integrals converge. Since this holds for each \(i\)-th component, and \(L(\theta)\) is a finite sum over these integrals, it follows that \(L(\theta^k) \to L(\theta)\). Therefore, \(L\) is continuous on \(\mathbb{R}^m\).\\

\textbf{\(L(\theta) = 0\) for all \(\theta \in \Lambda\):} 

Let \(\theta \in \Lambda\). By definition, this means that \(\theta_i \in q_i\) for all \(i\). There are two cases that we need to consider:

\textbf{Case 1}: If \(F^{-1}(\tau_i)\) is unique (i.e., \(q_i\) is a singleton), then \(\theta_i = q_i\), which implies:
\begin{equation}
\int_{q_i}^{\theta_i} (F(s) - \tau_i)\, ds = \int_{q_i}^{q_i} (F(s) - \tau_i)\, ds = 0.
\end{equation}

\textbf{Case 2}: If \(F\) has a flat region at level \(\tau_i\) (i.e., \(q_i = [a_i, b_i]\)), then \(\theta_i \in [a_i, b_i]\), and for all \(s \in [a_i, b_i]\), we have \(F(s) = \tau_i\), so:
\begin{equation}
\int_{q_i}^{\theta_i} (F(s) - \tau_i)\, ds = \int_{q_i}^{\theta_i} 0\, ds = 0.
\end{equation}

Thus, each component contributes zero to the sum, and \(L(\theta) = 0\) for all \(\theta \in \Lambda\).\\

\textbf{\(L(\theta) > 0\) for all \(\theta \notin \Lambda\):} 

Let \(\theta \notin \Lambda\). By definition, this means that there exists at least one \(i \in \{1, 2, \ldots, m\}\) such that \(\theta_i \notin q_i\). There are four cases that we need to consider:

\textbf{Case 1}: If \(F^{-1}(\tau_i)\) is unique (i.e., \(q_i\) is a singleton) and \(\theta_i > q_i\), we have that \(F(\theta_i) - \tau_i > 0\), such that:
\begin{equation}
\int_{q_i}^{\theta_i} (F(s) - \tau_i)\, ds > 0.
\end{equation}

\textbf{Case 2}: If \(F^{-1}(\tau_i)\) is unique (i.e., \(q_i\) is a singleton) and \(\theta_i < q_i\), we have that \(F(\theta_i) - \tau_i < 0\), such that:
\begin{equation}
\int_{q_i}^{\theta_i} (F(s) - \tau_i)\, ds = -\int_{\theta_i}^{q_i} (F(s) - \tau_i)\, ds > 0.
\end{equation}

\textbf{Case 3}: If \(F\) has a flat region at level \(\tau_i\) (i.e., \(q_i = [a_i, b_i]\)) and \(\theta_i > b_i\), we have that \(F(\theta_i) - \tau_i > 0\), such that:
\begin{equation}
\int_{b_i}^{\theta_i} (F(s) - \tau_i)\, ds > 0.
\end{equation}

\textbf{Case 4}: If \(F\) has a flat region at level \(\tau_i\) (i.e., \(q_i = [a_i, b_i]\)) and \(\theta_i < a_i\), we have that \(F(\theta_i) - \tau_i < 0\), such that:
\begin{equation}
\int_{a_i}^{\theta_i} (F(s) - \tau_i)\, ds = - \int_{\theta_i}^{a_i} (F(s) - \tau_i)\, ds > 0.
\end{equation}

Thus, we have at least one component contributing a positive value to the sum, with the remaining components contributing non-negative values, such that \(L(\theta) > 0\) for all \(\theta \notin \Lambda\).\\

\textbf{\(L(z_t) < L(z_s)\) for all \(z_s \not\in \Lambda\) and \(0 \leq s < t\):}  

We will show the desired result by analyzing the derivative of \(L\). To this end, let \(\{z_t\}_{t \geq 0}\) be a solution to \(\partial_t z_t \in H(z_t)\), and let \(0 \leq s < t\). From Definition \ref{def:di-soln}, we know that there exists an integrable function \(g: [0, \infty) \rightarrow \mathbb{R}^m\), such that:
\begin{equation}
z_t = \int_0^t g_u\, du,
\end{equation}
where \(g_u \in H(z_u)\) for almost all \(u \geq 0\). Accordingly, we have that \(\frac{d}{dt} z_t = g_t\).\\

Now consider the derivative of \(L\). Since \(L\) is non-smooth (i.e., it is not differentiable at the jump discontinuities of \(F\)), we need to utilize the Clarke generalized gradient of \(L\). In particular, we can apply the chain rule to derive the derivative of \(L\) as follows:
\begin{equation} \frac{d}{dt}L(z_t) \in \left\{ \sum_{i=1}^m v_{i,t} \cdot \frac{dz_{i, t}}{dt} = \sum_{i=1}^m v_{i,t} \cdot g_{i, t} \; ; \; v_{i,t} \in \frac{\partial L(z_t)}{\partial z_{i, t}} \right\}, \end{equation}
where \(\frac{\partial L(z_t)}{\partial z_{i, t}} \doteq [F(z_{i,t}^-) - \tau_i, F(z_{i,t}) - \tau_i]\) is the \(i\)-th component of the Clarke generalized gradient of \(L\), and, by the definition of the Marchaud map, \(H(z_t)\) (see Lemma \ref{marchaud_lemma}), \(g_{i, t} \in [\tau_i - F(z_{i,t}), \tau_i - F(z_{i,t}^-)]\) is the \(i\)-th component of \(g_t\).\\

Let us now analyze the product \(v_{i,t} \cdot g_{i,t}\) for any \(z_{i,t} \notin q_i\). There are two cases that we need to consider:\\

\textbf{Case 1:} \(z_{i,t} > \sup q_i = b_i\): Since \(z_{i,t} > b_i\)  and \(F\) is non-decreasing, we know that any \(z_{i,t} > b_i\) must satisfy \(F(z_{i,t}) \ge F(z_{i,t}^-) > \tau_i\). Therefore, any gradient \(v_i \in [F(z_{i,t}^-) - \tau_i, F(z_{i,t}) - \tau_i]\) must be strictly positive. Similarly, any \(g_i \in [\tau_i - F(z_{i,t}), \tau_i - F(z_{i,t}^-)]\) 
must be strictly negative. Hence, the product \(v_{i,t} \cdot g_{i,t}\) must be strictly negative.\\

\textbf{Case 2:} \(z_{i,t} < \inf q_i = a_i\): Since \(z_{i,t} < a_i\)  and \(F\) is non-decreasing, we know that any \(z_{i,t} < a_i\) must satisfy \(F(z_{i,t}^-) \le F(z_{i,t}) < \tau_i\). Therefore, any gradient \(v_i \in [F(z_{i,t}^-) - \tau_i, F(z_{i,t}) - \tau_i]\) must be strictly negative. Similarly, any \(g_i \in [\tau_i - F(z_{i,t}), \tau_i - F(z_{i,t}^-)]\) must be strictly positive. Hence, the product \(v_{i,t} \cdot g_{i,t}\) must be strictly negative.\\

Thus, for any \(v_{i,t}\) and \(g_{i,t}\), the product \(v_{i,t} \cdot g_{i,t}\) is strictly negative.\\ 

Now, let us define \(D(z_t) = \{i \in \{1, \ldots, m\}; z_{i, t} \not\in q_i\}\). If \(z_t \not\in \Lambda\), then \(D(z_t)\) is non-empty, and, by the above analysis, we have that:
\begin{equation}
\frac{d}{dt} L(z_t) = \sum_{i=1}^m v_{i, t} \cdot g_{i,t} = \sum_{i \in D(z_t)} v_{i, t} \cdot g_{i,t} < 0 \quad \forall \, v_{i,t}, g_{i,t}.
\end{equation}

Therefore, if \(z_t \not\in \Lambda\), we can conclude that \(\frac{d}{dt} L(z_t) < 0\).\\

As such, integrating over \([s, t]\) gives:
\[
L(z_t) - L(z_s) = \int_s^t \frac{d}{d\sigma} L(z_\sigma) \, d\sigma < 0,
\]
which implies that \(L(z_t) < L(z_s)\) for all \(z_s \not\in \Lambda\) and \(0 \leq s < t\).\\

As such, we have now verified the four conditions, and can therefore conclude that \(L\) is a Lyapunov function for \(\partial_t z_t \in H(z_t)\) and \(\Lambda\).
\end{proof}

With the above Lemmas, we have now established that the necessary conditions for convergence hold for the per-step reward quantile update rule \eqref{eqn_proofs_3}. In particular, we are now in a position to apply Theorem \ref{thm:benaim-result}, which allows us to conclude that the quantile estimates generated by the per-step reward update rule \eqref{eqn_proofs_3} converge, almost surely, to the quantiles of the limiting per-step reward distribution. We formalize this argument as Theorem \ref{thm_q_convergence} below:\\

\begin{theorem}[Convergence of the Per-Step Reward Quantile Estimates]
\label{thm_q_convergence}
Consider the per-step reward quantile update rule \eqref{eqn_proofs_3}:
\[
\theta_{i, t+1} = \theta_{i, t} + \alpha_t\left(\tau_i - \mathds{1}\{R_{t+1} < \theta_{i,t}\}\right), \quad \forall i = 1, 2, \ldots, m,
\]
where \(\theta_{i, t}\) is the estimate of the \(\tau_i\)-quantile of the limiting per-step reward distribution, \(\alpha_t\) is the step size, and \(R_{t+1} \in \mathcal{R}\) is the per-step reward.

If:
\begin{enumerate}
\item the per-step rewards \(\{R_{t+1}\}_{t \geq 0}\) are bounded,
\item the step sizes \(\{\alpha_t\}_{t \geq 0}\) satisfy \(\alpha_t > 0 \), \(\sup_{t \geq 0} \alpha_t < \infty\), \(\sum_{t=0}^{\infty} \alpha_t = \infty\), and \(\alpha_t = o(1/\log(t))\), and 
\item the distribution of per-step rewards converges to a unique limiting distribution with CDF, \(F\) (e.g. via Assumptions \ref{assumption_unichain} or \ref{assumption_communicating}),
\end{enumerate}

then, the sequence \(\{\theta_{i,t}\}_{t \geq 0}\) converges, almost surely, to the set \(q_i, \; \forall i = 1, 2, \ldots, m\), where:
\[
q_i =
\begin{cases}
F^{-1}(\tau_i), & \text{if } F^{-1}(\tau_i) \text{ is unique}, \\
[a_i, b_i], & \text{if } F \text{ has a flat region at level } \tau_i,
\end{cases}
\]
with \(a_i = \inf\{x : F(x) \geq \tau_i\}\) and \(b_i = \sup\{x : F(x) \leq \tau_i\}\), as \(t \rightarrow \infty\).
\end{theorem}

\begin{proof}
In Lemma \ref{martingale_lemma}, we showed that the per-step reward quantile update rule \eqref{eqn_proofs_3} can be written in terms of the more generic stochastic approximation update rule \eqref{eqn_proofs_2}. Moreover, in Lemmas \ref{marchaud_lemma}, \ref{martingale_lemma}, \ref{boundedness_lemma}, and \ref{lyapunov_lemma}, we showed that:
\begin{itemize}
\item The resulting set-valued map, \(H\), is a Marchaud map;
\item The resulting noise sequence, \(\{M_t\}_{t \geq 0}\), is a bounded martingale difference sequence;
\item The sequence \(\{\theta_{i,t}\}_{t \geq 0}\) is bounded almost surely \(\forall i = 1, 2, \ldots, m\); and
\item There exists an appropriate Lyapunov function, \(L\), for the resulting differential inclusion, \(\partial_t z_t \in H(z_t)\), and the subset, \(\Lambda \doteq \{ \theta \in \mathbb{R}^m : \theta_i \in q_i \; \forall i = 1, 2, \dots, m \}\).
\end{itemize}

As such, these properties collectively satisfy all of the conditions of Theorem \ref{thm:benaim-result} (i.e., Theorem 14 of \citet{Rowland2024-sg}), which guarantees the almost sure convergence of update rules that are of the form of the generic update rule \eqref{eqn_proofs_2}. Therefore, in the case of the per-step reward quantile update rule \eqref{eqn_proofs_3}, we can conclude that \(\theta_{i,t} \rightarrow q_i\), almost surely, for each \(i = 1, 2, \dots, m\).\\ 

This completes the proof.
\end{proof}

\newpage

\subsection{Convergence of the Value Function Estimates}
\label{section_proof_values}

This section contains the convergence proofs for the value function estimates of the tabular D2 TD-learning and tabular D2 Q-learning algorithms (i.e., Algorithms \ref{alg_3} and \ref{alg_4}, respectively). Our general strategy is as follows: we build upon existing results from Section \ref{section_proof_q}, in combination with a two-timescales argument, to show the almost sure convergence of the value function estimates of our algorithms.\\ 

In this section, we adopt the following notation:
\begin{itemize}\itemsep0mm
    \item For a given vector \(x\), let \(\sum x\) denote the sum of all elements in \(x\), such that \(\sum x \doteq \sum_{i} x(i)\).
    \item Let \(e\) denote an all-ones vector.\\
\end{itemize}

\paragraph{Convergence Proof for the Tabular D2 TD-learning Value Function Estimates:\\} 
\label{proof_td_value}

In this section, we present the proof for the convergence of the value function estimates of the D2 TD-learning algorithm (Algorithm \ref{alg_3}). To begin, consider an MDP, \(\mathcal{M} \doteq \langle\mathcal{S}, \mathcal{A}, \mathcal{R}, p \rangle\), with policy, \(\pi\). Given a state \(s \in \mathcal{S}\) and discrete step \(n \geq 0\), let \(A_n(s) \sim \pi(\cdot \mid s)\) denote the action selected, let \(R_n(s, A_n(s)) \in \mathcal{R}\) denote a sample of the resulting reward, and let \(S'_n(s, A_n(s)) \sim p(\cdot, \cdot \mid s, A_n(s))\) denote a sample of the resulting state. Let \(\{Y_n\}\) be a set-valued process taking values in the set of nonempty subsets of \(\mathcal{S}\), such that: \(Y_n = \{s: s\) component of the \(\vert \mathcal{S} \vert\)-sized table of state-value estimates, \(V\), that was updated at step \(n\}\). Let \(\nu(n, s) \doteq \sum_{j=0}^n I\{s \in Y_j\}\), where \(I\) is the indicator function, such that \(\nu(n, s)\) represents the number of times that \(V(s)\) was updated up until step \(n\). 

Now consider the following update rule for \(n \geq 0\):
\begin{align}
    V_{n+1}(s) & \doteq V_n(s) + \alpha_{\nu(n, s)} \delta_n(s) I\{s \in Y_n\}, \quad \forall s \in \mathcal{S}, \label{async_td_value_update_eqn}
\end{align}
where,
\begin{align}
\begin{split}
    \delta_n(s) & \doteq R_n(s, A_n(s)) - \bar{R}_n + V_n(S_n'(s, A_n(s))) - V_n(s).
\end{split}\label{async_td_td_error_eqn}
\end{align}

Here, \(\bar{R}_n\) denotes the estimate of the average-reward induced by following policy \(\pi\), \(\bar{r}_{\pi}\) (see Equation (2)), \(\delta_n(s)\) denotes the TD error, and \(\alpha_{\nu(n, s)}\) denotes the step size at step \(n\) for state \(s\).

We will now show that the value function update rule for the D2 TD-learning algorithm is a special case of the update rule \eqref{async_td_value_update_eqn}. To this end, consider a sequence of experience from our MDP \(\mathcal{M}\): \(S_t, A_t(S_t), R_{t+1}, S_{t+1}, \ldots\). Now recall the set-valued process \(\{Y_n\}\). If we let \(n\) = time step \(t\), we have: 
\begin{align*}
Y_t(s) = 
\begin{cases}
    1, s = S_t,\\
    0, \text{ otherwise,}
\end{cases}
\end{align*}
as well as \(S'_n(S_t, A_t(S_t)) = S_{t+1}\), and \(R_n(S_t, A_t(S_t)) = R_{t+1}\).\\ 

Hence, the update rule \eqref{async_td_value_update_eqn} becomes:
\begin{align}
\label{d2_td_value_update}
    V_{t+1}(S_t) & \doteq V_t (S_t) + \alpha_{\nu(t, S_t)} \delta_t \text{\; and \; } V_{t+1}(s) \doteq V_t (s), \forall s \neq S_t,
\end{align}

which is D2 TD-learning's value function update rule with \(\alpha_{\nu(t, S_t)}\) denoting the step size at time \(t\).\\

As such, if we can show that the update rule \eqref{async_td_value_update_eqn} converges, then it would imply that the value function update rule for the D2 TD-learning algorithm also converges. To this end, we will now specify the assumptions that are needed to ensure convergence:\\ 

\begin{assumption}[Unichain Assumption]\label{assumption_unichain_appendix}
The Markov chain induced by the policy is unichain. That is, the induced Markov chain consists of a single recurrent class and a potentially-empty set of transient states.\\
\end{assumption}

\begin{assumption}[Step Size Assumption] \label{assumption_step_size} \(\alpha_n > 0\), \(\sup_{n \geq 0} \alpha_n < \infty\), \(\sum_{n = 0}^\infty \alpha_n = \infty\), \(\sum_{n = 0}^\infty \alpha_n^2 < \infty\), and \(\alpha_n = o(1/\log(n))\).\\
\end{assumption}

\begin{assumption}[Asynchronous Step Size Assumption 1] \label{assumption_async_step_size_1}
Let \([\cdot]\) denote the integer part of \((\cdot)\). For \(x \in (0, 1)\), 
\begin{align*}
    \sup_i \frac{\alpha_{[xi]}}{\alpha_i} < \infty
\end{align*}
and 
\begin{align*}
    \frac{\sum_{j=0}^{[yi]} \alpha_j}{\sum_{j=0}^i \alpha_j} \to 1
\end{align*} 
uniformly in \(y \in [x, 1]\).\\
\end{assumption}

\begin{assumption}[Asynchronous Step Size Assumption 2] \label{assumption_async_step_size_td_2}
There exists \(\Delta > 0\) such that 
\begin{align*}
    \liminf_{n \to \infty} \frac{\nu(n, s)}{n+1} \geq \Delta ,
\end{align*}
almost surely, for all \(s \in \mathcal{S}\). 

Furthermore, for all \(x > 0\), and
\begin{align*}
    N(n, x) = \min \Bigg \{m \geq n: \sum_{i = n+1}^m \alpha_i \geq x \Bigg\} ,
\end{align*}
the limit 
\begin{align*}
    \lim_{n \to \infty} \frac{\sum_{i = \nu(n, s)}^{\nu(N(n, x), s)} \alpha_i}{\sum_{i = \nu(n, s')}^{\nu(N(n, x), s')} \alpha_i}
\end{align*} 

exists almost surely for all \(s, s'\).\\
\end{assumption}

\begin{assumption}[Quantile Step Size Assumption] \label{assumption_quantile_stepsize} In addition to the conditions listed in Assumption \ref{assumption_async_step_size_1}, the quantile step sizes, \(\{\alpha_{_\theta, n}\}_{n \geq 0}\), satisfy the following properties: \(\alpha_n / (\alpha_{_\theta, n}) \to 0\) and \(\sum_{n = 0}^\infty (\alpha_n^2 + \alpha_{_\theta, n}^2)< \infty\), where \(\{\alpha_n\}_{n \geq 0}\) denote the value function step sizes.\\
\end{assumption}

We are now ready to state the convergence theorem:\\

\begin{theorem}[Convergence of D2 TD-learning Value Function Estimates]\label{theorem_convergence_of_d2_td_update}

If Assumptions~\ref{assumption_unichain_appendix} -- \ref{assumption_quantile_stepsize} hold, then the update rule \eqref{async_td_value_update_eqn} converges, almost surely, \(V_n \to v_\infty \text{ as } n \to \infty\), thereby implying that the D2 TD-learning value function estimates \eqref{d2_td_value_update} converge almost surely to \(v_\infty \text{ as } n \to \infty\), where \(v_\infty\) denotes the solution to the Bellman equation (4).
\end{theorem}

We prove this theorem in the following section. To do so, we use a two-timescales argument, such that we leverage Theorem 2 in Section 6 of \citet{Borkar2009-sr}, along with the results from Theorems 4.3 and \ref{thm_q_convergence}.

\paragraph{Proof of Theorem \ref{theorem_convergence_of_d2_td_update}:\\}

To begin, let us consider Assumption \ref{assumption_quantile_stepsize}. In particular, \(\alpha_n / (\alpha_{_\theta, n}) \to 0\) implies that the quantile step size, \(\alpha_{_\theta, n}\), decreases to 0 at a slower rate than the value function step size, \(\alpha_n\). This implies that the quantile updates move on a faster timescale compared to the value function update. Hence, as argued in Section 6 of \citet{Borkar2009-sr}, the (faster) quantile updates \eqref{eqn_proofs_3} view the (slower) value function update \eqref{async_td_value_update_eqn} as quasi-static, while the (slower) value function update views the (faster) quantile updates as nearly equilibrated (as we will discuss below, the existence of such an equilibrium point was shown in Theorem \ref{thm_q_convergence}).

That is, given the results of Lemma 4.2, Theorem \ref{thm_q_convergence}, and Theorem 4.3, the two-timescales argument allows us to view the average-reward estimate, \(\bar{R}_n\), as nearly equilibrated at \(\bar{r}_{\pi}\). In particular, given Lemma 4.2, we have that \(\bar{R}_n \doteq \frac{1}{m}\sum_{i=1}^{m}\theta_{i, n}\), where \(m\) denotes the number of quantiles, and \(\theta_{i, n}\) denotes the \(i^\text{th}\) quantile estimate at step \(n\). Now, given Theorem \ref{thm_q_convergence}, we know that the quantile estimates converge to the desired quantiles of the limiting per-step reward distribution. As such, by Theorem 4.3, we have that \(\bar{R}_n \to \bar{r}_{\pi}\). Consequently, by the two-timescales argument, the value function update views the average-reward estimate, \(\bar{R}_n\), as nearly equilibrated at \(\bar{r}_{\pi}\) (that is, we can essentially treat \(\bar{r}_{\pi}\) as a known constant).

We now proceed to show the convergence of the value function estimates. To this end, consider the asynchronous value function update \eqref{async_td_value_update_eqn}. This update is of the same form as Equation 7.1.2 of \citet{Borkar2009-sr}. Accordingly, to show the convergence of the value function estimates, we can apply the result in Section 7.4 of \citet{Borkar2009-sr}, which shows the convergence of asynchronous updates of the same form as Equation 7.1.2. To apply this result, given Assumptions \ref{assumption_async_step_size_1} and \ref{assumption_async_step_size_td_2}, we only need to show the convergence of the \emph{synchronous} version of the value function update:
\begin{equation}
\label{sync_d2_update_td}
    V_{n+1}(s) = V_{n}(s) + \alpha_n \left(h(V_n) + M_{n+1}\right), \quad \forall s \in \mathcal{S},
\end{equation}
where, 
\begin{align*}
    h(V_n)(s) & \doteq \sum_{a}\pi(a\mid s)\sum_{s', r} p(s', r \mid s, a) (r - \bar{r}_{\pi} + V_n(s')) - V_n(s), \\
    & = T(V_n)(s) - V_n(s) - \bar{r}_{\pi}, \text{ and}\\
     M_{n + 1}(s) & \doteq \left( R_n(s, A_n(s)) - \bar{r}_{\pi} + V_n(S_n'(s, A_n(s))) - V_n(s) \right) - h(V_n)(s).
\end{align*}

To show the convergence of the synchronous update \eqref{sync_d2_update_td} under the two-timescales argument, we can apply the result of Theorem 2 in Section 6 of \citet{Borkar2009-sr} to show that \(V_n \to v_{\infty}\) almost surely as \(n \to \infty\). This theorem requires that 3 assumptions be satisfied. As such, we will now show, via Lemmas \ref{lemma_td_1} - \ref{lemma_td_3}, that these 3 assumptions are indeed satisfied.\\

\begin{lemma}
\label{lemma_td_1}
The quantile update rules, which can be written in the following form: \( \theta_{i, n+1} = \theta_{i, n} + \alpha_{_\theta, n} (g(\theta_{i, n}) + M_{i, n})\) (see Lemma \ref{martingale_lemma}), each have a globally asymptotically stable equilibrium, \(\theta_i(V_n), \; \forall i = 1, 2, \ldots m\), where \(m\) denotes the number of quantiles, and \(\theta_i\) is a Lipschitz map with respect to \(V_n\).
\end{lemma}
\begin{proof}
This was shown in Theorem \ref{thm_q_convergence}. We note that since \(\theta_i\) is not directly dependent on \(V_n\) (i.e., \(\theta_i(V_n) = \theta_i\)), the Lipschitz condition holds trivially.
\end{proof}

\begin{lemma}
\label{lemma_td_2}
The value function update rule \eqref{sync_d2_update_td} has a globally asymptotically stable equilibrium, \(v_{\infty}\).
\end{lemma}
\begin{proof}
The value function update rule \eqref{sync_d2_update_td} is of the same form as the update rule considered in Lemmas B.6 - B.9 of \citet{Wan2021-re} (with the importance sampling ratio set to \(1.0\) and \(f(V) = \bar{r}_\pi\)). Given that we adhere to all the assumptions required for the aforementioned Lemmas from \citet{Wan2021-re}, we can leverage their results, which show that the update rule \eqref{sync_d2_update_td} has a globally asymptotically stable equilibrium, \(v_{\infty}\). 
\end{proof}

\begin{lemma}
\label{lemma_td_3}
\(\sup_n(\vert\vert V_n \vert\vert + \sum_{i=1}^m{\vert\vert\theta_{i,n} \vert\vert)} < \infty\) almost surely.
\end{lemma}
\begin{proof}
It was shown in Lemma \ref{boundedness_lemma} that \(\sup_n(\vert\vert \theta_{i,n} \vert\vert) < \infty \; \forall i = 1, 2, \ldots, m\) almost surely. Hence, we only need to show that \(\sup_n(\vert\vert V_n \vert\vert) < \infty\) almost surely. To this end, we can apply Theorem 7 in Section 3 of \citet{Borkar2009-sr}. This theorem requires 4 assumptions:
\begin{itemize}
    \item \textbf{(A1)} The function \(h\) is Lipschitz. That is, \(\vert \vert h(x) - h(y)\vert \vert \leq L \vert \vert x - y\vert \vert\) for some \(0 < L < \infty\).
    \item \textbf{(A2)} The sequence \(\{ \alpha_n\}\) satisfies \(\alpha_n > 0\), \(\sum \alpha_n = \infty\), and \(\sum \alpha_n^2 < \infty\).
    \item \textbf{(A3)} \(\{M_n\}\) is a martingale difference sequence that is square-integrable.
    \item \textbf{(A4)} The functions \(h_d(x) \doteq h(dx)/d\), \(d \geq 1, x \in \mathbb{R}^{\vert \mathcal{S} \vert}\), satisfy \(h_d(x) \to h_\infty(x)\) as \(d \to \infty\), uniformly on compacts for some \(h_\infty \in C(\mathbb{R}^{\vert \mathcal{S} \vert})\). Furthermore, the ODE \(\dot x_t = h_\infty(x_t)\) has the origin as its unique globally asymptotically stable equilibrium.
\end{itemize}

We first note that Assumption \textbf{(A1)} is satisfied given that both the \(\sum\) and \(T\) operators are Lipschitz. Moreover, we note that Assumption \ref{assumption_step_size} satisfies Assumption \textbf{(A2)}.\\ 

We now show that Assumption \textbf{(A3)} is satisfied:

Let \(\mathcal{F}_n \doteq \sigma(V_i, M_i; \; 0 \leq i \leq n)\) denote an increasing family of \(\sigma\)-fields. We have that, for any \(s \in \mathcal{S}\):
\begin{align}
\mathbb{E}[M_{n+1}(s) \mid \mathcal{F}_n]
    & = \mathbb{E} \left[R_n(s, A_n(s)) - \bar{r}_{\pi} + V_n(S_n'(s, A_n(s))) - V_n(s) - h(V_n)(s) \mid \mathcal{F}_n \right] \\
    & = \mathbb{E}\left[R_n(s, A_n(s)) - \bar{r}_{\pi} + V_n(S_n'(s, A_n(s))) - V_n(s) \mid \mathcal{F}_n \right] - h(V_n)(s) \\
    & = \mathbb{E}\left[R_n(s, A_n(s)) + V_n(S_n'(s, A_n(s))) \mid \mathcal{F}_n \right] - V_n(s) - \bar{r}_{\pi} - h(V_n)(s) \\
    & = \mathbb{E}[R_n(s, A_n(s)) + V_n(S_n'(s, A_n(s))) \mid \mathcal{F}_n] - T(V_n)(s) \\
     &= 0.
\end{align}
Moreover, it can be easily shown via the triangle inequality that \(\mathbb{E} [\vert\vert M_{n+1}\vert\vert^2 \mid \mathcal{F}_n] \leq K (1 + \vert\vert V_n \vert\vert^2)\) for some constant \(K > 0\). As such, Assumption \textbf{(A3)} is satisfied.\\

Finally, we show that Assumption \textbf{(A4)} is also satisfied:

We first note that:
\begin{align}
    h_\infty(x) = \lim_{d \to \infty} h_d(x) = \lim_{d \to \infty} \frac{T(dx) - dx - \bar{r}_{\pi}e}{d} = T_0 (x) - x,
\end{align}
where, 
\begin{align}
    T_0 (x) \doteq \sum_a \pi(a \mid s) \sum_{s', r} p (s', r \mid s, a) x(s') .
\end{align}
Clearly, the function \(h_\infty\) is continuous in every \(x \in \mathbb{R}^{\vert \mathcal{S} \vert}\). As such, we have that \(h_\infty \in C(\mathbb{R}^{\vert \mathcal{S} \vert})\).

Next, consider the ODE \(\dot x_t = h_\infty(x_t) = T_0 (x_t) - x_t\). This ODE has the origin as an equilibrium since \(T_0 (0) - 0 = 0\). Moreover, we can leverage the results of Lemmas B.6 and B.7 from \citet{Wan2021-re} to conclude that the origin is the unique globally asymptotically stable equilibrium for this ODE. Hence, Assumption \textbf{A4} is satisfied.\\

As such, Assumptions \textbf{(A1)} - \textbf{(A4)} are all verified, meaning that we can apply the results of Theorem 7 in Section 3 of \citet{Borkar2009-sr} to conclude that \(\sup_n(\vert\vert V_n \vert\vert) < \infty\), almost surely, and hence, that \(\sup_n(\vert\vert V_n \vert\vert + \sum_{i=1}^m{\vert\vert\theta_{i,n} \vert\vert)} < \infty\), almost surely.
\end{proof}

As such, we have now verified the 3 assumptions required by Theorem 2 in Section 6 of \citet{Borkar2009-sr}, which means that we can apply the result of the theorem to conclude that \(V_n \to v_{\infty}\) almost surely as \(n \to \infty\).\\

This completes the proof of Theorem \ref{theorem_convergence_of_d2_td_update}.\\

\paragraph{Convergence Proof for the Tabular D2 Q-learning Value Function Estimates:\\} 
\label{proof_q_value}

In this section, we present the proof for the convergence of the value function estimates of the D2 Q-learning algorithm (Algorithm \ref{alg_4}). To begin, consider an MDP \(\mathcal{M} \doteq \langle\mathcal{S}, \mathcal{A}, \mathcal{R}, p \rangle\). Given a state \(s \in \mathcal{S}\), action \(a \in \mathcal{A}\), and discrete step \(n \geq 0\), let \(R_n(s, a) \in \mathcal{R}\) denote a sample of the resulting reward, and let \(S'_n(s, a) \sim p(\cdot, \cdot \mid s, a)\) denote a sample of the resulting state. Let \(\{Y_n\}\) be a set-valued process taking values in the set of nonempty subsets of \(\mathcal{S} \times \mathcal{A}\), such that: \(Y_n = \{(s, a): (s, a)\) component of the \(\vert \mathcal{S} \times \mathcal{A} \vert\)-sized table of state-action value estimates, \(Q\), that was updated at step \(n\}\). Let \(\nu(n, s, a) \doteq \sum_{j=0}^n I\{(s, a) \in Y_j\}\), where \(I\) is the indicator function, such that \(\nu(n, s, a)\) represents the number of times that the \((s, a)\) component of \(Q\) was updated up until step \(n\). 

Now consider the following update rule for \(n \geq 0\):
\begin{align}
    Q_{n+1}(s, a) & \doteq Q_n(s, a) + \alpha_{\nu(n, s, a)} \delta_n(s, a) I\{(s, a) \in Y_n\}, \quad \forall s \in \mathcal{S}, a \in \mathcal{A}, \label{async_q_value_update_eqn}
\end{align}
where,
\begin{align}
\begin{split}
    \delta_n(s, a) & \doteq R_n(s, a) - \bar{R}_n + \max_{a'} Q_n(S_n'(s, a), a') - Q_n(s, a).
\end{split}\label{async_q_td_error_eqn}
\end{align}

Here, \(\bar{R}_n\) denotes the estimate of the average-reward induced by following policy \(\pi\), \(\bar{r}_{\pi}\) (see Equation (2)), \(\delta_n(s, a)\) denotes the TD error, and \(\alpha_{\nu(n, s, a)}\) denotes the step size at step \(n\) for state-action pair \((s, a)\).

We will now show that the value function update rule for the D2 Q-learning algorithm is a special case of the update rule \eqref{async_q_value_update_eqn}. To this end, consider a sequence of experience from our MDP \(\mathcal{M}\): \(S_t, A_t, R_{t+1}, S_{t+1}, \ldots\)\, . Now recall the set-valued process \(\{Y_n\}\). If we let \(n\) = time step \(t\), we have: 
\begin{align*}
Y_t(s, a) = 
\begin{cases}
    1, s = S_t\text{ and }a = A_t,\\
    0, \text{ otherwise,}
\end{cases}
\end{align*}
as well as \(S'_n(S_t, A_t) = S_{t+1}\), and \(R_n(S_t, A_t) = R_{t+1}\).\\ 

Hence, the update rule \eqref{async_q_value_update_eqn} becomes:
\begin{align}
\label{d2_q_value_update}
    Q_{t+1}(S_t, A_t) & \doteq Q_t (S_t, A_t) + \alpha_{\nu(t, S_t, A_t)} \delta_t \text{ \; and \; } Q_{t+1}(s, a) \doteq Q_t (s, a), \forall s \neq S_t, a \neq A_t,
\end{align}

which is D2 Q-learning's value function update rule with \(\alpha_{\nu(t, S_t, A_t)}\) denoting the step size at time \(t\).\\

As such, if we can show that the update rule \eqref{async_q_value_update_eqn} converges, then it would imply that the value function update rule for the D2 Q-learning algorithm also converges. To this end, we will now specify the assumptions that are needed to ensure convergence:\\ 

\begin{assumption}[Communicating Assumption] \label{assumption_communicating_appendix}
The MDP has a single communicating class. That is, each state in the MDP is accessible from every other state under some deterministic stationary policy.\\
\end{assumption}

\begin{assumption}[State-Action Value Function Uniqueness] \label{assumption_action_value_function_uniqueness}
There exists a unique solution of \(q\) only up to a constant in the Bellman equation (5).\\
\end{assumption}

\begin{assumption}[Asynchronous Step Size Assumption 3] \label{assumption_async_step_size_q_2}
There exists \(\Delta > 0\) such that 
\begin{align*}
    \liminf_{n \to \infty} \frac{\nu(n, s, a)}{n+1} \geq \Delta,
\end{align*}
almost surely, for all \(s \in \mathcal{S}, a \in \mathcal{A}\).\\ 

Furthermore, for all \(x > 0\), and
\begin{align*}
    N(n, x) = \min \Bigg \{m > n: \sum_{i = n+1}^m \alpha_i \geq x \Bigg \} ,
\end{align*}
the limit 
\begin{align*}
    \lim_{n \to \infty} \frac{\sum_{i = \nu(n, s, a)}^{\nu(N(n, x), s, a)} \alpha_i}{\sum_{i = \nu(n, s', a')}^{\nu(N(n, x), s', a')} \alpha_i}
\end{align*}

exists almost surely for all \(s, s', a, a'\).\\
\end{assumption}

We are now ready to state the convergence theorem:\\

\begin{theorem}[Convergence of D2 Q-learning Value Function Estimates]\label{theorem_convergence_of_d2_q_update}

If Assumptions~\ref{assumption_step_size}, \ref{assumption_async_step_size_1},
\ref{assumption_quantile_stepsize},
\ref{assumption_communicating_appendix}, \ref{assumption_action_value_function_uniqueness}, and \ref{assumption_async_step_size_q_2} hold, then the update rule \eqref{async_q_value_update_eqn} converges, almost surely, \(Q_n \to q_\infty \text{ as } n \to \infty\), thereby implying that the D2 Q-learning value function estimates \eqref{d2_q_value_update} converge almost surely to \(q_\infty \text{ as } n \to \infty\), where \(q_\infty\) denotes the solution to the Bellman optimality equation (5).
\end{theorem}

We prove this theorem in the following section. To do so, we use a two-timescales argument, such that we leverage Theorem 2 in Section 6 of \citet{Borkar2009-sr}, along with the results from Theorems 4.3 and \ref{thm_q_convergence}.

\paragraph{Proof of Theorem \ref{theorem_convergence_of_d2_q_update}:\\}
\label{proof_d2_q}

To begin, let us consider Assumption \ref{assumption_quantile_stepsize}. In particular, \(\alpha_n / (\alpha_{_\theta, n}) \to 0\) implies that the quantile step size, \(\alpha_{_\theta, n}\), decreases to 0 at a slower rate than the value function step size, \(\alpha_n\). This implies that the quantile updates move on a faster timescale compared to the value function update. Hence, as argued in Section 6 of \citet{Borkar2009-sr}, the (faster) quantile updates \eqref{eqn_proofs_3} view the (slower) value function update \eqref{async_q_value_update_eqn} as quasi-static, while the (slower) value function update views the (faster) quantile updates as nearly equilibrated (as we will discuss below, the existence of such an equilibrium point was shown in Theorem \ref{thm_q_convergence}).

That is, given the results of Lemma 4.2, Theorem \ref{thm_q_convergence}, and Theorem 4.3, the two-timescales argument allows us to view the average-reward estimate, \(\bar{R}_n\), as nearly equilibrated at \(\bar{r}_{\pi}\). In particular, given Lemma 4.2, we have that \(\bar{R}_n \doteq \frac{1}{m}\sum_{i=1}^{m}\theta_{i, n}\), where \(m\) denotes the number of quantiles, and \(\theta_{i, n}\) denotes the \(i^\text{th}\) quantile estimate at step \(n\). Now, given Theorem \ref{thm_q_convergence}, we know that the quantile estimates converge to the desired quantiles of the limiting per-step reward distribution. As such, by Theorem 4.3, we have that \(\bar{R}_n \to \bar{r}_{\pi}\). Consequently, by the two-timescales argument, the value function update views the average-reward estimate, \(\bar{R}_n\), as nearly equilibrated at \(\bar{r}_{\pi}\) (that is, we can essentially treat \(\bar{r}_{\pi}\) as a known constant).

We now proceed to show the convergence of the value function estimates. To this end, consider the asynchronous value function update \eqref{async_q_value_update_eqn}. This update is of the same form as Equation 7.1.2 of \citet{Borkar2009-sr}. Accordingly, to show the convergence of the value function estimates, we can apply the result in Section 7.4 of \citet{Borkar2009-sr}, which shows the convergence of asynchronous updates of the same form as Equation 7.1.2. To apply this result, given Assumptions \ref{assumption_async_step_size_1} and \ref{assumption_async_step_size_q_2}, we only need to show the convergence of the \emph{synchronous} version of the value function update:
\begin{equation}
\label{sync_d2_update_q}
    Q_{n+1} = Q_{n} + \alpha_n (h(Q_n) + M_{n+1}), \quad \forall s \in \mathcal{S}, a \in \mathcal{A},
\end{equation}
where, 
\begin{align*}
    h(Q_n)(s, a) & \doteq \sum_{s', r} p (s', r \mid s, a) (r - \bar{r}_{\pi} + \max_{a'} Q_n(s', a')) - Q_n(s, a), \\
    & = T(Q_n)(s, a) - Q_n(s, a) - \bar{r}_{\pi}, \text{ and}\\
    M_{n + 1}(s, a) & \doteq \left( R_n(s, a) - \bar{r}_{\pi} + \max_{a'} Q_n(S_n'(s, a), a') - Q_n(s, a) \right) - h(Q_n)(s, a).
\end{align*}

To show the convergence of the synchronous update \eqref{sync_d2_update_q} under the two-timescales argument, we can apply the result of Theorem 2 in Section 6 of \citet{Borkar2009-sr} to show that \(Q_n \to q_{\infty}\) almost surely as \(n \to \infty\). This theorem requires that 3 assumptions be satisfied. As such, we will now show, via Lemmas \ref{lemma_q_1} - \ref{lemma_q_3}, that these 3 assumptions are indeed satisfied.\\

\begin{lemma}
\label{lemma_q_1}
The quantile update rules, which can be written in the following form: \( \theta_{i, n+1} = \theta_{i, n} + \alpha_{_\theta, n} (g(\theta_{i, n}) + M_{i, n})\) (see Lemma \ref{martingale_lemma}), each have a globally asymptotically stable equilibrium, \(\theta_i(Q_n), \; \forall i = 1, 2, \ldots m\), where \(m\) denotes the number of quantiles, and \(\theta_i\) is a Lipschitz map with respect to \(Q_n\).
\end{lemma}
\begin{proof}
This was shown in Theorem \ref{thm_q_convergence}. We note that since \(\theta_i\) is not directly dependent on \(Q_n\) (i.e., \(\theta_i(Q_n) = \theta_i\)), the Lipschitz condition holds trivially.
\end{proof}

\begin{lemma}
\label{lemma_q_2}
The value function update rule \eqref{sync_d2_update_q} has a globally asymptotically stable equilibrium, \(q_{\infty}\).
\end{lemma}
\begin{proof}
The value function update rule \eqref{sync_d2_update_q} is of the same form as the update rule considered in Lemmas B.1 - B.4 of \citet{Wan2021-re} (with \(f(Q) = \bar{r}_\pi\)). Given that we adhere to all the assumptions required for the aforementioned Lemmas from \citet{Wan2021-re}, we can leverage their results, which show that the update rule \eqref{sync_d2_update_q} has a globally asymptotically stable equilibrium, \(q_{\infty}\). 
\end{proof}

\begin{lemma}
\label{lemma_q_3}
\(\sup_n(\vert\vert Q_n \vert\vert + \sum_{i=1}^m{\vert\vert\theta_{i,n} \vert\vert)} < \infty\) almost surely.
\end{lemma}
\begin{proof}
It was shown in Lemma \ref{boundedness_lemma} that \(\sup_n(\vert\vert \theta_{i,n} \vert\vert) < \infty \; \forall i = 1, 2, \ldots, m\) almost surely. Hence, we only need to show that \(\sup_n(\vert\vert Q_n \vert\vert) < \infty\) almost surely. To this end, we can apply Theorem 7 in Section 3 of \citet{Borkar2009-sr}. This theorem requires 4 assumptions:
\begin{itemize}
    \item \textbf{(A1)} The function \(h\) is Lipschitz. That is, \(\vert \vert h(x) - h(y)\vert \vert \leq L \vert \vert x - y\vert \vert\) for some \(0 < L < \infty\).
    \item \textbf{(A2)} The sequence \(\{ \alpha_n\}\) satisfies \(\alpha_n > 0\), \(\sum \alpha_n = \infty\), and \(\sum \alpha_n^2 < \infty\).
    \item \textbf{(A3)} \(\{M_n\}\) is a martingale difference sequence that is square-integrable.
    \item \textbf{(A4)} The functions \(h_d(x) \doteq h(dx)/d\), \(d \geq 1, x \in \mathbb{R}^{\vert \mathcal{S} \times \mathcal{A} \vert}\), satisfy \(h_d(x) \to h_\infty(x)\) as \(d \to \infty\), uniformly on compacts for some \(h_\infty \in C(\mathbb{R}^{\vert \mathcal{S} \times \mathcal{A} \vert})\). Furthermore, the ODE \(\dot x_t = h_\infty(x_t)\) has the origin as its unique globally asymptotically stable equilibrium.
\end{itemize}

We first note that Assumption \textbf{(A1)} is satisfied given that both the \(\sum\) and \(T\) operators are Lipschitz. Moreover, we note that Assumption \ref{assumption_step_size} satisfies Assumption \textbf{(A2)}.\\ 

We now show that Assumption \textbf{(A3)} is satisfied:

Let \(\mathcal{F}_n \doteq \sigma(Q_i, M_i; \;0 \leq i \leq n)\) denote an increasing family of \(\sigma\)-fields. We have that, for any \(s \in \mathcal{S}, a \in \mathcal{A}\):
\begin{align}
\mathbb{E}[M_{n+1}(s, a) \mid \mathcal{F}_n]
    & = \mathbb{E} \left[R_n(s, a) - \bar{r}_{\pi} + \max_{a'}Q_n(S_n'(s, a), a') - Q_n(s, a) - h(Q_n)(s, a) \mid \mathcal{F}_n \right] \\
    & = \mathbb{E}\left[R_n(s, a) - \bar{r}_{\pi} + \max_{a'}Q_n(S_n'(s, a), a') - Q_n(s, a) \mid \mathcal{F}_n \right] - h(Q_n)(s, a) \\
    & = \mathbb{E}\left[R_n(s, a) + \max_{a'}Q_n(S_n'(s, a), a') \mid \mathcal{F}_n \right] - Q_n(s, a) - \bar{r}_{\pi} - h(Q_n)(s, a) \\
    & = \mathbb{E}[R_n(s, a) + \max_{a'}Q_n(S_n'(s, a), a') \mid \mathcal{F}_n] - T(Q_n)(s, a) \\
     &= 0.
\end{align}
Moreover, it can be easily shown via the triangle inequality that \(\mathbb{E} [\vert\vert M_{n+1}\vert\vert^2 \mid \mathcal{F}_n] \leq K (1 + \vert\vert Q_n \vert\vert^2)\) for some constant \(K > 0\). As such, Assumption \textbf{(A3)} is satisfied.\\

Finally, we show that Assumption \textbf{(A4)} is also satisfied:

We first note that:
\begin{align}
    h_\infty(x) = \lim_{d \to \infty} h_d(x) = \lim_{d \to \infty} \frac{T(dx) - dx - \bar{r}_{\pi}e}{d} = T_0 (x) - x,
\end{align}
where, 
\begin{align}
    T_0 (x) \doteq \sum_{s', r} p (s', r \mid  s, a) \max_{a'} x(s', a').
\end{align}
Clearly, the function \(h_\infty\) is continuous in every \(x \in \mathbb{R}^{\vert \mathcal{S} \times \mathcal{A} \vert}\). As such, we have that \(h_\infty \in C(\mathbb{R}^{\vert \mathcal{S} \times \mathcal{A} \vert})\).

Next, consider the ODE \(\dot x_t = h_\infty(x_t) = T_0 (x_t) - x_t\). This ODE has the origin as an equilibrium since \(T_0 (0) - 0 = 0\). Moreover, we can leverage the results of Lemmas B.1 and B.2 from \citet{Wan2021-re} to conclude that the origin is the unique globally asymptotically stable equilibrium for this ODE. Hence, Assumption \textbf{A4} is satisfied.\\

As such, Assumptions \textbf{(A1)} - \textbf{(A4)} are all verified, meaning that we can apply the results of Theorem 7 in Section 3 of \citet{Borkar2009-sr} to conclude that \(\sup_n(\vert\vert Q_n \vert\vert) < \infty\), almost surely, and hence, that \(\sup_n(\vert\vert Q_n \vert\vert + \sum_{i=1}^m{\vert\vert\theta_{i,n} \vert\vert)} < \infty\), almost surely.
\end{proof}

As such, we have now verified the 3 assumptions required by Theorem 2 in Section 6 of \citet{Borkar2009-sr}, which means that we can apply the result of the theorem to conclude that \(Q_n \to q_{\infty}\) almost surely as \(n \to \infty\).\\

This completes the proof of Theorem \ref{theorem_convergence_of_d2_q_update}.


\clearpage

\section{Numerical Experiments}
\label{appendix_experiments}

This appendix contains details regarding the numerical experiments performed as part of this work. The overall aim of the experiments was to contrast and compare our Differential Distributional RL algorithms (see Appendix \ref{appendix_algs}) with non-distributional differential (i.e., average-reward RL) algorithms. In particular, we aimed to: 1) show how the Differential Distributional RL algorithms could be utilized to find the average-reward-optimal policy while also learning the corresponding optimal per-step reward distribution, and 2) contrast the results to those of non-distributional algorithms, which served as a sort of ‘baseline’ to illustrate how our distributional approach compares to a non-distributional one. In this work, we utilized non-distributional average-reward RL algorithms derived from the Differential framework proposed in \citet{Wan2021-re}. The aforementioned (non-distributional) Differential algorithms used for comparison are included in Section \ref{baseline_algorithms}.

We ran two groups of experiments. In the first group of experiments, we aimed to validate whether the D2 and D3 algorithms could successfully learn the optimal per-step reward distribution. Accordingly, we conducted experiments in environments where the optimal per-step reward distribution is known. In the second group of experiments, we aimed to evaluate the empirical performance of the D2 and D3 algorithms in more difficult environments. The first group of experiments is described in Section \ref{exp_toy}, and the second group is described in Section \ref{exp_atari}.

\subsection{Validation Experiments}
\label{exp_toy}
In this section, we discuss the experiments performed in the \emph{red-pill blue-pill} environment \citep{Rojas2025-bf}, as well as in the well-known \emph{inverted pendulum} environment. Through these experiments, we aimed to validate whether the D2 and D3 algorithms could successfully learn the optimal per-step reward distribution (i.e., these experiments correspond to environments where the optimal per-step reward distribution is known). 

\begin{figure}[htbp]
\centerline{\includegraphics[scale=0.6]{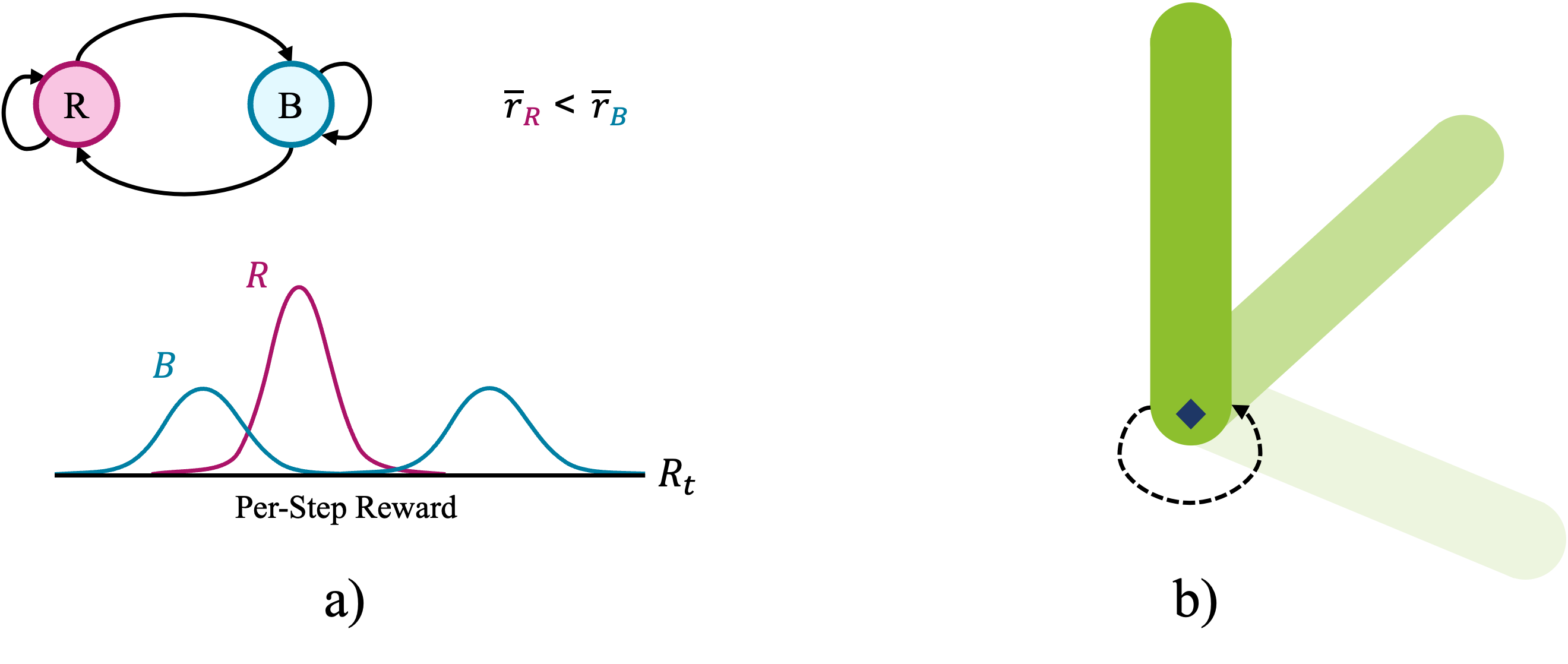}}
\caption{An illustration of the \textbf{a)} red-pill blue-pill, and \textbf{b)} inverted pendulum environments.}
\label{fig_experiments}
\end{figure}

\paragraph{Red-Pill Blue-Pill Experiment:\\}
\label{exp_red_pill_blue_pill}

In the first experiment, we consider the two-state \emph{red-pill blue-pill} environment, where at every time step, an agent can take either a ‘red pill’, which takes them to the ‘red world’ state, or a ‘blue pill’, which takes them to the ‘blue world’ state. Each state has its own characteristic per-step reward distribution, such that the blue world state has a reward distribution with a higher (better) average-reward compared to the red world state. Hence, we would expect that the D2 and D3 Q-learning algorithms learn a policy that prefers to stay in the blue world state, as well as the corresponding limiting per-step reward distribution induced by such a policy. Similarly, we would also expect that the Differential Q-learning algorithm learns a policy that prefers to stay in the blue world state. This task is illustrated in Figure \ref{fig_experiments}a).

For this experiment, we compared our tabular D2 and D3 Q-learning algorithms (Algorithms \ref{alg_4} and \ref{alg_8}, respectively) to the tabular Differential Q-learning algorithm (Algorithm \ref{alg_reg_1}). We ran each algorithm using various combinations of step sizes and number of quantiles (where appropriate). We used an \(\varepsilon\)-greedy policy with a fixed epsilon of 0.1. We used the same number of per-step reward quantiles and differential return quantiles (i.e., \(m=n\)). We set all initial guesses to zero. We ran the algorithms for 100k time steps.

For the Differential Q-learning algorithm, we tested every combination of the value function step size, \(\alpha\in\{\text{2e-5, 2e-4, 2e-3, 2e-2, 2e-1}\}\), with the average-reward step size, \(\eta\alpha\), where \(\eta\in\{\text{1e-3, 1e-2, 1e-1, 1.0, 2.0, 10.0, 100.0}\}\), for a total of 35 unique combinations. Each combination was run 50 times using different random seeds, and the results were averaged across the runs. The resulting (averaged) average-reward across all time steps is displayed in Figure \ref{fig_tuning_1}a). As shown in the figure, a value function step size of 2e-3 and an average-reward \(\eta\) of 2.0 resulted in the highest average-reward across all time steps in the red-pill blue-pill task. These are the parameters used to generate the results displayed in Figure 3. In Figure 3, the 95\% confidence interval is over 50 runs.
\begin{figure}[htbp]
\centerline{\includegraphics[scale=0.53]{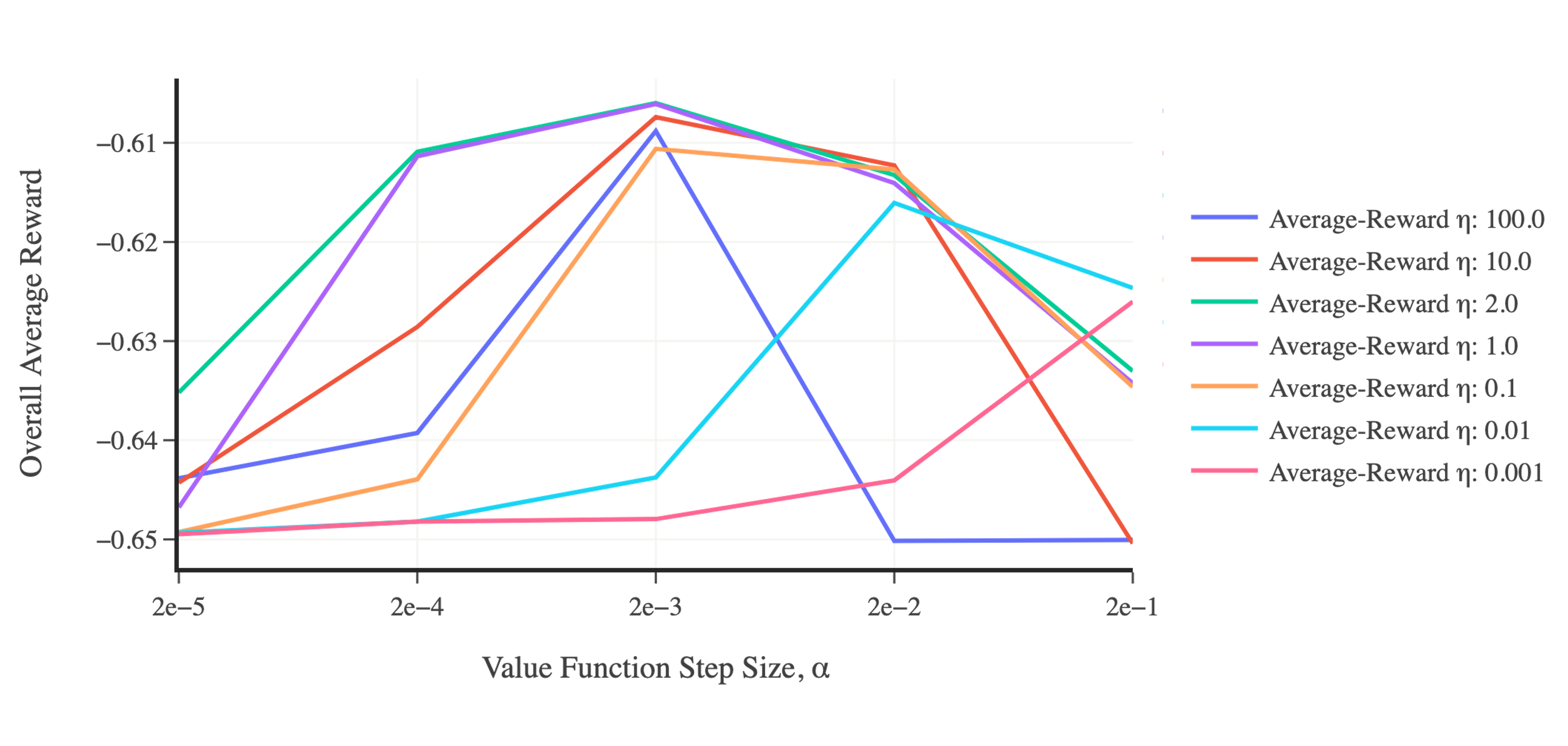}}
\caption{Step size tuning results for the red-pill blue-pill task when using the Differential Q-learning algorithm. The average-reward across all time steps is displayed for various combinations of step size parameters.}
\label{fig_tuning_1}
\end{figure}

For the D2 Q-learning algorithm, we tested every combination of the number of quantiles, \(n=m \in \{\text{10, 51}\}\), with the value function step size, \(\alpha\in\{\text{2e-5, 2e-4, 2e-3, 2e-2, 2e-1}\}\), and the per-step reward quantile step size, \(\alpha_{_\theta} \doteq \eta_{_\theta}\alpha\), where \(\eta_{_\theta} \in \{\text{1e-3, 1e-2, 1e-1, 1.0, 2.0, 10.0, 100.0}\}\), for a total of 70 unique combinations. Each combination was run 50 times using different random seeds, and the results were averaged across the runs. A value function step size of 2e-3, a reward quantile \(\eta_{_\theta}\) of 2.0, and the number of quantiles set to 51 yielded the best results and were used to generate the results displayed in Figure 3. In Figure 3, the 95\% confidence interval is over 50 runs. A value function step size of 2e-4, a per-step reward quantile \(\eta_{_\theta}\) of 2.0, and the number of quantiles set to 10 were used to generate the results displayed in Figure 2.\\ 

For the D3 Q-learning algorithm, we tested every combination of the number of quantiles, \(n=m \in \{\text{10, 51}\}\), with the differential return quantile step size, \(\alpha\in\{\text{2e-5, 2e-4, 2e-3, 2e-2, 2e-1}\}\), and the per-step reward quantile step size, \(\alpha_{_\theta} \doteq \eta_{_\theta}\alpha\), where \(\eta_{_\theta} \in \{\text{1e-3, 1e-2, 1e-1, 1.0, 2.0, 10.0, 100.0}\}\), for a total of 70 unique combinations. Each combination was run 50 times using different random seeds, and the results were averaged across the runs. A differential return quantile step size of 2e-2, a per-step reward quantile \(\eta_{_\theta}\) of 2.0, and the number of quantiles set to 51 yielded the best results and were used to generate the results displayed in Figure 3. In Figure 3, the 95\% confidence interval is over 50 runs.\\

\begin{figure}[htbp]
\centerline{\includegraphics[scale=0.57]{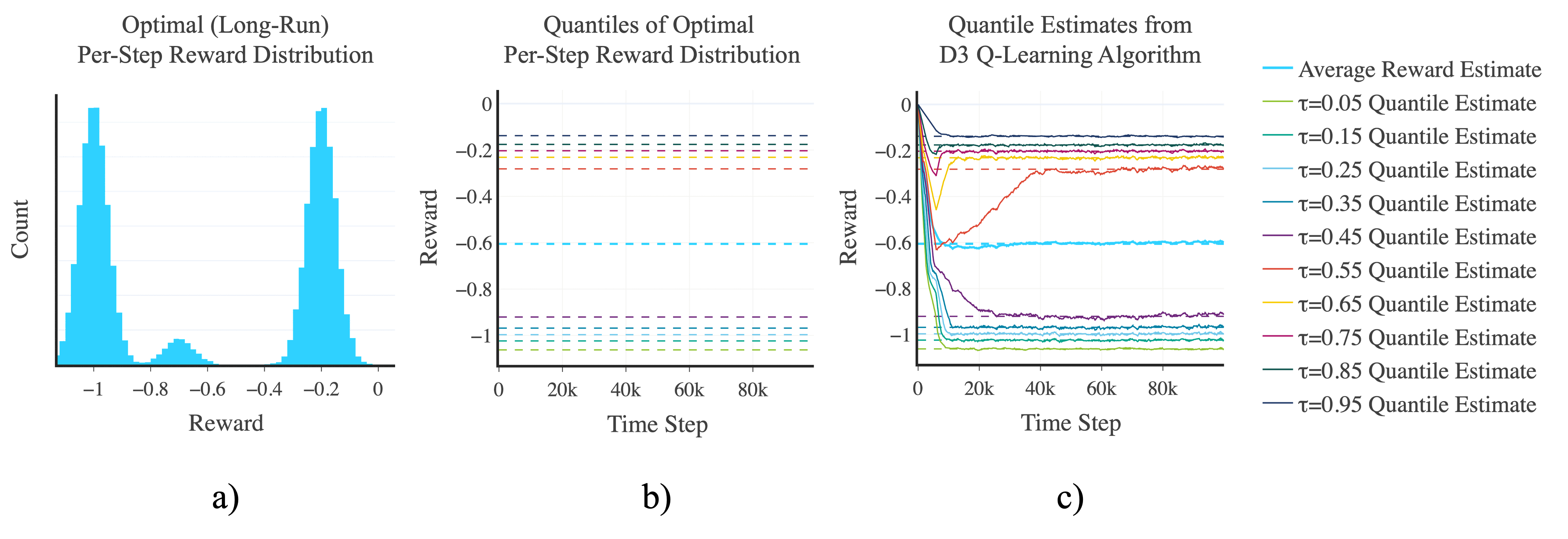}}
\caption{\textbf{a)} Histogram showing the empirical (\(\varepsilon\)-greedy) optimal (long-run) per-step reward distribution in the red-pill blue-pill task. \textbf{b)}  Quantiles of the (\(\varepsilon\)-greedy) optimal (long-run) per-step reward distribution in the red-pill blue-pill task. \textbf{c)} Convergence plot of the agent's per-step reward quantile estimates as learning progresses when using the D3 Q-learning algorithm in the red-pill blue-pill task.}
\label{fig_d3_rpbp_estimates}
\end{figure}

Figure \ref{fig_d3_rpbp_estimates} shows the agent's (per-step) reward quantile estimates as learning progresses when using the D3 Q-learning algorithm in the red-pill blue-pill task. As shown in the figure, the agent's quantile estimates converge to the quantiles of the limiting per-step reward distribution induced by the optimal policy (i.e., the policy that remains in the blue world state). A differential return quantile step size of 2e-4, a per-step reward quantile \(\eta_{_\theta}\) of 2.0, and the number of quantiles set to 10 were used to generate the results displayed in Figure \ref{fig_d3_rpbp_estimates}.

\begin{figure}[htbp]
\centerline{\includegraphics[scale=0.7]{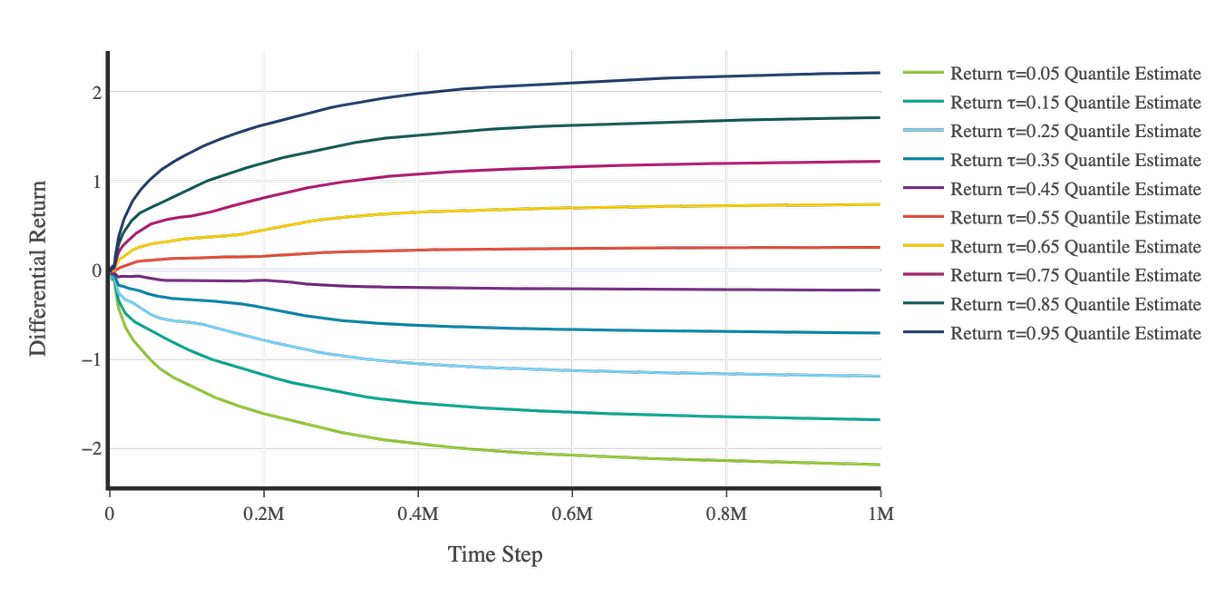}}
\caption{Convergence plot of the agent's differential return quantile estimates for \(S_t = \) ‘blue world’ and \(A_t = \) ‘blue pill’ as learning progresses when using the D3 Q-learning algorithm in the red-pill blue-pill task.}
\label{fig_d3_rpbp_return}
\end{figure}

Figure \ref{fig_d3_rpbp_return} shows the agent's differential return quantile estimates for \(S_t = \) ‘blue world’ and \(A_t = \) ‘blue pill’ as learning progresses when using the D3 Q-learning algorithm in the red-pill blue-pill task. As shown in the figure, the agent's differential return quantile estimates converge. A differential return quantile step size of 2e-4, a per-step reward quantile \(\eta_{_\theta}\) of 2.0, and the number of quantiles set to 10 were used to generate the results displayed in Figure \ref{fig_d3_rpbp_return}.\\

\paragraph{Inverted Pendulum Experiment:\\}
\label{exp_pendulum}

In the second validation experiment, we consider the well-known \emph{inverted pendulum} environment, where an agent learns how to optimally balance an inverted pendulum. We chose this task because it provides us with the opportunity to test our algorithms in an environment where: 1) we must use function approximation (given the high-dimensional state-space), and 2) where the optimal policy is known (i.e., the optimal policy is the policy that successfully balances the pendulum, thereby yielding a limiting per-step reward distribution \(\approx \textbf{0}\)). This hence allows us to directly assess the quality of the quantile estimates (to see if they all converge to \(\approx 0\)), as well as to gauge how function approximation affects the performance of our algorithms. For this task, we utilized a simple actor-critic architecture \citep{Barto1983-qr, Sutton2018-eh} as this allowed us to compare the performance of (non-tabular) D2 and D3 TD-learning algorithms with a (non-tabular) Differential TD-learning algorithm. This task is illustrated in Figure \ref{fig_experiments}b).

For this experiment, we compared our D2 and D3 Actor-Critic algorithms (Algorithms \ref{alg_5} and \ref{alg_9}, respectively) to the Differential Actor-Critic algorithm (Algorithm \ref{alg_reg_2}). We ran each algorithm using various combinations of step sizes and number of quantiles (where appropriate). We used the same number of per-step reward quantiles and differential return quantiles (i.e., \(m=n\)). We used a quantile Huber loss \(\lambda\) of 1.0. We set all initial guesses to zero. We ran the algorithms for 10k time steps. For simplicity, we used tile coding \citep{Sutton2018-eh} for both the value function and policy parameterizations, where we parameterized a softmax policy. For each parameterization, we used 32 tilings, each with 8 X 8 tiles. By using a linear function approximator (i.e., tile coding), the gradients for the value function and policy parameterizations can be simplified as follows:
\begin{equation}
\label{eq_d_1}
\nabla \hat{v}(s,\boldsymbol{w}) = \boldsymbol{x}(s),
\end{equation}
\begin{equation}
\label{eq_d_2}
\nabla \text{ln} \pi(a \mid s,\boldsymbol{\theta}) = \boldsymbol{x}_h(s,a) - \sum_{b \in \mathcal{A}} \pi(b \mid s, \boldsymbol{\theta})\boldsymbol{x}_h(s, b),
\end{equation}

where \(s \in \mathcal{S}\), \(a \in \mathcal{A}\), \(\boldsymbol{x}(s)\) is the state feature vector, and \(\boldsymbol{x}_h(s,a)\) is the softmax preference vector.\\

For the Differential Actor-Critic algorithm, we tested every combination of the value function step size, \(\alpha\in\{\text{2e-4, 2e-3, 2e-2, 2e-1}\}\), with the average-reward step size, \(\eta_{_{\bar{R}}}\alpha\), where \(\eta_{_{\bar{R}}}\in\{\text{1e-3, 1e-2, 1e-1, 1.0, 2.0}\}\), and policy \(\eta_{\pi}\in\{\text{1e-2, 1e-1, 1.0, 2.0, 10.0, 100.0}\}\), for a total of 120 unique combinations. Each combination was run 10 times using different random seeds, and the results were averaged across the runs. A value function step size of 2e-2, an average-reward \(\eta\) of 1e-2, and a policy \(\eta\) of 1e-1 yielded the best results and were used to generate the results displayed in Figure \ref{fig_pendulum_results}. In Figure \ref{fig_pendulum_results}, the 95\% confidence interval is over 10 runs.

For the D2 Actor-Critic algorithm, we tested every combination of the number of quantiles, \(n=m \in \{\text{10, 51}\}\), with the value function step size, \(\alpha\in\{\text{2e-4, 2e-3, 2e-2, 2e-1}\}\), the per-step reward quantile step size, \(\eta_{_\theta}\alpha\), where \(\eta_{_\theta}\in\{\text{1e-3, 1e-2, 1e-1, 1.0, 2.0}\}\), and policy step size, \(\eta_{\pi}\alpha\), where \(\eta_{\pi}\in\{\text{1e-2, 1e-1, 1.0, 2.0, 10.0, 100.0}\}\), for a total of 240 unique combinations. Each combination was run 10 times using different random seeds, and the results were averaged across the runs. A value function step size of 2e-2, a per-step reward quantile \(\eta\) of 1e-1, a policy  \(\eta\) of 1e-1, and the number of quantiles set to 10 yielded the best results and were used to generate the results displayed in Figures \ref{fig_pendulum_results} and \ref{pendulum_d2_estimates}. In Figure \ref{fig_pendulum_results}, the 95\% confidence interval is over 10 runs.

Figure \ref{pendulum_d2_estimates} shows the agent's (per-step) reward quantile estimates as learning progresses when using the D2 Actor-Critic algorithm in the inverted pendulum task. As shown in the figure, the agent's quantile estimates converge to the quantiles of the limiting per-step reward distribution induced by the optimal policy (i.e., the policy that balances the pendulum, thereby yielding a limiting per-step reward distribution \(\approx \textbf{0}\)).
\begin{figure}[htbp]
\centerline{\includegraphics[scale=0.57]{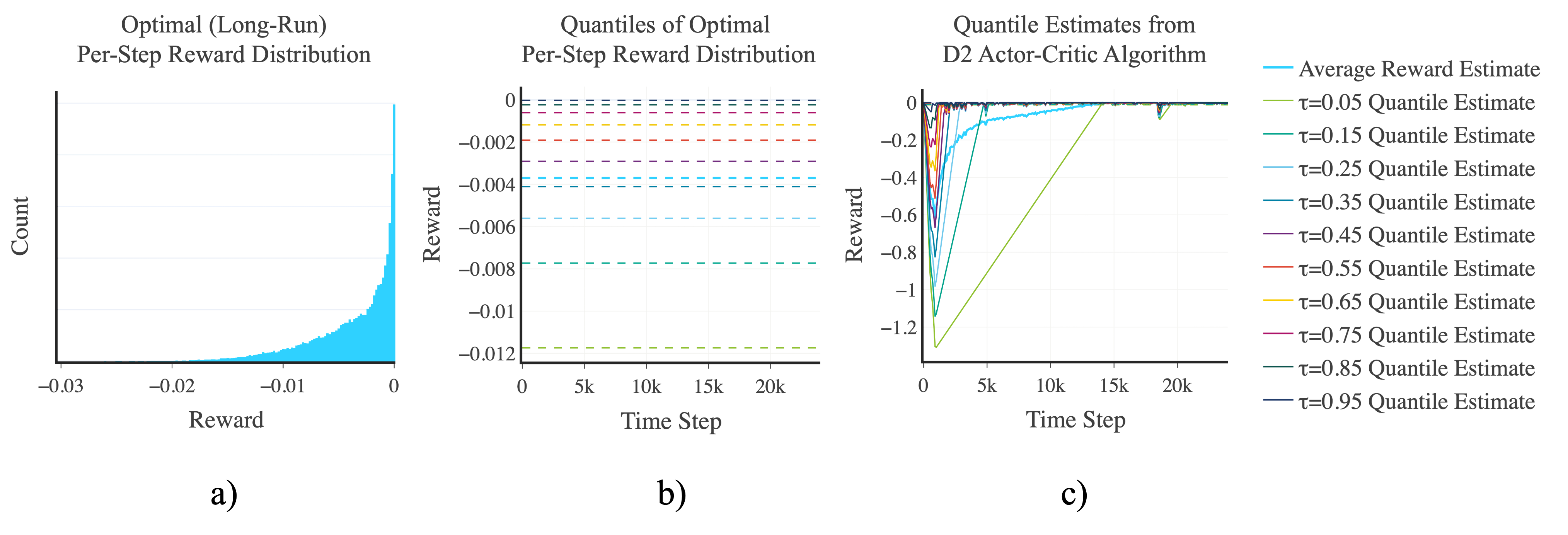}}
\caption{\textbf{a)} Histogram showing the empirical (\(\varepsilon\)-greedy) optimal (long-run) per-step reward distribution in the inverted pendulum task. \textbf{b)}  Quantiles of the (\(\varepsilon\)-greedy) optimal (long-run) per-step reward distribution in the inverted pendulum task. \textbf{c)} Convergence plot of the agent's per-step reward quantile estimates as learning progresses when using the D2 Actor-Critic algorithm in the inverted pendulum task.}
\label{pendulum_d2_estimates}
\end{figure}

For the D3 Actor-Critic algorithm, we tested every combination of the number of quantiles, \(n=m \in \{\text{10, 51}\}\), with the differential return quantile step size, \(\alpha\in\{\text{2e-4, 2e-3, 2e-2, 2e-1}\}\), the per-step reward quantile step size, \(\eta_{_\theta}\alpha\), where \(\eta_{_\theta}\in\{\text{1e-3, 1e-2, 1e-1, 1.0, 2.0}\}\), and policy step size, \(\eta_{\pi}\alpha\), where \(\eta_{\pi}\in\{\text{1e-2, 1e-1, 1.0, 2.0, 10.0, 100.0}\}\), for a total of 240 unique combinations. Each combination was run 10 times using different random seeds, and the results were averaged across the runs. A differential return quantile step size of 2e-1, a per-step reward quantile \(\eta\) of 1e-3, a policy \(\eta\) of 1e-1, and the number of quantiles set to 10 yielded the best results and were used to generate the results displayed in Figures \ref{fig_pendulum_results} and \ref{fig_d3_pendulum_estimates}. In Figure \ref{fig_pendulum_results}, the 95\% confidence interval is over 10 runs.

Figure \ref{fig_d3_pendulum_estimates} shows the agent's (per-step) reward quantile estimates as learning progresses when using the D3 Actor-Critic algorithm in the inverted pendulum task. As shown in the figure, the agent's quantile estimates converge to the quantiles of the limiting per-step reward distribution induced by the optimal policy (i.e., the policy that balances the pendulum, thereby yielding a limiting per-step reward distribution \(\approx \textbf{0}\)).
\begin{figure}[htbp]
\centerline{\includegraphics[scale=0.57]{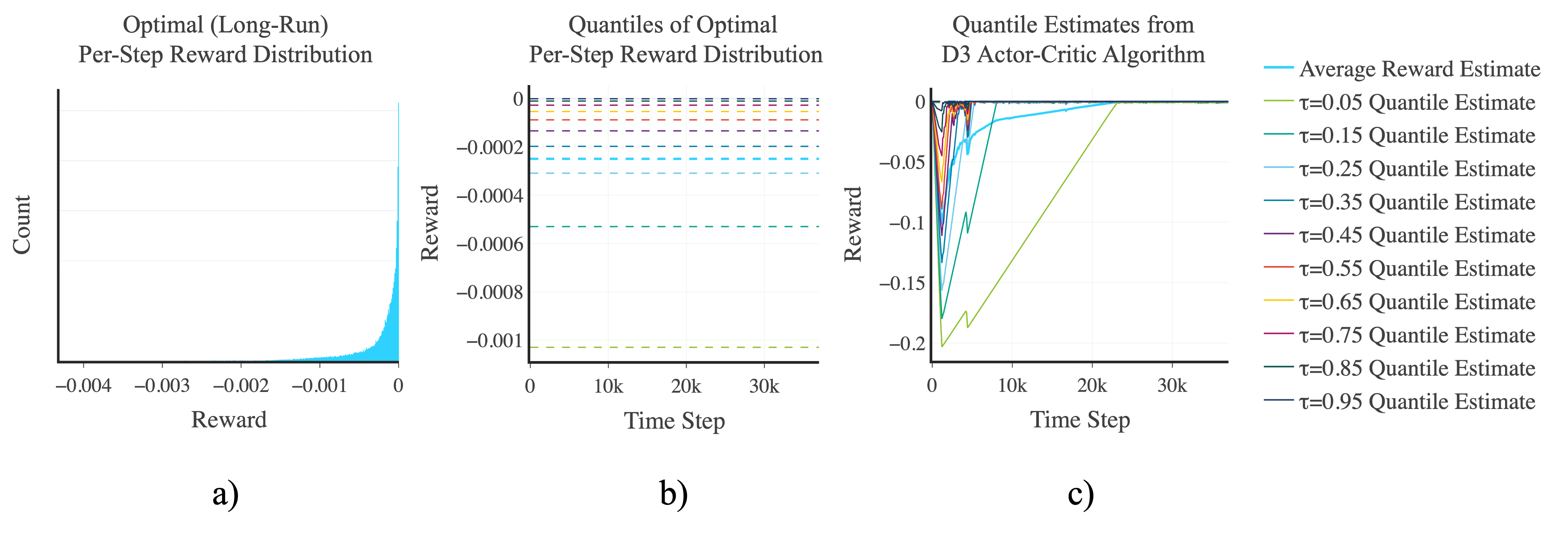}}
\caption{\textbf{a)} Histogram showing the empirical (\(\varepsilon\)-greedy) optimal (long-run) per-step reward distribution in the inverted pendulum task. \textbf{b)}  Quantiles of the (\(\varepsilon\)-greedy) optimal (long-run) per-step reward distribution in the inverted pendulum task. \textbf{c)} Convergence plot of the agent's per-step reward quantile estimates as learning progresses when using the D3 Actor-Critic algorithm in the inverted pendulum task.}
\label{fig_d3_pendulum_estimates}
\end{figure}

Figure \ref{fig_pendulum_results} shows the rolling average-reward as learning progresses when using the D2 and D3 algorithms vs. the non-distributional Differential algorithm in the inverted pendulum environment. As shown in the figure, the D2 and D3 algorithms yield competitive performance when compared to their non-distributional counterpart. 

\begin{figure}[htbp]
\centerline{\includegraphics[scale=0.75]{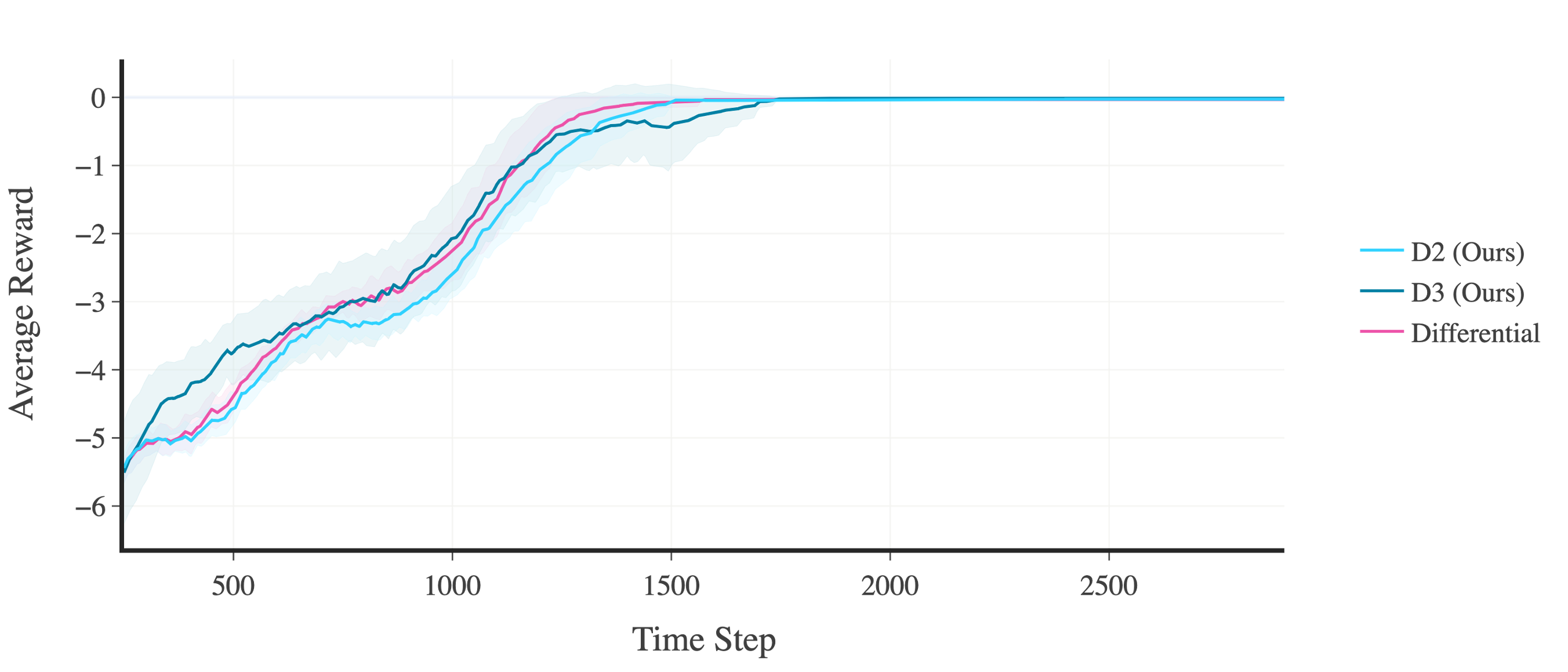}}
\caption{Rolling average-reward when using the D2 and D3 algorithms vs. the (non-distributional) Differential algorithm in the inverted pendulum environment. A solid line denotes the mean average-reward, and the corresponding shaded region denotes a 95\% confidence interval over 10 runs.}
\label{fig_pendulum_results}
\end{figure}

\newpage

\subsection{Atari Experiments}
\label{exp_atari}
In this section, we discuss the experiments performed in the \emph{Arcade Learning (i.e., Atari 2600) Environment} \citep{Bellemare2013-ef}. Through these experiments, we aimed to evaluate the empirical performance of the D2 and D3 algorithms in more difficult environments. More specifically, we aimed to compare the empirical performance of our distributional approach to that of an equivalent non-distributional one in these more difficult environments.

Our motivation for testing our algorithms in the Arcade Learning Environment (ALE) was twofold. First, ALE is a standard benchmark for evaluating distributional RL algorithms in the discounted setting, thereby making it a potentially-effective tool to use when one aims to measure the relative performance gains of distributional methods over their non-distributional counterparts. As such, it provided a means for us to compare the performance of our D2 and D3 algorithms to that of non-distributional Differential algorithms in challenging environments, as well as environments that require the use of non-linear function approximation. Second, there are currently no widely adopted benchmarks for the average-reward setting that match the scale and diversity of ALE. As such, ALE is one of the closest approximations we have for evaluating our algorithms under complex, high-dimensional environments. Moreover, using ALE in this context allowed us to observe how the D2 and D3 algorithms perform in cases where the environment does not strictly satisfy the theoretical assumptions of the average-reward setting. This hence provided us with the opportunity to gain insight into the robustness and generality of our methods beyond idealized settings. 

\begin{figure}[htbp]
\centerline{\includegraphics[scale=0.6]{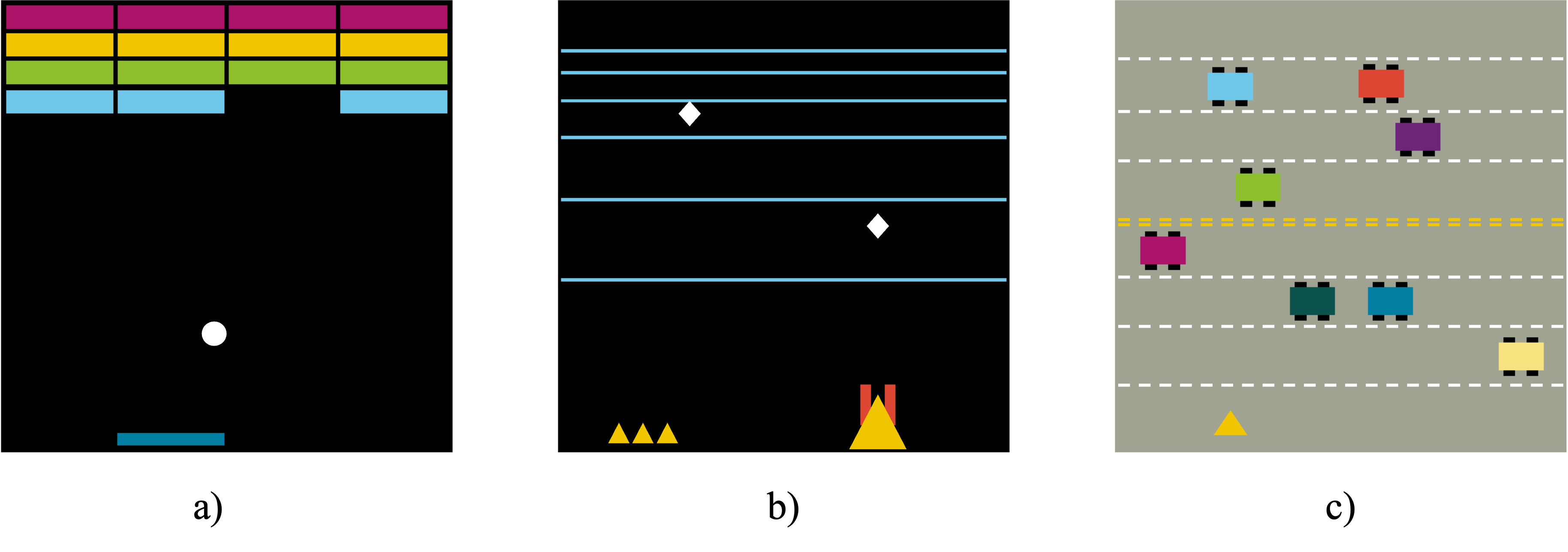}}
\caption{An illustration of the \textbf{a)} \emph{Breakout}, \textbf{b)} \emph{BeamRider}, and \textbf{c)} \emph{Freeway} Atari 2600 (ALE) environments.}
\label{fig_experiments_atari}
\end{figure}
 
In terms of the experiments performed, we tested our algorithms in three ALE environments: \emph{Breakout}, \emph{BeamRider}, and \emph{Freeway} (see Figure \ref{fig_experiments_atari} for an illustration of these experiments). For each environment, we compared our (non-tabular) D2 and D3 Q-learning algorithms (Algorithms \ref{alg_6} and \ref{alg_10}, respectively) to a (non-tabular) Differential Q-learning algorithm (Algorithm \ref{alg_reg_3}). We ran each algorithm using various combinations of step sizes and number of quantiles (where appropriate). We used the same number of per-step reward quantiles and differential return quantiles (i.e., \(m=n\)). We used a smooth L1 and quantile Huber loss \(\lambda\) of 1.0. We used a replay buffer with a minimum size of 1e2 and a max size of 1e5. We used a replay buffer sample batch size of 32. We used Polyak averaging to update the target networks (with a fixed step size of 0.005). The neural network architectures used in these experiments are provided at the end of this section. We set all initial guesses to zero, except for the average-reward and per-step reward quantiles, which used an initial guess of 0.25. For the hyperparameter tuning, we ran the \emph{Breakout} and \emph{BeamRider} experiments for 1M steps, and the \emph{Freeway} experiment for 3M steps (we subsequently ran the experiments for longer and across more random seeds to obtain the final results shown in Figure 4).\\

For the Differential Q-learning algorithm, we tested every combination of the value function step size, \(\alpha\in\{\text{2e-6, 2e-5, 2e-4, 2e-3}\}\), with the average-reward step size, \(\eta\alpha\), where \(\eta\in\{\text{1e-2, 1e-1, 1.0, 2.0, 10.0, 100.0}\}\), for a total of 24 unique combinations. Each combination was run 5 times using different random seeds, and the results were averaged across the runs. The combination of hyperparameters that yielded the best results and were used to generate the results displayed in Figure 4 are shown below in Table \ref{table_atari_1}.
\begin{table}[h!]
\centering
\begin{tabular}{|l|c|c|c|}
\hline
\textbf{} & \textbf{Breakout} & \textbf{BeamRider} & \textbf{Freeway} \\
\hline
Value Function Step Size & 2e-5 & 2e-6 & 2e-3 \\
Average-Reward \(\eta\) & 10.0 & 1e-2 & 10.0 \\
\hline
\end{tabular}
\caption{Tuned Hyperparameters for the Differential Q-learning Algorithm (\ref{alg_reg_3}) in the ALE Experiments.}
\label{table_atari_1}
\end{table}

\newpage

For the D2 Q-learning algorithm, we tested every combination of the number of quantiles, \(n=m \in \{\text{10, 51}\}\), with the value function step size, \(\alpha\in\{\text{2e-6, 2e-5, 2e-4, 2e-3}\}\), and the per-step reward quantile step size, \(\alpha_{_\theta} \doteq \eta_{_\theta}\alpha\), where \(\eta_{_\theta} \in \{\text{1e-2, 1e-1, 1.0, 2.0, 10.0, 100.0}\}\), for a total of 48 unique combinations. Each combination was run 5 times using different random seeds, and the results were averaged across the runs. The combination of hyperparameters that yielded the best results and were used to generate the results displayed in Figure 4 are shown below in Table \ref{table_atari_2}.
\begin{table}[h!]
\centering
\begin{tabular}{|l|c|c|c|}
\hline
\textbf{} & \textbf{Breakout} & \textbf{BeamRider} & \textbf{Freeway} \\
\hline
Value Function Step Size & 2e-5 & 2e-3 & 2e-5 \\
Per-Step Reward Quantile \(\eta\) & 100.0 & 100.0 & 2.0 \\
Number of Quantiles & 10 & 10 & 10 \\
\hline
\end{tabular}
\caption{Tuned Hyperparameters for the D2 Q-learning Algorithm (\ref{alg_6}) in the ALE Experiments.}
\label{table_atari_2}
\end{table}  

For the D3 Q-learning algorithm, we tested every combination of the number of quantiles, \(n=m \in \{\text{10, 51}\}\), with the differential return quantile step size, \(\alpha\in\{\text{2e-6, 2e-5, 2e-4, 2e-3}\}\), and the per-step reward quantile step size, \(\alpha_{_\theta} \doteq \eta_{_\theta}\alpha\), where \(\eta_{_\theta} \in \{\text{1e-2, 1e-1, 1.0, 2.0, 10.0, 100.0}\}\), for a total of 48 unique combinations. Each combination was run 5 times using different random seeds, and the results were averaged across the runs. The combination of hyperparameters that yielded the best results and were used to generate the results displayed in Figure 4 are shown below in Table \ref{table_atari_3}.
\begin{table}[h!]
\centering
\begin{tabular}{|l|c|c|c|}
\hline
\textbf{} & \textbf{Breakout} & \textbf{BeamRider} & \textbf{Freeway} \\
\hline
Differential Return Quantile Step Size & 2e-5 & 2e-4 & 2e-5 \\
Per-Step Reward Quantile \(\eta\) & 100.0 & 1.0 & 2.0 \\
Number of Quantiles & 51 & 51 & 51 \\
\hline
\end{tabular}
\caption{Tuned Hyperparameters for the D3 Q-learning Algorithm (\ref{alg_10}) in the ALE Experiments.}
\label{table_atari_3}
\end{table}  

\paragraph{Networks:\\}
Below is the PyTorch implementation of the neural networks used in the Atari 2600 (ALE) experiments. The first architecture was used in conjunction with the Differential and D2 algorithms. The second architecture was used in conjunction with the D3 algorithm:\\

DQN-style network used with the Differential and D2 algorithms:
\vspace{11pt}

\begin{lstlisting}[language=Python, numbers=none]
import torch

class ValueNetwork(torch.nn.Module):
    def __init__(self, state_dim, action_dim):
        super(ValueNetwork, self).__init__()
        self.conv1 = torch.nn.Conv2d(state_dim, 32, kernel_size=8, stride=4)
        self.conv2 = torch.nn.Conv2d(32, 64, kernel_size=4, stride=2)
        self.conv3 = torch.nn.Conv2d(64, 64, kernel_size=3, stride=1)
        self.fc1 = torch.nn.Linear(64 * 7 * 7, 512)
        self.fc2 = torch.nn.Linear(512, action_dim)

    def forward(self, x):
        x = torch.nn.functional.relu(self.conv1(x))
        x = torch.nn.functional.relu(self.conv2(x))
        x = torch.nn.functional.relu(self.conv3(x))
        x = x.view(x.size(0), -1)
        x = torch.nn.functional.relu(self.fc1(x))
        return self.fc2(x)
\end{lstlisting}

\newpage

QR-DQN-style network used with the D3 algorithm:
\vspace{11pt}

\begin{lstlisting}[language=Python, numbers=none]
import torch

class ValueNetwork(torch.nn.Module):
    def __init__(self, state_dim, action_dim, n_quantiles):
        super(ValueNetwork, self).__init__()
        self.state_dim = state_dim
        self.action_dim = action_dim
        self.distribution_dim = n_quantiles

        self.conv1 = torch.nn.Conv2d(state_dim, 32, kernel_size=8, stride=4)
        self.conv2 = torch.nn.Conv2d(32, 64, kernel_size=4, stride=2)
        self.conv3 = torch.nn.Conv2d(64, 64, kernel_size=3, stride=1)
        self.fc1 = torch.nn.Linear(64 * 7 * 7, 512)
        self.fc2 = torch.nn.Linear(512, action_dim * n_quantiles)

    def forward(self, x):
        x = torch.nn.functional.relu(self.conv1(x))
        x = torch.nn.functional.relu(self.conv2(x))
        x = torch.nn.functional.relu(self.conv3(x))
        x = x.view(x.size(0), -1)
        x = torch.nn.functional.relu(self.fc1(x))
        return self.fc2(x)
\end{lstlisting}

\vspace{24pt}

\subsection{Baselines}
Below is the pseudocode for the non-distributional Differential baseline algorithms used for comparison in our experiments:
\label{baseline_algorithms}

\begin{algorithm}
   \caption{Differential Q-Learning (Tabular) \citep{Wan2021-re}}
   \label{alg_reg_1}
\begin{algorithmic}
    \STATE {\bfseries Input:} the policy \(\pi\) to be used (e.g., \(\varepsilon\)-greedy)
    \STATE {\bfseries Algorithm parameters:} step size parameters \(\alpha\), \(\eta\)
    \STATE Initialize \(Q(s, a) \: \forall s, a\) (e.g. to zero)
    \STATE Initialize \(\bar{R}\) arbitrarily (e.g. to zero)
    \STATE Obtain initial \(S\)
    \WHILE{still time to train}
        \STATE \(A \leftarrow\) action given by \(\pi\) for \(S\)
        \STATE Take action \(A\), observe \(R, S'\)
        \STATE \(\delta = R - \bar{R} + \max_a Q(S', a) - Q(S, A)\)
        \STATE \(\bar{R} = \bar{R} + \eta \alpha \delta\)
        \STATE \(Q(S, A) = Q(S, A) + \alpha\delta\)
        \STATE \(S = S'\)
    \ENDWHILE
    \STATE return \(Q\)
\end{algorithmic}
\end{algorithm}

\begin{algorithm}
   \caption{Differential Actor-Critic}
   \label{alg_reg_2}
\begin{algorithmic}
    \STATE {\bfseries Input:} a differentiable state-value function parameterization \(\hat{v}(s, \boldsymbol{w})\), a differentiable policy parameterization \(\pi(a \mid s, \boldsymbol{u})\)
    \STATE {\bfseries Algorithm parameters:} step size parameters \(\alpha\), \(\eta_{\pi}\), \(\eta_{_{\bar{R}}}\)
    \STATE Initialize state-value weights \(\boldsymbol{w} \in \mathbb{R}^{d}\) and policy weights \(\boldsymbol{u} \in \mathbb{R}^{d'}\) (e.g. to \(\boldsymbol{0}\))
    \STATE Initialize \(\bar{R}\) arbitrarily (e.g. to zero)
    \STATE Obtain initial \(S\)
    \WHILE{still time to train}
        \STATE \(A \sim \pi(\cdot \mid S, \boldsymbol{u})\)
        \STATE Take action \(A\), observe \(R, S'\)
        \STATE \(\delta = R - \bar{R} + \hat{v}(S', \boldsymbol{w}) - \hat{v}(S, \boldsymbol{w})\)
        \STATE \(\bar{R} = \bar{R} + \eta_{_{\bar{R}}}\alpha\delta\)
        \STATE \(\boldsymbol{w} = \boldsymbol{w} + \alpha\delta\nabla\hat{v}(S, \boldsymbol{w})\)
        \STATE \(\boldsymbol{u} = \boldsymbol{u} + \eta_{\pi}\alpha\delta\nabla \text{ln} \pi(A \mid S, \boldsymbol{u})\)
        \STATE \(S = S'\)
    \ENDWHILE
    \STATE return \(\boldsymbol{w}\) and \(\boldsymbol{u}\)
\end{algorithmic}
\end{algorithm}

\begin{algorithm}
   \caption{Differential Q-Learning (Function Approximation with Replay Buffer)}
   \label{alg_reg_3}
\begin{algorithmic}
    \STATE {\bfseries Input:} a differentiable state-action value function parameterization: \(\hat{q}(s, a, \boldsymbol{w})\) (with target network \(\hat{q}_{_T}(s, a, \boldsymbol{w_{_T}})\)), the policy \(\pi\) to be used (e.g., \(\varepsilon\)-greedy)
    \STATE {\bfseries Algorithm parameters:} step size parameters \(\{\alpha\), \(\eta\}\), smooth L1 loss parameter \(\lambda\)
    \STATE Initialize state-action value weights \(\boldsymbol{w}, \boldsymbol{w_{_T}} \in \mathbb{R}^{d}\) arbitrarily (e.g. to \(\boldsymbol{0}\))    
    \STATE Initialize \(\bar{R}\) arbitrarily (e.g. to zero)
    \STATE Obtain initial \(S\)
    \WHILE{still time to train}
        \STATE \(A \leftarrow\) action given by \(\pi\) for \(S\)
        \STATE Take action \(A\), observe \(R, S'\)
        \STATE Store \((S, A, R, S')\) in replay buffer
        \IF {time to update estimates}
        \STATE Sample a minibatch of \(B\) transitions from replay buffer: \(\{(S_b, A_b, R_b, S_b')\}_{b=1}^{B}\)
        \STATE For each \(b\)-th transition: \(\delta_b = R_b - \bar{R} + \max_a \hat{q}_{_T}(S_b', a, \boldsymbol{w_{_T}}) - \hat{q}(S_b, A_b, \boldsymbol{w})\)
        \STATE For stability, use the smallest-magnitude TD error for the average-reward update: 
        \STATE \quad \(y = \arg\min_{b \in \{1,\dots,B\}} |\delta_b|\) (breaking ties arbitrarily)
        \STATE \quad \(\bar{R} = \bar{R} + \eta\alpha\delta_y\)
        \STATE \(\ell = -\frac{1}{B}\sum_{b=1}^{B}L^{ \lambda}(\delta_b)\) (See Equation \eqref{eq_smoothl1})
        \STATE \(\boldsymbol{w} = \boldsymbol{w} + \alpha\frac{\partial \ell}{\partial \boldsymbol{w}}\)
        \STATE Update \(\boldsymbol{w_{_T}}\) as needed (e.g. using Polyak averaging)
        \ENDIF
        \STATE \(S = S'\)
    \ENDWHILE
    \STATE return \(\boldsymbol{w}\)
\end{algorithmic}
\end{algorithm}



